\documentclass{llncs}
\usepackage{amsmath, amssymb, enumerate}
\usepackage{xcolor}
\usepackage{graphicx}
\usepackage{algorithm}
\usepackage{algpseudocode}

\usepackage{float}
\usepackage{placeins}

\newtheorem{defn}{Definition}

\usepackage{fancyhdr}
\usepackage{lipsum} 

\fancypagestyle{alim}{\fancyhf{}\fancyfoot[L]{This is a longer version of the conference paper which appeared in the proceedings of PAKDD 2019.}}

\title{Cost Sensitive Learning in the Presence of Symmetric Label Noise}
\author{Sandhya Tripathi \and N. Hemachandra}
\vspace{-1cm}
\institute{IEOR, Indian Institute of Technology Bombay, Mumbai 400076, India\\
	\email{\{sandhya.tripathi,nh\}@iitb.ac.in}}
\begin{document}
	\maketitle
	
	\begin{abstract}		
		In binary classification framework, we are interested in  making cost sensitive label predictions in the presence of uniform/symmetric label noise. We first observe that $0$-$1$ Bayes classifiers are not (uniform) noise robust in cost sensitive setting. To circumvent this impossibility result, we present two schemes; unlike the existing methods, our schemes do not require noise rate.  The first one uses $\alpha$-weighted $\gamma$-uneven margin squared loss function, $l_{\alpha, usq}$, which can handle cost sensitivity arising due to domain requirement (using user given $\alpha$) or class imbalance (by tuning $\gamma$) or  both. However, we observe that $l_{\alpha, usq}$  Bayes classifiers are also not cost sensitive and noise robust. We show that regularized ERM of this loss function over the class of linear classifiers yields a cost sensitive uniform noise robust classifier as a solution of a system of linear equations.  We also provide a performance bound for this classifier. The second scheme that we propose is a re-sampling based scheme that exploits the special structure of the uniform noise models and uses in-class probability $\eta$ estimates. Our computational experiments on some UCI datasets with class imbalance show that classifiers of our two schemes are on par with the existing methods and in fact better in some cases w.r.t. Accuracy and Arithmetic Mean, without using/tuning noise rate. We also consider other cost sensitive performance measures viz., F measure and Weighted Cost for evaluation. 
		
		If the noise is adversarial then, one can interpret not requiring the noise rates as a dominant strategy and cross-validation over noise rates as a  search for a dominant strategy. 
		Also, the required restriction of the hypothesis class to the linear class can be interpreted as an implicit form of regularization.
		As our re-sampling scheme requires estimates of $\eta$, we provide a detailed comparative study of various $\eta$ estimation methods on synthetic datasets, w.r.t.  half a dozen evaluation criterion. Also, we provide understanding on the interpretation of cost parameters $\alpha$ and $\gamma$ using different synthetic data experiments.
	\end{abstract}
	
	\thispagestyle{alim}
	\section{Introduction} \label{sec: intro}
	\vspace{-0.24cm}
	
	
	
	
	
	
	
	
	

	We are interested in cost sensitive label predictions when only noise corrupted labels are available. The labels might be corrupted when the data has  been collected by crowd scouring with not so high labeling expertise. We consider the basic case when the induced label noise is independent of the class or example, i.e., symmetric/uniform noise model. In real world, there are various applications requiring differential misclassification costs due to class imbalance or domain requirement or both; we elaborate on these below.
	
	In \textit{case 1}, there is no explicit need for different penalization of classes but the data has imbalance. Example: Predicting whether the age of an applicant for a vocational training course is above or below 15 years.  As the general tendency is to apply after high school, data is imbalanced but there is no need for asymmetric misclassification cost.
	Here, asymmetric cost should be learnt from data.
	
	In \textit{case 2}, there is no imbalance in data but one class's misclassification cost is higher than that of other. Example: Product recommendation by paid web advertisements. Even though the product is liked or disliked by approximately equal proportions of population, losing a potential customer by not recommending the product is more costlier than showing paid advertisement to a customer who doesn't like it. Here, the misclassification cost has to come from the domain.
	
	In \textit{case 3}, there is both imbalance and need for differential costing. Example: Rare (imbalance) disease diagnosis where the cost of missing a patient with disease is higher than cost of wrongly diagnosing a person with disease. The model should incorporate both the cost from domain and the cost due to imbalance.  
	
	In Section \ref{sec: back_material}, we provide a summary of how these 3 cases are handled. For cost and uniform noise, we have considered real datasets belonging to cases 2 and 3.
	\\
	\textbf{Contributions}
	\vspace{-.3cm}
	\begin{itemize}
		\item[$\bullet$] Show that, unlike $0$-$1$ loss, weighted $0$-$1$ loss is not cost sensitive uniform noise robust.
		\item[$\bullet$] Show $\alpha$-weighted $\gamma$-uneven margin squared loss $l_{\alpha, usq}$ with linear classifiers is both uniform noise robust and handles cost sensitivity. Present a performance bound of a classifier obtained from $l_{\alpha,usq}$ based regularized ERM.
		\item[$\bullet$] Propose a re-sampling based scheme for cost sensitive label prediction in the presence of uniform noise using in-class probability estimates.
		\item [$\bullet$] Unlike existing work, both the proposed schemes do not need true noise rate. 
		\item[$\bullet$] Using a balanced dataset (Bupa) which requires domain cost too, we demonstrate that tuning $\gamma$ on such corrupted datasets can be beneficial.
		\item[$\bullet$] Provide an extensive comparative study on synthetic datasets of in-class probability $\eta$ estimation methods like the link function based approach (\cite{reid2010composite}), LSPC (\cite{sugiyama2010superfast}), KLIEP (\cite{sugiyama2008direct}), $k$-NN based estimation. Also, a variety of metrics like MSE, MAD, Accuracy, KL divergence, Difference between maximum (minimum) value, etc.,  of the estimated and actual $\eta$ values are computed. 
		\item[$\bullet$] Identify the role of the two cost parameters for different scenarios of differential costing, viz., cost from domain or cost due to class imbalance or both. The interpretation is supported by synthetic data experiments w.r.t. Accuracy, AM and F measure.
	\end{itemize}
	\vspace{-.2cm}
	\begin{remark}
		Let us consider the case when there is an adversary who adds uniform noise to the data. Different levels of uniform noise is the strategy set of the adversary. The two schemes we propose are our strategies against the adversary. As we do not need true noise rates and hence indifferent to the adversary's strategy, our strategies would be always a dominant strategy. 
		However, if one cross validates over the noise rates, like in some existing work, the cross validation can be interpreted as a search for the dominant strategy or best response. Also, we consider differentital costing (class based as well as margin based) of the misclassifications. 
	\end{remark}
	
	\begin{remark}
		Restricting the search space for the classifiers from hypothesis class of all measurable functions to the class of linear functions can be viewed as a classical form of regularization. The intuition behind this is as follows: In the space of all measurable functions, the classifier is probably over-fitting and learning the noisy data. Whereas, when the hypothesis class is restricted to the class of linear functions, the generalization of the classifier learnt on noisy data is better.
	\end{remark}

	\textbf{Related work}
	For classification problems with label noise, particularly in Empirical Risk Minimization framework, the most recent work (\cite{manwani2013noise,natarajan2013learning,van2015learning,patrini2016loss,ghosh2015making}) aims to  make the loss function noise robust and then develop algorithms. A major roadblock one has to get around in label noise algorithms is the non-robustness of linear classifiers from convex potentials as given in \cite{long2010random}. We circumvent this problem by using squared loss function which is not a convex potential and hence a candidate for label noise robustness. We indeed show that $\alpha$-weighted $\gamma$-uneven margin squared loss with linear classifiers is SLN robust.  Cost sensitive learning has been widely studied by \cite{elkan2001foundations,masnadi2010risk,scott2012calibrated} and many more. 
	An extensive empirical study on the effect of label noise on cost sensitive learning is presented in \cite{zhu2007empiricalCostNoise}. 
	The problem of cost sensitive uniform noise robustness is considered in \cite{natarajan18cost} where asymmetric misclassification cost $\alpha$ is tuned and class dependent noise rates are cross validated over corrupted data. However, our work incorporates cost due to both imbalance ($\gamma$) and domain requirement ($\alpha$) with the added benefit that there is no need to know the true noise rate. \\
	\textbf{Organization}
	Section \ref{sec: back_material} has some details about weighted uneven margin loss functions and  in-class probability estimates. In Section \ref{sec: Cost_SLN_neg}, weighted $0$-$1$ loss is shown to be non cost sensitive uniform noise robust. Sections \ref{sec: l_usq_positive} and \ref{sec: res-scheme} present two different schemes that make cost sensitive predictions in the presence of uniform label noise. Section \ref{sec: experiments} has empirical evidence of the performance of proposed methods. Some discussion and future directions are presented in Section \ref{sec: discussion}.	\\  
	\textbf{Notations} Let $D$ be the joint distribution over $\mathbf{X}\times Y$ with $\mathbf{X} \in \mathcal{X} \subseteq \mathbb{R}^{n}$ and $Y\in \mathcal{Y} = \{-1,1\}.$ Let the in-class probability and class marginal on $D$ be denoted by $\eta(\mathbf{x}):=P(Y=1|\mathbf{x})$ and $\pi := P(Y=1)$. Let the decision function be $f:\mathbf{X}\mapsto \mathbb{R}$, hypothesis class of all measurable functions be $\mathcal{H}$ and class of linear hypothesis be $\mathcal{H}_{lin}= \{ (\mathbf{w},b), \mathbf{w}\in \mathbb{R}^n, b \in \mathbf{R}: \Vert \mathbf{w} \Vert_2 \leq W\}.$ Let $\tilde{D}$ denote the distribution on $\mathbf{X}\times \tilde{Y}$ obtained by inducing noise to $D$ with $\tilde{Y} \in \{-1,1\}$. The corrupted sample is $\tilde{S} = \{(\mathbf{x}_1,\tilde{y}_1),\ldots,(\mathbf{x}_m,\tilde{y}_m)\} \sim \tilde{D}^m.$ The noise rate $\rho := P(\tilde{Y}=-y|Y=y, \mathbf{X}=\mathbf{x})$ is constant across classes  and the model is referred to as Symmetric Label Noise (SLN) model. In such cases, the corrupted in-class probability is $\tilde{\eta}(\mathbf{x}) := P(\tilde{Y}=1|\mathbf{x}) = (1-2\rho)\eta(\mathbf{x}) + \rho$ and the corrupted class marginal is $\tilde{\pi}:= P(\tilde{Y}=1) = (1-2\rho)\pi+\rho.$ Symmetric and uniform noise are synonymous in this work.
	
	\subsection{Some relevant background} \label{sec: back_material}
	The first choice of loss function for cost sensitive learning is that of $\alpha$-weighted $0$-$1$ loss defined as follows: 
	\begin{equation}\label{eq: l_01alpha} 
	l_{0-1, \alpha}(f(\mathbf{x}),y) = (1-\alpha)\mathbf{1}_{\{Y=1, f(\mathbf{x}) \leq 0\}} + \alpha \mathbf{1}_{\{Y=-1,f(\mathbf{x}) > 0\}},  ~~ \forall \alpha \in (0,1)
	\end{equation}
	Let the $\alpha$-weighted $0$-$1$ risk be $R_{D,\alpha}:= E_D[l_{0-1,\alpha}(f(\mathbf{x}),y)]$. The minimizer of this risk is $f^*_{\alpha}(\mathbf{x}) = sign(\eta(\mathbf{x}) - \alpha)$ and referred to as cost-sensitive Bayes classifier. The corresponding surrogate $l_{\alpha}$ based risk and the minimizer is defined as $R_{D,l_{\alpha}}:= E_D[l_{\alpha}(f(\mathbf{x}),y)]$ and $f^*_{l_{\alpha}} \in \mathcal{H}$ respectively. 
	Consider the following notion of $\alpha$-classification calibration.
	\begin{defn}[$\alpha$-Classification Calibration \cite{scott2012calibrated}]
		For $\alpha \in (0,1)$ and a loss function $l$, define the $\alpha$-weighted loss: 
		\begin{equation} \label{eq: l_alpha}
		l_{\alpha}(f(\mathbf{x}),y) = (1-\alpha)l_1(f(\mathbf{x})) + \alpha l_{-1}(f(\mathbf{x})),
		\end{equation}
		where $l_1(\cdot) := l(\cdot,1)$ and $l_{-1}(\cdot) := l(\cdot,-1)$.
		$l_{\alpha}$ is $\alpha$-classification calibrated ($\alpha$-CC) iff there exists a convex, non-decreasing and invertible transformation $\psi_{l_{\alpha}},$ with $\psi_{l_{\alpha}}^{-1}(0) = 0$, such that 
		\begin{equation} \label{eq: alpha_CC_cond}
		\psi_{l_{\alpha}}(R_{D,\alpha}(f) - R_{D,\alpha}(f^*_{\alpha})) \leq R_{D,l_{\alpha}}(f) - R_{D,l_{\alpha}}(f^*_{l_{\alpha}}). 
		\end{equation}
		
	\end{defn}
	If the classifiers obtained from $\alpha$-CC losses are consistent w.r.t  $l_{\alpha}$-risk then they are also consistent w.r.t  $\alpha$-weighted $0$-$1$ risk.
	We consider the $\alpha$-weighted uneven margin squared loss \cite{scott2012calibrated} which is by construction $\alpha$-CC and defined as follows:
	\begin{equation} \label{eq: un-sq_loss}
	l_{\alpha,usq}(f(\mathbf{x}),y)= (1-\alpha)\mathbf{1}_{\{y=1\}}(1-f(\mathbf{x}))^2 + \alpha\mathbf{1}_{\{y=-1\}}\frac{1}{\gamma}(1+\gamma f(\mathbf{x}))^2, ~~ \gamma >0
	\end{equation}
	\textbf{Interpretation of $\alpha$ and $\gamma$}
	The role of $\alpha$ and $\gamma$ can be related to the three cases of differential costing described at the start of this paper. In \textit{case 1}, there are 3 options: fix $\alpha =0.5$ and let tuned $\gamma$ pick up the imbalance; fix $\gamma =1$ and tune $\alpha$; tune both $\alpha$ and $\gamma$. Our experimental results suggest that latter two perform equally good. For \textit{case 2}, $\alpha$ is given and $\gamma$ can be fixed at $1$. However, we observe that even in this case tuning $\gamma$ can be more informative.   For \textit{case 3}, $\gamma$ is tuned and $\alpha$ is given a priori.	
	There would be a trade-off between $\alpha$ and $\gamma$, i.e., for a given $\alpha$, there would be an optimal $\gamma$ in a suitable sense. Above observations are based on the experiments described in Supplementary material Section D. \\
	\textbf{Choice of $\eta$ estimation method}
	As $\eta$ estimates are required in re-sampling scheme, we investigated the performance of 4 methods:
	\textit{Lk-fun} \cite{reid2010composite}, uses classifier to get the estimate $\hat{\eta}$; interpreting $\eta$ as a conditional expectation and obtaining it by a suitable squared deviation minimization; LSPC \cite{sugiyama2010superfast}, density ratio estimation by  $l_2$ norm minimization; KLIEP \cite{sugiyama2008direct}, density ratio estimation by KL divergence minimization; $k$-nearest neighbour based estimation method. 
	We chose to use \textit{Lk-fun} with logistic loss $l_{log}$ and squared loss $l_{sq}$, LSPC, and $\hat{\eta}_{norm}$, a normalized version of KLIEP because in re-sampling algorithm we are concerned with label prediction and these estimators performed equally well on Accuracy measure. Another method to estimate $\eta$, based on $k$-nearest neighbours, is available in \cite{chen2018explaining}. This is a simple and computationally cheap method. However, it might be affected by the curse of dimensionality and mislead when not used with a carefully chosen value of $k$. A detailed study is available in Supplementary material Section F. 
	
	
	\section{Cost sensitive Bayes classifiers using $l_{0-1,\alpha}$ $\&$ $l_{\alpha,usq}$ need not be uniform noise robust } \label{sec: Cost_SLN_neg}
	The robustness notion for risk minimization in cost insensitive scenarios was introduced by \cite{manwani2013noise}. They also proved that cost insensitive $0$-$1$ loss based risk minimization is SLN robust. We extend this definition to cost-sensitive learning.
	\begin{defn}[Cost sensitive noise robust] \label{def: defn_CS_noise_gen}
		Let $f^*_{\alpha,A}$ and $\tilde{f}^*_{\alpha,A}$ be obtained from clean and corrupted distribution $D$ and $\tilde{D}$ using any arbitrary scheme $A$, then the scheme $A$ is said to be cost sensitive noise robust if 
		$$R_{D,\alpha}(\tilde{f}^*_{\alpha,A}) = R_{D,\alpha}(f^*_{\alpha,A}).$$
		If  $f^*_{\alpha,A}$ and $\tilde{f}^*_{\alpha,A}$ are obtained from a cost sensitive loss function $l$ and noise induced is symmetric, then $l$ is said to be cost sensitive uniform noise robust.	
	\end{defn}
	Let the $ l_{0-1,\alpha}$ risk on $\tilde{D}$ be denoted by $R_{\tilde{D},\alpha}(f)$. 
	If one is interested in cost sensitive learning with noisy labels, then the sufficient condition of \cite{ghosh2015making} becomes $(1-\alpha) \mathbf{1}_{[f(\mathbf{x}) \leq 0]} + \alpha \mathbf{1}_{[f(\mathbf{x})> 0]} = K$. This condition is satisfied if and only if $ (1-\alpha) =K = \alpha$ implying that it cannot be a sufficient condition for SLN robustness if there is a differential costing of $(1-\alpha, \alpha),~ \alpha \neq 0.5$. 
	
	Let $f^*_{\alpha}$ and $\tilde{f}^*_{\alpha}$ be the minimizers of $R_{D,\alpha}(f)$ and $R_{\tilde{D},\alpha}(f)$. Then, it is known that they have the following form: \\    
	{ \footnotesize
		\begin{minipage}[!htbp]{0.35\linewidth}
			\begin{equation} \label{eq: bayes_0-1_clean_alpha}
			f^*_{\alpha}{\mathbf{(x)}} = sign\left( \eta(\mathbf{x}) - \alpha \right)
			\end{equation}
		\end{minipage}
		\begin{minipage}[!htbp]{0.65\linewidth}
			\begin{eqnarray}\label{eq: bayes_0-1_corrup_alpha}
			\tilde{f}^*_{\alpha}(\mathbf{x}) = sign(\tilde{\eta}(\mathbf{x}) - \alpha)
			= sign \left( \eta(\mathbf{x}) - \frac{\alpha-\rho}{(1-2\rho)} \right)
			\end{eqnarray}
	\end{minipage}}
	The last equality in \eqref{eq: bayes_0-1_corrup_alpha} follows from the fact that $\tilde{\eta} = (1-2\rho)\eta + \rho$.
	In Example \ref{exam: unifDis_noise} below, we show that $l_{0-1,\alpha}$ is not cost sensitive uniform noise robust with $\mathcal{H}$.
	\begin{example} \label{exam: unifDis_noise}
		Let $Y$ has a Bernoulli distribution with parameter $p=0.2$. Let $\mathbf{X} \subset \mathbb{R}$  be such that  $\mathbf{X}|Y=1 \sim Uniform(0,p)$ and $\mathbf{X}|Y=-1 \sim Uniform(1-p,1)$. Then, the in-class probability $\eta(\mathbf{x})$ is given as follows:
		$$\eta(\mathbf{x}) = P(Y=1|\mathbf{X}=\mathbf{x}) = p =0.2$$ 
		Suppose $\rho = 0.3$. Then, $\tilde{\eta}(\mathbf{x}) = (1-2\rho)\eta(\mathbf{x})+\rho = 0.38$. If $\alpha = 0.25$,
		$ f^*_{\alpha}(\mathbf{x}) = -1 \text{ and } \tilde{f}^*_{\alpha}(\mathbf{x}) = 1, \forall \mathbf{x} \in \mathbf{X}.$
		Consider the $\alpha$-weighted $0$-$1$ risk of $f^*_{\alpha}(\mathbf{x})$ and $\tilde{f}^*_{\alpha}(\mathbf{x})$:
		{ \footnotesize
			\begin{eqnarray*}
				R_{D,\alpha}(f^*_{\alpha}) &=& E_{D}[l_{0-1,\alpha}(f^*_{\alpha}(\mathbf{x}),y)] 
				= (1-\alpha)p,   ~~~~~~~~~~~~~\text{ since } f^*_{\alpha}(\mathbf{x}) \leq 0 ~~ \forall \mathbf{x} \\
				R_{D,\alpha}( \tilde{f}^*_{\alpha}) &=& E_{D}[l_{0-1,\alpha}(\tilde{f}^*_{\alpha}(\mathbf{x}),y)] 
				= \alpha (1-p),  ~~~~~~~~~~~~~\text{ since } \tilde{f}^*_{\alpha}(\mathbf{x}) > 0 ~~ \forall \mathbf{x}
		\end{eqnarray*}}
		Therefore, $R_{D,\alpha}( f^*_{\alpha}) \neq R_{D,\alpha}( \tilde{f}^*_{\alpha})$ implying that the $\alpha$-weighted $0$-$1$ loss function $l_{0-1,\alpha}$ is not uniform noise robust with $\mathcal{H}$. Details are in Supplementary material Section B.1. Note that due to $p<0.5$, $D$ is linearly separable; a linearly inseparable variant can be obtained by $p>0.5$. Another linearly inseparable distribution based counter-example is available in Supplementary material Section B.4. 
		
		In view of the above example, one can try to use the principle of inductive bias, i.e., consider a strict subset of the above set of classifiers; however, Example \ref{exam: unifDis_noise_lin} below says that the set of linear class of classifiers need not be cost sensitive uniform noise robust.
	\end{example}
	\begin{example} \label{exam: unifDis_noise_lin}
		Consider the training set $\{(3,-1),(8,-1),(12,1)\}$ with uniform probability distribution.        Let the linear classifier be of the form $fl = sign(\mathbf{x}+b)$. Let $\alpha = 0.3$ and the uniform noise be $\rho = 0.42$. Then,
		{ \footnotesize
			\begin{eqnarray*}
				fl^*_{\alpha} &=& \arg\min\limits_{fl} E_{D}[l_{0-1,\alpha}(y,fl)]  = b^* \in (-8,-12) ~~\text{ with }~~ R_{\alpha,D}(fl^*_{\alpha}) = 0.\\
				\tilde{fl}^*_{\alpha} &=& \arg\min_{\tilde{fl}}E_{\tilde{D}}[l_{0-1,\alpha}(\tilde{y},\tilde{fl})] = \tilde{b}^* \in (-3,\infty) ~~\text{ with }~~  R_{\alpha,D}(\tilde{fl}^*_{\alpha}) = 0.2.
		\end{eqnarray*}}
	\end{example} 	
	Details of Example \ref{exam: unifDis_noise_lin} are available in Supplementary material Section B.2 
	To avoid above counter-examples, we resort to convex surrogate loss function and a type of regularization which restricts the hypothesis class. Consider an $\alpha$-weighted uneven margin loss functions $l_{\alpha, un}$ \cite{scott2012calibrated} with its optimal classifiers on $D$ and $\tilde{D}$ denoted by $f_{l_{\alpha, un}}^*$ and $\tilde{f}_{l_{\alpha, un}}^*$ respectively. Regularized risk minimization defined below is known to avoid over-fitting. 
	\begin{eqnarray}
	R^r_{D,l_{\alpha,un}}(f) = E_{D}[l_{\alpha,un}(f(\mathbf{x}),y)] + \lambda\Vert f \Vert_2^2,  ~~~~~ \text{where}  ~~~\lambda >0
	\end{eqnarray}
	Let the regularized risk of $l_{\alpha,un}$ on $\tilde{D}$ be $R^r_{\tilde{D},l_{\alpha,un}}(f)$. Also, let the minimizers of clean and corrupted $l_{\alpha,un}$-regularized risks be $f^*_{r,l_{\alpha,un}}$ and $\tilde{f}^*_{r,l_{\alpha,un}}$.
	Now, Definition \ref{def: defn_CS_noise_gen} can be specialized to $l_{\alpha,un}$ to assure cost sensitivity, classification calibration and uniform noise robustness as follows:
	\begin{defn}[$(\alpha,\gamma,\rho)$-robustness of risk minimization]
		For a loss function $l_{\alpha, un}$ and classifiers $\tilde{f}^*_{l_{\alpha, un}}$ and $f^*_{l_{\alpha, un}}$, risk minimization is said to be $(\alpha,\gamma,\rho)$-robust if 
		\begin{eqnarray} \label{eq: alpha_noise_robustness}
		R_{D,\alpha}(\tilde{f}^*_{l_{\alpha, un}}) = R_{D,\alpha}(f^*_{l_{\alpha, un}}).
		\end{eqnarray}
		Further, if the classifiers in equation \eqref{eq: alpha_noise_robustness} are $f^*_{r,l_{\alpha,un}}$ and $\tilde{f}^*_{r,l_{\alpha,un}}$ then, we say that regularized risk minimization under $l_{\alpha,un}$ is  $(\alpha,\gamma,\rho)$-robust.
	\end{defn}
	Due to squared loss's SLN robustness property \cite{manwani2013noise}, we check whether $l_{\alpha,usq}$ is $(\alpha,\gamma,\rho)$ robust or not. It is not with $\mathcal{H}$ as shown in Example \ref{exam: unifDis_noise_squ}; details of Example \ref{exam: unifDis_noise_squ} are available in Supplementary material Section B.3.
	\begin{example} \label{exam: unifDis_noise_squ}
		Consider the settings as in Example \ref{exam: unifDis_noise}. Let $\alpha = 0.25$ and $\gamma = 0.4$. Then, for all $\mathbf{x}$,
		{\scriptsize
			$$ f^*_{l_{\alpha,usq}}(\mathbf{x}) = \frac{\eta(\mathbf{x}) - \alpha}{\eta(\mathbf{x})(1-\alpha) + \gamma\alpha(1-\eta(\mathbf{x}))} =-0.21, ~~ \tilde{f}^*_{l_{\alpha,usq}}(\mathbf{x}) = \frac{\tilde{\eta}(\mathbf{x}) - \alpha}{\tilde{\eta}(\mathbf{x})(1-\alpha) + \gamma\alpha(1-\tilde{\eta}(\mathbf{x}))}= 0.37.$$
		}
		$$\text{And, }~~ R_{D,\alpha}(f^*_{l_{\alpha,usq}}) = (1-\alpha)p = 0.15, ~~~~R_{D,\alpha}(\tilde{f}^*_{l_{\alpha,usq}}) = \alpha (1-p) = 0.2.$$
		Hence, $R_{D,\alpha}(f^*_{l_{\alpha,usq}}) \neq R_{D,\alpha}(\tilde{f}^*_{l_{\alpha,usq}})$. 
		implying that $l_{\alpha,usq}$ based ERM may not be cost sensitive uniform noise robust.
	\end{example}
	
	We again have a negative result with $l_{\alpha,usq}$ when we consider hypothesis class $\mathcal{H}$. 
	In next section, we present a positive result and show that regularized risk minimization under loss function $l_{\alpha,usq}$ is $(\alpha,\gamma,\rho)$-robust if the hypothesis class is restricted to  $\mathcal{H}_{lin}$. 
	
	\section{$(l_{\alpha,usq},\mathcal{H}_{lin})$ is $(\alpha,\gamma,\rho)$ robust} \label{sec: l_usq_positive}
	In this section, we consider the weighted uneven margin squared loss function $l_{\alpha,usq}$ from equation \eqref{eq: un-sq_loss} with restricted hypothesis class $\mathcal{H}_{lin}$ and show a positive result that regularized cost sensitive risk minimization under loss function $l_{\alpha,usq}$ is $(\alpha,\gamma,\rho)$-robust.
	A proof is available in Supplementary material Section A.1.
	\begin{theorem} \label{the: un-sq robustness}
		$(l_{\alpha,usq},\mathcal{H}_{lin})$ is $(\alpha,\gamma,\rho)$-robust, i.e., linear classifiers obtained from $\alpha$-weighted $\gamma$-uneven margin squared loss $l_{\alpha,usq}$ based regularized risk minimization are SLN robust. 
	\end{theorem}
	
	\begin{remark}
		The above results relating to counter-examples and Theorem \ref{the: un-sq robustness} about cost sensitive uniform noise robustness can be summarized as follows: There are two loss functions, $l_{0-1,\alpha}$ and $l_{\alpha,usq}$ and two hypothesis classes, $\mathcal{H}_{lin}$ and $\mathcal{H}$. Out of the four combinations of loss functions and hypothesis classes only $l_{\alpha,usq}$ with $\mathcal{H}_{lin}$ is cost sensitive uniform noise robust, others are not.
	\end{remark}

	Next, we provide a closed form expression for the classifier learnt on corrupted data by minimizing empirical $l_{\alpha,usq}$-regularized risk. We also provide a performance bound on the clean risk of this classifier.
	\subsection{$l_{\alpha,usq}$ based classifier from corrupted data $\&$ its performance} \label{sec: sq-cls-f-perf}
	In this subsection, we present a descriptive scheme to learn a cost-sensitive linear classifier in the presence of noisy labels, by minimizing $l_{\alpha,usq}$ based regularized empirical risk, i.e.,
	\begin{equation}
	\hat{f}_r := \arg\min\limits_{f\in \mathcal{H}_{lin}}\hat{R}^r_{\tilde{D}, l_{\alpha,usq}}(f),
	\end{equation}
	where $ \hat{R}^r_{\tilde{D}, l_{\alpha,usq}}(f) := \frac{1}{m}\sum\limits_{i=1}^{m}l_{\alpha,usq}(f(\mathbf{x}_i),\tilde{y}_i) + \lambda \Vert f\Vert_2^2 ,$ $\alpha$ is user given, $\gamma$ and regularization parameter $\lambda>0$ are to be tuned by cross validation. A proof is available in Supplementary material Section A.2.
	\begin{proposition} \label{prop: l_sq_closed_form}
		Consider corrupted regularized empirical risk $\hat{R}^r_{\tilde{D}, l_{\alpha,usq}}(f)$ of $l_{\alpha,usq}.$ Then, the optimal $(\alpha,\gamma,\rho)$-robust linear classifier $\hat{f}_r = (\mathbf{w},b) \in \mathcal{H}_{lin}$ with $\mathbf{w} \in \mathbb{R}^{n}$ has the following form:
		\begin{equation}\label{eq: f-emp-uneven-sq}
		\hat{f}_r = \bar{\mathbf{w}}^* = (A + \lambda \mathbf{I})^{-1}\mathbf{c},~~~~ \lambda >0
		\end{equation}
		where $\bar{\mathbf{w}} = [w_1,w_2,\ldots,w_n,b]^T$, a $n+1$ dimensional vector of variables; $A$, a $(n+1 \times n+1)$ dimensional known symmetric matrix and $\mathbf{c}$, a $n+1$ dimensional known vector are as follows:
		{\footnotesize
			$$ A+\lambda I = \left[ {\begin{array}{ccccc}
				\sum\limits_{i=1}^m x_{i1}^2a_i + \lambda & \sum\limits_{i=1}^m x_{i1}x_{i2}a_i & \ldots & \sum\limits_{i=1}^m x_{i1}x_{in}a_i & \sum\limits_{i=1}^m x_{i1}a_i \\
				\sum\limits_{i=1}^m x_{i2}x_{i1}a_i & \sum\limits_{i=1}^m x_{i2}^2a_i + \lambda & \ldots & \sum\limits_{i=1}^m x_{i2}x_{in}a_i & \sum\limits_{i=1}^m x_{i2}a_i\\
				\vdots & \vdots & \ddots & \vdots & \vdots\\
				\sum\limits_{i=1}^m x_{in}x_{i1}a_i & \sum\limits_{i=1}^m x_{in}x_{i2}a_i & \ldots & \sum\limits_{i=1}^m x_{in}^2a_i + \lambda & \sum\limits_{i=1}^m x_{in}a_i\\
				\sum\limits_{i=1}^m x_{i1}a_i & \sum\limits_{i=1}^m x_{i2}a_i & \ldots & \sum\limits_{i=1}^m x_{in}a_i & \sum\limits_{i=1}^m a_i + \lambda \\
				\end{array} } \right], ~~ \mathbf{c} = \begin{bmatrix}
			\sum\limits_{i=1}^{m}x_{i1}c_i \\
			\sum\limits_{i=1}^{m}x_{i2}c_i  \\
			\vdots \\
			\sum\limits_{i=1}^{m}x_{in}c_i  \\
			\sum\limits_{i=1}^{m}c_i 
			\end{bmatrix}$$
			with $a_i = (\mathbf{1}_{[\tilde{y}_i=1]}(1-\alpha) +\gamma\alpha\mathbf{1}_{[\tilde{y}_i=-1]})$ and $c_i = (\mathbf{1}_{[\tilde{y}_i=1]}(1-\alpha) - \alpha\mathbf{1}_{[\tilde{y}_i=-1]})$.}
	\end{proposition}
	
	Next, we provide a result on the performance of $\hat{f}_r$ in terms of the Rademacher complexity of the function class $\mathcal{H}_{lin}.$ For this, we need Lemma \ref{lem: max_dev_Rad} and \ref{lem: R_noisy_clean_l_sq} whose proofs are available in Supplementary material Section A.3 and A.4 respectively.
	\begin{lemma} \label{lem: max_dev_Rad} 
		Consider the $\alpha$-weighted uneven margin squared loss $l_{\alpha,usq}(f(\mathbf{x}),y)$ which is locally $L$-Lipschitz with $L=2a+2$ where $|f(\mathbf{x})|\leq a$, for $a\geq 0$. Then, with probability at least $1-\delta$,{\footnotesize
			$$\max\limits_{f\in \mathcal{H}_{lin}}|\hat{R}_{\tilde{D}, l_{\alpha,usq}}(f) - R_{\tilde{D},l_{\alpha,usq}}(f)| \leq 2L\mathfrak{R}(\mathcal{H}_{lin}) + \sqrt{\frac{log(1/\delta)}{2m}},$$
		}
		where $\mathfrak{R}(\mathcal{H}_{lin}):= E_{\mathbf{X},\sigma}\left[\sup\limits_{f \in \mathcal{H}_{lin}}\frac{1}{m}\sum\limits_{i=1}^{m}\sigma_i f(\mathbf{x}_i)\right]$ is the Rademacher complexity of the function class $\mathcal{H}_{lin}$ with $\sigma_i$'s as independent uniform random variables taking values in $\{-1,1\}$.
	\end{lemma}
	
	\begin{lemma} \label{lem: R_noisy_clean_l_sq}
		For a  classifier $f \in \mathcal{H}$ and user given $\alpha \in (0,1)$, the $l_{\alpha,usq}$ risk on clean and corrupted distribution satisfy the following equation:
		\begin{equation}
		R_{\tilde{D},l_{\alpha,usq}}(f) = R_{D,l_{\alpha,usq}}(f) + 4\rho E_{D}[yf(\mathbf{x})[(1-\alpha)\mathbf{1}_{[y=1]} + \alpha\mathbf{1}_{[y=-1]}]].
		\end{equation}
	\end{lemma}
	
	\begin{theorem} \label{thm: perf_bound_f_from_usq}
		Under the settings of Lemma \ref{lem: max_dev_Rad}, with probability at least $1-\delta$,
		\begin{align*}
		R_{D,l_{\alpha,usq}}(\hat{f}_r) &\leq \min\limits_{f \in \mathcal{H}_{lin}}R_{D,l_{\alpha,usq}}(f) + 4L\mathfrak{R}(\mathcal{H}_{lin}) + 2\sqrt{\frac{log(1/\delta)}{2m}} +  2\lambda W^2 +\\ 
		& \frac{4\rho}{(1-2\rho)}E_{\mathbf{X}}[(\tilde{f}l^*_{l_{\alpha,usq}}(\mathbf{x}) - (1-2\rho)\hat{f}_r(\mathbf{x}))(\eta(\mathbf{x}) - \alpha)]
		\end{align*}
		where $\tilde{f}l^*_{l_{\alpha,usq}}$ is the linear minimizer of $R_{\tilde{D},l_{\alpha,usq}}$ and ${\eta}(\mathbf{x})$ is the in-class probability for $\mathbf{x}$. Furthermore, as $l_{\alpha,usq}$ is $\alpha$-CC, there exists a non-decreasing and invertible function $\psi_{l_{\alpha,usq}}$ with $\psi_{l_{\alpha,usq}}^{-1}(0)=0$ such that,
		{ \footnotesize
			\begin{eqnarray} \nonumber
			R_{D,\alpha}(\hat{f}_r) - R_{D,\alpha}(f^*_{\alpha})
			&\leq& \psi_{l_{\alpha,usq}}^{-1} \left(   \min\limits_{f \in \mathcal{H}_{lin}}R_{D,l_{\alpha,usq}}(f) - \min\limits_{f \in \mathcal{H}}R_{D,l_{\alpha,usq}}(f) + 4L\mathfrak{R}(\mathcal{H}_{lin}) \right. \\ \nonumber
			& & + 2\sqrt{\frac{log(1/\delta)}{2m}}  + \frac{4\rho}{(1-2\rho)}E_{\mathbf{X}}[(\tilde{f}l^*_{l_{\alpha,usq}}(\mathbf{x}) -\hat{f}_r(\mathbf{x}))(\eta(\mathbf{x}) - \alpha)]  \\  \label{eq: sq_bound_perf}
			& & \left. + \frac{8\rho^2}{(1-2\rho)}E_{\mathbf{X}}[\hat{f}_r(\mathbf{x})(\eta(\mathbf{x}) - \alpha)] + 2\lambda W^2 \right).
			\end{eqnarray}}
	\end{theorem}
	A proof of Theorem \ref{thm: perf_bound_f_from_usq} is available in Supplementary material Section A.5.
	The first two terms (involving the difference) in the right hand side of equation \eqref{eq: sq_bound_perf} denotes the approximation error which is small if $\mathcal{H}_{lin}$ is large and the third term involving the Rademacher complexity denotes the estimation error which is small if $\mathcal{H}_{lin}$ is small. The fourth term denotes the sample complexity which vanishes as the sample size increases. The bound in \eqref{eq: sq_bound_perf} can be used to show consistency of $l_{\alpha,usq}$ based regularized ERM if the argument of $\psi^{-1}_{l_{\alpha,usq}}$ tends to zero as sample size increases. However, in this case, it is not obvious because the last two terms involving noise rates  may not vanish with increasing sample size. In spite of this, our empirical experience with this algorithm is very good.
	
	\section{A re-sampling based \ref{alg: eta_cost_SLN}} \label{sec: res-scheme}
	In this section, we present a cost sensitive label prediction algorithm based on re-balancing (which is guided by the costs) the noisy training set given to the learning algorithm. 
	Let us consider uneven margin version of $\alpha$-weighted $0$-$1$ loss from equation \eqref{eq: l_01alpha} defined as follows:
	{ \footnotesize
		\begin{equation*} \label{eq: 0-1-alpha-uneven}
		l_{0-1,\alpha,\gamma}(f(x),y) = (1-\alpha)\mathbf{1}_{\{Y=1, f(\mathbf{x}) \leq 0\}} + \frac{\alpha}{\gamma} \mathbf{1}_{\{Y=-1,\gamma f(\mathbf{x}) > 0\}},  ~~ \forall \alpha \in (0,1)
		\end{equation*} }
	where $\alpha$ is user given cost and $\gamma$, tunable cost handles the class imbalance. This definition is along the lines of the uneven margin losses defined in \cite{scott2012calibrated}.	
	Let $l_{0-1,\alpha,\gamma}$-risk on $D$ be $R_{D,\alpha,\gamma}(f)$ and corresponding optimal classifier be $f^*_{0-1,\alpha,\gamma}$:
	{\footnotesize
		\begin{eqnarray} \label{eq: f_0-1_alpha-uneven_clean}
		f^*_{0-1,\alpha,\gamma} = \arg\min\limits_{f \in \mathcal{H}}R_{D,\alpha,\gamma}(f) 
		= sign\left( {\eta}(\mathbf{x}) - \frac{\alpha}{\gamma +(1-\gamma)\alpha}\right).
		\end{eqnarray}}
	Also, let $l_{0-1,\alpha,\gamma}$-risk on $\tilde{D}$ be $R_{\tilde{D},\alpha,\gamma}(f)$ and the corresponding optimal classifier be $\tilde{f}^*_{0-1,\alpha,\gamma}$ as given below:
	{\footnotesize
		\begin{eqnarray} \label{eq: f_0-1_alpha-uneven_noisy}
		\tilde{f}^*_{0-1,\alpha,\gamma} &=& sign\left( \tilde{\eta}(\mathbf{x}) - \frac{\alpha}{\gamma +(1-\gamma)\alpha}\right).
		\end{eqnarray}}
	We propose \ref{alg: eta_cost_SLN} which is mainly based on two ideas: (i) predictions based on a certain threshold $(p^*)$ can correspond to predictions based on threshold $(p_0)$ if the number of negative examples in the training set is multiplied by $r^* = \frac{p^*}{1-p^*}\frac{1-p_0}{p_0}$ (Theorem 1 of \cite{elkan2001foundations}) (ii) for a given $\mathbf{x}$, $\tilde{\eta}(\mathbf{x})$ and $\eta(\mathbf{x})$ lie on the same side of threshold $0.5$ when noise rate is $\rho$. We first formalize the latter idea in terms of a general result.
	A proof is available in Supplementary material Section A.6.
	\begin{lemma} \label{lem: SLN_pi_mono}
		In SLN models, for a given noise rate $\rho < 0.5$, the clean and corrupted class marginals $\pi$ and $\tilde{\pi}$ satisfy the following condition:
		$$  ~~~~~~~~~~~\pi \lesseqgtr 0.5  \Rightarrow \tilde{\pi} \lesseqgtr 0.5.$$
		Further, the above monotonicity holds for $\eta(\mathbf{x})$ and $\tilde{\eta}(\mathbf{x})$ too.
	\end{lemma}
	
	
	In our case, the cost sensitive label prediction requires the desired threshold to be $\frac{\alpha}{\gamma +(1-\gamma)\alpha}$ ($=p^*$) but the threshold which we can use is $0.5 ~(=p_0)$ implying that for us $r^* = \frac{\alpha}{\gamma(1-\alpha)}$. If $m_+$ and $m_{-}$ are number of positive and negative examples in $m_{tr}$, then we should re-sample such that the size of balanced dataset is $m_{tr,b} = m_+ + \left\lfloor{\frac{\alpha m_{-}}{\gamma(1-\alpha)}}\right\rfloor$. 
	As we have access to only corrupted data, the learning scheme is: re-balance the corrupted data using $r^*$ and then threshold $\tilde{\eta}$ at $0.5$. Since, for SLN model, predictions made by thresholding $\tilde{\eta}$ at $0.5$ are same as the predictions made by thresholding $\eta$ at $0.5$, for a test point $\mathbf{x}_0$ from $D$, predicted label is $sign(\tilde{\eta}(\mathbf{x}_0)-0.5).$ The main advantage of this algorithm is that it doesn't require the knowledge of true noise rates. Also, unlike Section \ref{sec: l_usq_positive}'s scheme involving $l_{\alpha,usq}$ based regularized  ERM, this algorithm uses $\tilde{\eta}$ estimates and hence is a generative learning scheme.
	
	
	Since, we do not want to lose any minority (rare) class examples, we reassign positive labels to the minority class WLOG, if needed, implying that negative class examples are always under-sampled. The performance of \ref{alg: eta_cost_SLN} is majorly dependent on sampling procedure and $\tilde{\eta}$ estimation methods used. 
	\begin{algorithm}[!htbp]
		\renewcommand{\thealgorithm}{Algorithm $\boldsymbol{(\tilde{\eta},\alpha)}$}
		\floatname{algorithm}{}
		\caption{$\tilde{\eta}$ based scheme to make cost sensitive label predictions from uniform noise corrupted data} 
		\label{alg: eta_cost_SLN}
		\begin{algorithmic}[1]
			\Statex \hspace{-0.44cm}\textbf{Input:} Training data $\tilde{S}_{tr} = \{(x_1,\tilde{y}_i)\}_{i=1}^{m_{tr}}$, test data $S_{te} = \{(x_i,y_i)\}_{i=1}^{m_{te}}$, cost $\alpha$, performance measure $PM \in \{Acc, AM, F, WC \}$ 
			\Statex \hspace{-0.44cm}\textbf{Output:} Predicted labels and $\tilde{\eta}$ estimate on test data $S_{te}$
			\State $\gamma_0 = \frac{\alpha}{1-\alpha} + 0.001,~~~~$  since for under-sampling $r^* = \frac{\alpha}{\gamma(1-\alpha)} < 1 \Rightarrow \gamma_0 > \frac{\alpha}{1-\alpha}$
			\State $\Gamma = \{\frac{\alpha}{1-\alpha} + i \times 0.05 ~~~ i=1,\ldots,12\}$
			\For{$\gamma \in \Gamma$}
			\State Under-sample the $-ve$ class to  get $\tilde{S}_{tr,b}$ such that $\vert \tilde{S}_{tr,b} \vert = m_+ + \left\lfloor{\frac{\alpha m_{-}}{\gamma(1-\alpha)}}\right\rfloor .$
			\State Use $5$-fold CV to estimate  $\tilde{\eta}$ from $\tilde{S}_{tr,b}$ via \textit{Lk-fun} or LSPC or KLIEP.
			\State Compute $5$-fold cross validated $PM$ from the partitioned data.
			\EndFor 
			\State $\gamma^* = \arg\max\limits_{\gamma \in \Gamma} PM$ if $PM = \{Acc,AM,F\}$ otherwise $\gamma^* = \arg\min\limits_{\gamma \in \Gamma} WC$. 
			\State Under-sample the $-ve$ class to  get $\tilde{S}_{tr,b}$ such that $\vert \tilde{S}_{tr,b} \vert = m_+ + \left\lfloor{\frac{\alpha m_{-}}{\gamma^*(1-\alpha)}}\right\rfloor .$
			\State Estimate $\tilde{\eta}$ from $\tilde{S}_{tr,b}$ using \textit{Lk-fun} method or LSPC or KLIEP.
			\For{$i=1,2,\ldots,m_{te}$}
			\State Compute $\tilde{\eta}(\mathbf{x}_i)$
			\State $\hat{y}_{i} = sign(\tilde{\eta}(\mathbf{x}_i) -0.5)$ 
			\EndFor
		\end{algorithmic}
	\end{algorithm}
	\begin{remark}\label{REM-Resampling}
		\ref{alg: eta_cost_SLN} exploits the fact that $sign(\tilde{\eta}-\frac{\alpha}{\gamma +(1-\gamma)\alpha}) = sign(\tilde{\eta}_b-0.5)$ where $\tilde{\eta}_b$ is learnt on re-sampled data. This implies $R_{D,\alpha}(sign(\tilde{\eta}_b-0.5))=R_{D,\alpha}(sign(\tilde{\eta}-\frac{\alpha}{\gamma +(1-\gamma)\alpha}))$ but due to counter-examples in Section \ref{sec: Cost_SLN_neg}, these risks may not be equal to $R_{D,\alpha}(sign(\eta-\frac{\alpha}{\gamma +(1-\gamma)\alpha}))$. Hence, this scheme is not in contradiction to Section \ref{sec: Cost_SLN_neg}. However, as $\tilde{\eta}$ estimation methods use a subset of $\mathcal{H}$ (e.g., LSPC and KLIEP use linear combinations of finite Gaussian kernels as basis functions), these risks may be equal to $R_{D,\alpha}(sign(\tilde{\eta}_e -\frac{\alpha}{\gamma +(1-\gamma)\alpha}))$  where $\tilde{\eta}_{e}$ is estimate of  $\tilde{\eta}$ obtained from strict subset of hypothesis class.
		Also, based on very good empirical performance of the scheme, we believe that $R_{D,\alpha}(sign(\tilde{\eta}_{e}-\frac{\alpha}{\gamma +(1-\gamma)\alpha})) = R_{D,\alpha}(sign(\eta_{e}-\frac{\alpha}{\gamma +(1-\gamma)\alpha}))$  where $\eta_e$ is an estimate of $\eta$.
	\end{remark}
	\section{Comparison  of $l_{\alpha,usq}$ based regularized ERM and \ref{alg: eta_cost_SLN} to existing methods on UCI datasets} \label{sec: experiments}
	In this section, we consider some UCI datasets \cite{UCI_dataset} and demonstrate that $l_{\alpha,usq}$ is $(\alpha,\gamma,\rho)$-robust. Also, we demonstrate the performance of \ref{alg: eta_cost_SLN} with $\tilde{\eta}$ estimated using  \textit{Lk-fun}, LSPC and KLIEP. In addition to Accuracy (Acc), Arithmetic mean (AM) of True positive rate (TPR) and True negative rate (TNR), we also consider two  measures suited for evaluating classifiers learnt on imbalanced data, viz., F measure and Weighted cost (WC) defined as below: 
	{\scriptsize
		$$ F =  \frac{2TP}{2TP + FP+ FN} ~~~\text{ and }~~~ WC = (1-\alpha)FN + \frac{\alpha}{\gamma}FP$$ 
	}
	where TP, TN, FP, FN are number of true positives, true negatives, false positives and false negatives for a classifier.
	
	To account for randomness in the flips to simulate a given noise rate, we repeat each experiment 10 times, with independent corruptions of the data set for same noise ($\rho$) setting. In every trial, the data is partitioned into train and test with $80$-$20$ split. Uniform noise induced $80\%$ data is used for training and validation (if there are any parameters to be tuned like $\gamma$).
	Finally, $20\%$ clean test data is used for evaluation.
	Regularization parameter, $\lambda$ is tuned over the set $\Lambda = \{0.01,0.1,1,10\}.$ 
	On a synthetic dataset, we observed that the performance of our methods and cost sensitive Bayes classifier on clean data w.r.t. Accuracy, AM, F and Weighted Cost measure is comparable for moderate noise rates; details in Supplementary material Section E.3. 
	In all the tables, values within $1\%$ of the best across a row are in bold. \\
	\textbf{Class imbalance and domain requirement of cost}
	We report the accuracy and AM values of logistic loss based unbiased estimator (MUB) approach and approach of surrogates for weighted 0-1 loss (S-W0-1), as it is from the work of \cite{natarajan18cost} and compare them to Accuracy and AM for our cost sensitive learning schemes. It is to be noted that \cite{natarajan18cost} assumes that the true noise rate $\rho$ is known and cost $\alpha$ is tuned. We are more flexible and user friendly as we don't need the noise rate $\rho$ and allow for user given misclassification cost $\alpha$ and tune $\gamma$. %
	
	It can be observed in Table \ref{tab: Ac_measure_cost} that as far as Accuracy is concerned \ref{alg: eta_cost_SLN} and $l_{\alpha,usq}$ based regularized ERM have comparable values to that from MUB and S-W0-1 on all datasets. As depicted in Table \ref{tab: AM_measure_cost}, the proposed algorithms have marginally better values of AM measure than that of MUB and S-W0-1 method. 
	When comparing the performance of classifiers learnt using $l_{\alpha,usq}$ based RegERM on noisy data and benchmark of classifiers learnt on clean data (i.e., $\rho=0.0$), we observe that in some cases like German dataset, the validity of robustness definition in Eq. \eqref{eq: alpha_noise_robustness} is questionable (See Table \ref{tab: Ac_measure_cost}). However, this can be explained due  to the availability of only a finite sample for training and evaluating the classifier.  
	We also investigated the performance of two proposed methods w.r.t. F measure and Weighted Cost measure on UCI datasets and observed that \ref{alg: eta_cost_SLN} works well on most datasets but the underlying $\eta$ estimation method used varies across datasets.
	However, due to space constraint the details are presented in Supplementary material Section E.1.
	\vspace{-0.5cm}
	\FloatBarrier
	\begin{table}[!ht]
		\centering
		{ \scriptsize
			\bgroup
			\def\arraystretch{0.87}%
			{\setlength{\tabcolsep}{0.21em}	
				\begin{tabular}{|c|c|c|c|c|c|c|c|c|c|}
					\hline
					\textbf{Dataset} & \textbf{Cost} & \textbf{$\rho$} & \multicolumn{4}{c|}{\textbf{\ref{alg: eta_cost_SLN}}: $\tilde{\eta}$ estimate from } & \textbf{\begin{tabular}[c]{@{}c@{}}$l_{\alpha,usq}$\\ RegERM\end{tabular}} & \textbf{MUB} & \textbf{S-W0-1} \\ \hline
					$(n,m_+,m_-)$& \textbf{$\alpha$} & \textbf{} & \textbf{\begin{tabular}[c]{@{}c@{}} \textit{Lk-fun} \\ with $l_{sq}$\end{tabular}} & \textbf{\begin{tabular}[c]{@{}c@{}}\textit{Lk-fun} \\ with $l_{log}$\end{tabular}} & \textbf{LSPC} & \textbf{KLIEP} & \textbf{} & \textbf{$l_{log}$} & \textbf{$l_{log}$} \\ \hline
					\textbf{\begin{tabular}[c]{@{}c@{}}Breast\\ Cancer \\ (9,77,186)\end{tabular}} & 0.2 & \begin{tabular}[c]{@{}c@{}}0.0 \\ 0.2\\ 0.4\end{tabular} & \begin{tabular}[c]{@{}c@{}}  0.70$\pm$0.06 \\0.71$\pm$0.06\\ 0.45$\pm$0.15 \end{tabular} & \begin{tabular}[c]{@{}c@{}}0.71$\pm$0.05 \\ \textbf{0.72$\pm$0.04} \\ 0.46 $\pm 0.16$\end{tabular} & \begin{tabular}[c]{@{}c@{}} \textbf{ 0.72$\pm$0.04} \\ \textbf{0.72$\pm$0.04}\\ 0.52$\pm$0.11\end{tabular} & \begin{tabular}[c]{@{}c@{}} \textbf{0.72$\pm$0.03} \\ 0.61$\pm$0.09\\ 0.54$\pm$0.13 \end{tabular} & \begin{tabular}[c]{@{}c@{}}0.71$\pm$0.05 \\ \textbf{0.72$\pm$0.04}\\ \textbf{0.66$\pm$0.06}\end{tabular} & \begin{tabular}[c]{@{}c@{}}- \\0.70\\ \textbf{0.67}\end{tabular} & \begin{tabular}[c]{@{}c@{}} - \\0.66\\ 0.56\end{tabular} \\ \hline
					\textbf{\begin{tabular}[c]{@{}c@{}}Pima\\ Diabetes \\ (8,268,500)\end{tabular}} & 0.25 & \begin{tabular}[c]{@{}c@{}}0.0 \\ 0.2\\ 0.4\end{tabular} & \begin{tabular}[c]{@{}c@{}} 0.75$\pm$0.03 \\ 0.74$\pm$0.02\\ 0.70$\pm$0.06\end{tabular} & \begin{tabular}[c]{@{}c@{}}\textbf{0.76$\pm$0.02} \\ 0.74$\pm$0.03\\ \textbf{0.71$\pm$0.04}\end{tabular} & \begin{tabular}[c]{@{}c@{}}0.75$\pm$0.03\\ 0.71$\pm$0.04\\ 0.51$\pm$0.09\end{tabular} & \begin{tabular}[c]{@{}c@{}}  0.71$\pm$0.03 \\ 0.63$\pm$0.03\\ 0.57$\pm$0.07\end{tabular} & \begin{tabular}[c]{@{}c@{}}\textbf{0.76$\pm$0.02} \\ \textbf{0.77$\pm$0.03}\\ \textbf{0.71$\pm$ 0.04}\end{tabular} & \begin{tabular}[c]{@{}c@{}}- \\ \textbf{0.76}\\ 0.65\end{tabular} & \begin{tabular}[c]{@{}c@{}}- \\ 0.73\\ 0.66\end{tabular} \\ \hline
					\textbf{\begin{tabular}[c]{@{}c@{}}German \\ (20,300,700)\end{tabular}} & 0.16 & \begin{tabular}[c]{@{}c@{}}0.0 \\ 0.2\\ 0.4  \end{tabular} & \begin{tabular}[c]{@{}c@{}}\textbf{0.76$\pm$0.03} \\ \textbf{0.74$\pm$0.02}\\ 0.64$\pm$0.05\end{tabular} & \begin{tabular}[c]{@{}c@{}}0.75$\pm$0.03 \\ \textbf{0.74$\pm$0.02}\\ 0.64$\pm$0.05\end{tabular} & \begin{tabular}[c]{@{}c@{}} 0.74$\pm$ 0.04 \\ 0.72$\pm$0.01\\ 0.58$\pm$0.04\end{tabular} & \begin{tabular}[c]{@{}c@{}}0.71$\pm$0.01 \\ 0.70$\pm$0.03 \\ \textbf{0.66$\pm$0.02}  \end{tabular} & \begin{tabular}[c]{@{}c@{}}0.75$\pm$0.02 \\ \textbf{0.74$\pm$0.03}\\ 0.64$\pm$0.03\end{tabular} & \begin{tabular}[c]{@{}c@{}}- \\ 0.66\\ 0.55\end{tabular} & \begin{tabular}[c]{@{}c@{}}- \\ 0.71\\ \textbf{0.67}\end{tabular} \\ \hline
					\textbf{\begin{tabular}[c]{@{}c@{}}Thyroid \\ (5,65,150)\end{tabular}} & 0.3 & \begin{tabular}[c]{@{}c@{}} 0.0 \\ 0.2\\ 0.4  \end{tabular} & \begin{tabular}[c]{@{}c@{}}0.86$\pm$0.04 \\ 0.84$\pm$0.04\\ 0.70$\pm$0.15\end{tabular} & \begin{tabular}[c]{@{}c@{}}0.89$\pm$0.04 \\0.85$\pm$0.05\\ 0.67$\pm$0.15\end{tabular} & \begin{tabular}[c]{@{}c@{}} \textbf{0.94$\pm$0.03}\\ \textbf{0.89$\pm$0.03}\\ 0.61$\pm$0.07\end{tabular} & \begin{tabular}[c]{@{}c@{}}0.78$\pm$0.10 \\ 0.76$\pm$0.13\\ 0.70$\pm$0.12\end{tabular} & \begin{tabular}[c]{@{}c@{}}0.86$\pm$0.03 \\ 0.84$\pm$0.03\\ \textbf{0.82$\pm$0.06} \end{tabular} & \begin{tabular}[c]{@{}c@{}}- \\ 0.87\\ \textbf{0.83}\end{tabular} & \begin{tabular}[c]{@{}c@{}}- \\ 0.82\\ 0.76\end{tabular} \\ \hline
		\end{tabular}}}
		\egroup
		\caption[Accuracy from \ref{alg: eta_cost_SLN} and $l_{\alpha,usq}$ on real datasets]{Averaged Acc $(\pm~s.d.)$ of the cost sensitive predictions made by \ref{alg: eta_cost_SLN} and $l_{\alpha,usq}$ based regularized ERM on UCI datasets corrupted by uniform noise. Accuracy values for \textbf{MUB} and \textbf{S-W0-1} are not available at $\rho=0.0$.}
		\label{tab: Ac_measure_cost} 
	\end{table}
	\vspace{-0.5cm}
	\FloatBarrier
	\begin{table}[!ht]
		\centering
		{ \scriptsize
			\bgroup
			\def\arraystretch{0.87}%
			{\setlength{\tabcolsep}{0.18em}	
				\begin{tabular}{|c|c|c|c|c|c|c|c|c|c|}
					\hline
					\textbf{Dataset} & \textbf{Cost} & \textbf{$\rho$} & \multicolumn{4}{c|}{\textbf{\ref{alg: eta_cost_SLN}}: $\tilde{\eta}$ estimate from } & \textbf{\begin{tabular}[c]{@{}c@{}}$l_{\alpha,usq}$\\ RegERM\end{tabular}} & \textbf{MUB} & \textbf{S-W0-1} \\ \hline
					$(n,m_+,m_-)$& \textbf{$\alpha$} & \textbf{} & \textbf{\begin{tabular}[c]{@{}c@{}} \textit{Lk-fun} \\ with $l_{sq}$\end{tabular}} & \textbf{\begin{tabular}[c]{@{}c@{}}\textit{Lk-fun} \\ with $l_{log}$\end{tabular}} & \textbf{LSPC} & \textbf{KLIEP} & \textbf{} & \textbf{$l_{log}$} & \textbf{$l_{log}$} \\ \hline
					\textbf{\begin{tabular}[c]{@{}c@{}}Breast\\ Cancer \\ (9,77,186)\end{tabular}} & 0.2 & \begin{tabular}[c]{@{}c@{}}0.0 \\ 0.2\\ 0.4\end{tabular} & \begin{tabular} [c]{@{}c@{}}0.61$\pm$ 0.06 \\ 0.62$\pm$0.06 \\ 0.49$\pm$0.04\end{tabular} & \begin{tabular}[c]{@{}c@{}}0.63$\pm$0.06 \\ \textbf{0.64$\pm$0.05} \\ 0.50$\pm$0.03 \end{tabular} & \begin{tabular}[c]{@{}c@{}}  0.62 $\pm$ 0.07 \\ 0.60$\pm$0.07\\ 0.54$\pm$0.05\end{tabular} & \begin{tabular}[c]{@{}c@{}}0.61$\pm$0.07 \\ 0.56$\pm$0.09\\ 0.52$\pm$0.05 \end{tabular} & \begin{tabular}[c]{@{}c@{}}\textbf{0.64$\pm$0.06} \\  0.59$\pm$0.04\\ \textbf{0.56$\pm$0.05}\end{tabular} & \begin{tabular}[c]{@{}c@{}}- \\0.59\\ 0.51\end{tabular} & \begin{tabular}[c]{@{}c@{}}- \\ \textbf{0.63}\\ \textbf{0.56}\end{tabular} \\ \hline
					\textbf{\begin{tabular}[c]{@{}c@{}}Pima\\ Diabetes \\ (8,268,500)\end{tabular}} & 0.25 & \begin{tabular}[c]{@{}c@{}}0.0 \\ 0.2\\ 0.4\end{tabular} & \begin{tabular}[c]{@{}c@{}}0.74$\pm$0.02 \\0.71$\pm$0.02\\ \textbf{0.67$\pm$0.06} \end{tabular} & \begin{tabular}[c]{@{}c@{}}0.74$\pm$0.02 \\ 0.72$\pm$0.02 \\ \textbf{0.68$\pm$0.07}\end{tabular} & \begin{tabular}[c]{@{}c@{}}0.71$\pm$0.03 \\  0.71$\pm$0.02\\ 0.54$\pm$0.03\end{tabular} & \begin{tabular}[c]{@{}c@{}}0.70$\pm$0.03 \\ 0.65$\pm$0.01\\ 0.57$\pm$0.06 \end{tabular} & \begin{tabular}[c]{@{}c@{}}\textbf{0.75$\pm$0.02} \\ \textbf{0.73$\pm$0.03}\\ \textbf{0.68$\pm$0.05}\end{tabular} &  \begin{tabular}[c]{@{}c@{}}- \\ 0.63\\ 0.56\end{tabular}& \begin{tabular}[c]{@{}c@{}}- \\ \textbf{0.74}\\ \textbf{0.67}\end{tabular} \\ \hline
					\textbf{\begin{tabular}[c]{@{}c@{}}German \\ (20,300,700)\end{tabular}} & 0.16 & \begin{tabular}[c]{@{}c@{}}0.0 \\ 0.2\\ 0.4  \end{tabular} & \begin{tabular}[c]{@{}c@{}}0.7$\pm$0.02 \\ 0.67$\pm$0.02\\ 0.57$\pm$0.06\end{tabular} & \begin{tabular}[c]{@{}c@{}}0.69$\pm$ 0.02 \\ 0.66$\pm$0.01\\ 0.56$\pm$0.06\end{tabular} & \begin{tabular}[c]{@{}c@{}} \textbf{ 0.71$\pm$0.03} \\ 0.63$\pm$0.03\\ 0.57$\pm$0.03\end{tabular} & \begin{tabular}[c]{@{}c@{}}0.67$\pm$0.02 \\0.59$\pm$0.04\\ 0.53$\pm$0.04 \end{tabular} & \begin{tabular}[c]{@{}c@{}}0.69$\pm$0.03 \\ 0.65$\pm$0.04\\ \textbf{0.60$\pm$0.03}\end{tabular} & \begin{tabular}[c]{@{}c@{}}- \\ 0.67\\ 0.51\end{tabular} & \begin{tabular}[c]{@{}c@{}}- \\ \textbf{0.69}\\ 0.56\end{tabular} \\ \hline
					\textbf{\begin{tabular}[c]{@{}c@{}}Thyroid \\ (5,65,150)\end{tabular}} & 0.3 & \begin{tabular}[c]{@{}c@{}}0.0\\ 0.0\\ 0.4  \end{tabular} & \begin{tabular}[c]{@{}c@{}}0.78$\pm$0.06 \\ 0.78$\pm$0.06 \\ 0.63$\pm$0.14\end{tabular} & \begin{tabular}[c]{@{}c@{}}0.84$\pm$0.06 \\ 0.80$\pm$0.05\\ 0.62$\pm$0.11 \end{tabular} & \begin{tabular}[c]{@{}c@{}}\textbf{0.92$\pm$0.04} \\ \textbf{0.85$\pm$0.06}\\ 0.64$\pm$0.06\end{tabular} & \begin{tabular}[c]{@{}c@{}}0.83$\pm$0.05 \\ 0.72$\pm$0.18\\ 0.63$\pm$0.19\end{tabular} & \begin{tabular}[c]{@{}c@{}}0.78$\pm$0.05 \\ 0.75$\pm$0.05\\ \textbf{0.74$\pm$0.08}\end{tabular} & \begin{tabular}[c]{@{}c@{}}- \\ 0.82\\ 0.53\end{tabular} & \begin{tabular}[c]{@{}c@{}}- \\ 0.78\\ 0.72\end{tabular}  \\ \hline
		\end{tabular}}}
		\egroup
		\caption[AM from \ref{alg: eta_cost_SLN} and $l_{\alpha,usq}$ on real datasets]{Averaged AM $(\pm~s.d.)$ of the cost sensitive predictions made by \ref{alg: eta_cost_SLN} and $l_{\alpha,usq}$ based regularized ERM on UCI datasets corrupted by uniform noise. AM values for \textbf{MUB} and \textbf{S-W0-1} are not available at $\rho=0.0$.}
		\label{tab: AM_measure_cost}
	\end{table}
	\subsubsection{Class balanced and domain requirement of cost}
	We consider the Bupa dataset \cite{UCI_dataset} with $6$ features where the label is $+1$ if the value of feature $6$ is greater than $3$ otherwise $-1$ ($m_+ = 176, m_- = 169$). We learn  a cost sensitive classifier by implementing \ref{alg: eta_cost_SLN} with $\alpha = 0.25$ and $\tilde{\eta}$ estimated using \textit{Lk-fun} with logistic loss.
	When comparing tuned $\gamma$ case (from the set $\Gamma = \{0.5,0.8,1,1.2,1.5\}$) and  $\gamma=1$ case, we observed in Table \ref{tab: Re-lg-ERM-Bupa} that tuning $\gamma$ is favorable for all measures. 
	Increase in noise rates from $(0.1/0.2/0.3)$ to $0.4$ leads to change in $\gamma$ values from $0.5$ to $1$ or $1.5$ due to imbalance induced by noise. 
	Results for $\tilde{\eta}$ estimated using \textit{Lk-fun} with $l_{sq}$ are available in Supplementary material E.2.
	\vspace{-0.5cm}
	\begin{table}[!htbp]
		\centering
		\bgroup
		\def\arraystretch{0.99}%
		{\setlength{\tabcolsep}{0.22em}		
			{ \scriptsize
				\begin{tabular}{|c|c|c|c|c|c|c|c|c|}
					\hline
					& \multicolumn{8}{c|}{ \ref{alg: eta_cost_SLN}: $\tilde{\eta}$ estimate from \textit{Lk-fun} with $l_{log}$ } \\
					\hline
					& \multicolumn{4}{c|}{\textbf{$\gamma$ tuned}} & \multicolumn{4}{c|}{\textbf{$\gamma = 1$}} \\ \hline
					\textbf{$\rho$} & \textbf{Acc} & \textbf{AM} & \textbf{F} & \textbf{WC} & \textbf{Acc} & \textbf{AM} & \textbf{F} & \textbf{WC} \\ \hline
					0.0 & \textbf{1.0$\pm$0.0} & \textbf{1.0$\pm$0.0} & \textbf{1.0$\pm$0.0} & \textbf{0.0$\pm$0.0} & 0.99$\pm$0.00 & 0.99$\pm$0.005 & 0.99$\pm$0.005 & 0.05$\pm$0.1 \\ \hline
					0.1 & \textbf{0.96$\pm$0.02} & \textbf{0.96$\pm$0.03} & \textbf{0.96$\pm$0.04} & \textbf{2.36$\pm$1.01} & 0.88$\pm$0.88 & 0.88$\pm$0.07 & 0.89$\pm$0.05 & \textbf{2.22$\pm$1.29} \\ \hline
					0.2 & \textbf{0.91$\pm$0.03} & \textbf{0.90$\pm$0.03} & \textbf{0.91$\pm$0.02} & \textbf{4.09$\pm$1.33} & 0.68$\pm$0.11 & 0.68$\pm$0.11 & 0.77$\pm$0.06 & 5.95$\pm$1.88 \\ \hline
					0.3 & \textbf{0.88$\pm$0.03} & \textbf{0.88$\pm$0.03} & \textbf{0.82$\pm$0.1} & \textbf{5.47$\pm$0.32} & 0.55$\pm$0.05 & 0.55$\pm$0.05 & 0.69$\pm$0.02 & 8.02$\pm$0.96 \\ \hline
					0.4 & \textbf{0.61$\pm$0.11} & \textbf{0.58$\pm$0.09} & \textbf{0.67$\pm$0.00} & \textbf{5.64$\pm$0.29} & 0.51$\pm$0.01 & 0.51$\pm$0.01 & \textbf{0.67$\pm$0.01} & 8.3$\pm$0.475 \\ \hline
		\end{tabular}}}
		\egroup
		\caption{Dataset: BUPA Liver Disorder $(6,176,169)$. The above table depicts the performance of \ref{alg: eta_cost_SLN} when $\tilde{\eta}$ estimate is from \textit{Lk-fun} with $l_{log}$ based regularized ERM  ($\lambda$ tuned via CV). Here, $\alpha = 0.25$ and $\gamma$ is tuned.} 
		\label{tab: Re-lg-ERM-Bupa}
	\end{table}	
	\section{Discussion} \label{sec: discussion}
	\vspace{-0.3cm}
	We considered the binary classification  problem of cost sensitive learning in the presence of uniform label noise. We are interested in the scenarios where there can be two separate costs: $\alpha$, the one fixed due to domain requirement and $\gamma$ that can be tuned to capture class imbalance. We first show that weighted $0$-$1$ loss is neither uniform noise robust with measurable class of classifiers nor with linear classifiers. In spite of this, we propose two schemes to address the problem in consideration, without requiring the knowledge of true noise rates.
	
	For the first scheme, we show that linear classifiers obtained using weighted uneven margin squared loss $l_{\alpha,usq}$ is uniform noise robust and incorporates cost sensitivity. This classifier is obtained by solving $l_{\alpha,usq}$ based regularized ERM, that only requires a matrix inversion.
	Also, a performance bound with respect to clean distribution for such a classifier is provided. One possible issue here could be weighted uneven margin squared loss function's susceptibility to outliers. However, in our experiments, this scheme performed well. The second scheme is a re-sampling based algorithm using the corrupted in-class probability estimates. It handles class imbalance in the presence of uniform label noise but it can be heavily influenced by the quality of $\tilde{\eta}$ estimates. However, we chose the $\tilde{\eta}$ estimation method based on their label prediction ability (as discussed in Supplementary material Section F.6) and obtained good empirical results. 
	
	We provide empirical evidence for the performance of our schemes. We observed that the Accuracy and AM values for our schemes are comparable to those given by \cite{natarajan18cost}. Our schemes have other measures like F and Weighted Cost comparable to that of SVMs trained on clean data; details are in Supplementary material Section E.1.
	The two proposed schemes have comparable performance measures among themselves.
	
	An interesting direction to explore would be cost sensitive learning in Class Conditional Noise (CCN) models without using/tuning noise rates $\rho_+$ and $\rho_-$. 
	We could invoke Elkan's result  \cite{elkan2001foundations} in the re-sampling scheme due to the uniform nature of noise. However, the estimation of $\tilde{\eta}$ induces an iota of approximation bias in the scheme.
	A theoretical proof for correctness of re-sampling scheme  could provide directions in dealing with CCN models,   including perhaps a bound on the error in SLN models; see Remark \ref{REM-Resampling}.   Even though the constant term appears in the upper bound of Theorem \ref{thm: perf_bound_f_from_usq}, it would be interesting to understand whether it is due to the affect of noise on the systemic component of risk or something else. Also, using synthetic data, one can comment on the tightness of the bound. A more fundamental question would be ``Is lack of guarantee on consistency a price that one has to always pay while learning to get away without requiring the knowledge of noise rate?" 
	\vspace{-0.3cm}
	\bibliographystyle{plain}
	\bibliography{PAKDD19}
	\newpage
	\appendix
	\hrule 
	\begin{center}
		{\Large \textbf{ Supplementary Material}} \\
		{\large Cost Sensitive Learning in the Presence of Symmetric Label Noise}
	\end{center}
	\hrule 
	\section{Proofs}
	\subsection{Proof of Theorem \ref{the: un-sq robustness}} \label{app: un-sq robustness}
	\begin{proof}
		Let the linear classifier be of the form $\mathbf{x}^T\mathbf{w}_{\alpha} + b = \tilde{\mathbf{w}}_{\alpha}^T\tilde{\mathbf{x}}$ where $\tilde{\mathbf{w}}_{\alpha} = [\mathbf{w}_{\alpha} ~b]^T$ and $\tilde{\mathbf{x}}= [\mathbf{x} ~1]^T$ trained on clean data $D$ and $\tilde{\mathbf{w}}_{\rho,\alpha}$ trained on $\tilde{D}$ with noise rate $\rho.$ %
		For $\lambda >0$, the weighted clean regularized risk is
		{\scriptsize
			\begin{eqnarray} \label{eq: weighted uneven squared risk clean} \nonumber
			R^r_{D,l_{\alpha,usq}}(\tilde{\mathbf{w}}_{\alpha}) &=& E_{D}[(1-\alpha)\mathbf{1}_{\{y=1\}}(f(\mathbf{x})-y)^2 + \alpha\mathbf{1}_{\{y=-1\}}\frac{1}{\gamma}( \gamma f(\mathbf{x}) - y )^2] + \lambda\Vert f \Vert_2^2 \\
			&=& E_{D}[(1-\alpha)\mathbf{1}_{\{y=1\}}(\tilde{\mathbf{x}}^T\tilde{\mathbf{w}}_{\alpha}-y)^2 + \alpha\mathbf{1}_{\{y=-1\}}\frac{1}{\gamma}(\gamma\tilde{\mathbf{x}}^T\tilde{\mathbf{w}}_{\alpha}-y )^2] + \lambda\Vert \tilde{\mathbf{w}}_{\alpha} \Vert_2^2
			\end{eqnarray}}
		Differentiating \eqref{eq: weighted uneven squared risk clean} w.r.t. $\tilde{\mathbf{w}}_{\alpha}$ and equating to $0$, we get
		\begin{small}
			\begin{eqnarray*} \nonumber
				E_{D}[2\tilde{\mathbf{x}}(1-\alpha)\mathbf{1}_{\{y=1\}}(\tilde{\mathbf{x}}^T\tilde{\mathbf{w}}_{\alpha}-y) + 2\tilde{\mathbf{x}}\alpha\mathbf{1}_{\{y=-1\}}(\gamma\tilde{\mathbf{x}}^T\tilde{\mathbf{w}}_{\alpha}- y )] + 2\lambda \tilde{\mathbf{w}}_{\alpha}= 0 \\ \nonumber
				E_{D}[(1-\alpha)\mathbf{1}_{\{y=1\}}(\tilde{\mathbf{x}}\tilde{\mathbf{x}}^T\tilde{\mathbf{w}}_{\alpha}-\tilde{\mathbf{x}}y) + \alpha\mathbf{1}_{\{y=-1\}}(\gamma\tilde{\mathbf{x}}\tilde{\mathbf{x}}^T\tilde{\mathbf{w}}_{\alpha}- \tilde{\mathbf{x}}y )] + \lambda \tilde{\mathbf{w}}_{\alpha} = 0 
			\end{eqnarray*}
		\end{small}
		{\small
			\begin{eqnarray} \nonumber
			\tilde{\mathbf{w}}^*_{\alpha} &= [E_D[(1-\alpha)\mathbf{1}_{\{y=1\}}\tilde{\mathbf{x}}\tilde{\mathbf{x}}^T + \gamma\alpha\mathbf{1}_{\{y=-1\}}\tilde{\mathbf{x}}\tilde{\mathbf{x}}^T]  \\ \label{w_cost_clean}
			&+ \lambda \mathbf{I}]^{-1}E_D[(1-\alpha)\mathbf{1}_{\{y=1\}}\tilde{\mathbf{x}}y + \alpha\mathbf{1}_{\{y=-1\}}\tilde{\mathbf{x}}y]
			\end{eqnarray}}
		
		where $\mathbf{I}$ is ${n+1\times n+1}$ dimensional identity matrix. 
		Now, consider the weighted corrupted regularized risk as:
		{ \scriptsize
			\begin{eqnarray} \label{eq: weighted uneven squared risk noisy} \nonumber
			R^r_{\tilde{D},l_{\alpha,usq}}(\tilde{\mathbf{w}}_{\rho,\alpha}) &=& E_{\tilde{D}}[(1-\alpha)\mathbf{1}_{\{\tilde{y}=1\}}(\tilde{\mathbf{x}}^T\tilde{\mathbf{w}}_{\rho,\alpha}-\tilde{y})^2 + \alpha\mathbf{1}_{\{\tilde{y}=-1\}}\frac{1}{\gamma}( \gamma \tilde{\mathbf{x}}^T\tilde{\mathbf{w}}_{\rho,\alpha} - \tilde{y} )^2] + \lambda\Vert \tilde{\mathbf{w}}_{\rho,\alpha} \Vert_2^2 \\ \nonumber
			&=& (1-\rho)E_{D}[(1-\alpha)\mathbf{1}_{\{y=1\}}(\tilde{\mathbf{x}}^T\tilde{\mathbf{w}}_{\rho, \alpha}-y)^2 + \alpha\mathbf{1}_{\{y=-1\}}\frac{1}{\gamma}(\gamma\tilde{\mathbf{x}}^T\tilde{\mathbf{w}}_{\rho, \alpha}-y )^2]   \\
			& & + \rho E_{D}[(1-\alpha)\mathbf{1}_{\{y=1\}}(\tilde{\mathbf{x}}^T\tilde{\mathbf{w}}_{\rho, \alpha}+y)^2 + \alpha\mathbf{1}_{\{y=-1\}}\frac{1}{\gamma}(\gamma\tilde{\mathbf{x}}^T\tilde{\mathbf{w}}_{\rho, \alpha}+y )^2] + \lambda\Vert \tilde{\mathbf{w}}_{\rho,\alpha} \Vert_2^2
			\end{eqnarray}}
		Differentiating \eqref{eq: weighted uneven squared risk noisy} w.r.t. $\tilde{\mathbf{w}}_{\rho,\alpha}$, we get
		{ \scriptsize
			\begin{align*}
			& (1-\rho)E_{D}[2\tilde{\mathbf{x}}(1-\alpha)\mathbf{1}_{\{y=1\}}(\tilde{\mathbf{x}}^T\tilde{\mathbf{w}}_{\rho, \alpha}-y) + 2\tilde{\mathbf{x}}\alpha\mathbf{1}_{\{y=-1\}}(\gamma\tilde{\mathbf{x}}^T\tilde{\mathbf{w}}_{\rho, \alpha}-y )] \\ &+ \rho E_{D}[2\tilde{\mathbf{x}}(1-\alpha)\mathbf{1}_{\{y=1\}}(\tilde{\mathbf{x}}^T\tilde{\mathbf{w}}_{\rho, \alpha}+y) + 2\tilde{\mathbf{x}} \alpha\mathbf{1}_{\{y=-1\}}(\gamma\tilde{\mathbf{x}}^T\tilde{\mathbf{w}}_{\rho, \alpha}+y )] + 2\lambda \Vert \tilde{\mathbf{w}}_{\rho,\alpha} \Vert= 0 \\
			&(1-\rho)E_D[(1-\alpha)\mathbf{1}_{\{y=1\}}\tilde{\mathbf{x}}\tilde{\mathbf{x}}^T\tilde{\mathbf{w}} + \gamma\alpha\mathbf{1}_{\{y=-1\}}\tilde{\mathbf{x}}\tilde{\mathbf{x}}^T\tilde{\mathbf{w}}] + \rho E_D[(1-\alpha)\mathbf{1}_{\{y=1\}}\tilde{\mathbf{x}}\tilde{\mathbf{x}}^T\tilde{\mathbf{w}} + \gamma \alpha\mathbf{1}_{\{y=-1\}}\tilde{\mathbf{x}}\tilde{\mathbf{x}}^T\tilde{\mathbf{w}}] \\
			&+ \lambda \Vert \tilde{\mathbf{w}}_{\rho,\alpha} \Vert 
			= (1-\rho)E_D[(1-\alpha)\mathbf{1}_{\{y=1\}}\tilde{\mathbf{x}}y + \alpha\mathbf{1}_{\{y=-1\}}\tilde{\mathbf{x}}y] -\rho E_{D}[(1-\alpha)\mathbf{1}_{\{y=1\}}\tilde{\mathbf{x}}y + \alpha\mathbf{1}_{\{y=-1\}}\tilde{\mathbf{x}}y]
			\end{align*}}
		Therefore, we get the optimal noisy weighted classifier as
		{\scriptsize
			\begin{eqnarray} \label{w_cost_noisy} \nonumber
			\tilde{\mathbf{w}}^*_{\rho,\alpha} &=& (1-2\rho) [E_D[(1-\alpha)\mathbf{1}_{\{y=1\}}\tilde{\mathbf{x}}\tilde{\mathbf{x}}^T + \gamma\alpha\mathbf{1}_{\{y=-1\}}\tilde{\mathbf{x}}\tilde{\mathbf{x}}^T]+ \lambda \mathbf{I}]^{-1}E_D[(1-\alpha)\mathbf{1}_{\{y=1\}}\tilde{\mathbf{x}}y + \alpha\mathbf{1}_{\{y=-1\}}\tilde{\mathbf{x}}y] \\
			&=& (1-2\rho)\mathbf{w}^*_{\alpha}
			\end{eqnarray}}
		Taking the transpose in above equation and post multiplying by $\tilde{\mathbf{x}}$, we get $\tilde{\mathbf{w}}^{*T}_{\rho,\alpha}\tilde{\mathbf{x}} = (1-2\rho)\mathbf{w}^{*T}_{\alpha}\tilde{\mathbf{x}} $ and hence,
		\begin{eqnarray} \label{eq: f_sq_cost_clean_noisy}
		\tilde{f}l^*_{r,l_{\alpha,usq}} = (1-2\rho)fl^*_{r,l_{\alpha,usq}},
		\end{eqnarray}
		where $fl^*_{r, l_{\alpha,usq}}$ and $\tilde{f}l^*_{r, l_{\alpha,usq}}$ are the optimal liner classifiers learnt using $l_{\alpha,usq}$ on $D$ and $\tilde{D}$ respectively.
		Since the noise rate $\rho < 0.5$, $sign(\tilde{f}l^*_{r,l_{\alpha,usq}})$ = $sign(fl^*_{r,l_{\alpha,usq}})$. This implies that \eqref{eq: alpha_noise_robustness} holds for $(l_{\alpha,usq},\mathcal{H}_{lin}).$
	\end{proof} 
	
	\subsection{Proof of Proposition \ref{prop: l_sq_closed_form}} \label{app: l_sq_closed_form}
	\begin{proof}
		Consider the empirical noisy regularized risk $\hat{R}^r_{\tilde{D}, l_{\alpha,usq}}(f)$ as follows:
		{ \scriptsize
			\begin{eqnarray} \nonumber
			\hat{R}^r_{\tilde{D}, l_{\alpha,usq}}(f) &=& \frac{1}{m}\sum\limits_{i=1}^{m}[\mathbf{1}_{[\tilde{y}_i =1]}(\mathbf{w}^T\mathbf{x}_i +b -y_i)^2(1-\alpha) + \frac{\alpha}{\gamma}\mathbf{1}_{[\tilde{y}_i = -1]}(\gamma \mathbf{w}^T\mathbf{x}_i + \gamma b -y_i)^2] + \lambda \Vert (\mathbf{w},b) \Vert_{2}^2 \\ \label{eq: emp-risk-sq-uneven}
			&=& \frac{1}{m}\sum\limits_{i=1}^{m}[\mathbf{1}_{[\tilde{y}_i =1]}(\mathbf{w}^T\mathbf{x}_i +b -1)^2(1-\alpha) + \frac{\alpha}{\gamma}\mathbf{1}_{[\tilde{y}_i = -1]}(\gamma \mathbf{w}^T\mathbf{x}_i + \gamma b +1)^2] + \lambda \Vert (\mathbf{w},b) \Vert_{2}^2
			\end{eqnarray}}
		Differentiating equation \eqref{eq: emp-risk-sq-uneven} w.r.t. $w_{j}$ and equating to $0$, we get
		{ \scriptsize
			\begin{equation} \label{eq: der-w-sq}
			\frac{1}{m}\sum\limits_{i=1}^{m}x_{ij}\left[ 
			\left(\sum\limits_{k=1}^n w_k x_{ik} + b\right)(\mathbf{1}_{[\tilde{y}_i=1]}(1-\alpha) +\gamma\alpha\mathbf{1}_{[\tilde{y}_i=-1]}) - (\mathbf{1}_{[\tilde{y}_i=1]}(1-\alpha) - \alpha\mathbf{1}_{[\tilde{y}_i=-1]})\right] + \lambda w_j= 0
			\end{equation}}
		Differentiating equation \eqref{eq: emp-risk-sq-uneven} w.r.t. $b$ and equating to $0$, we get
		{ \scriptsize
			\begin{equation} \label{eq: der-b-sq}
			\frac{1}{m}\sum\limits_{i=1}^{m}\left[ 
			\left(\sum\limits_{k=1}^n w_k x_{ik} + b\right)(\mathbf{1}_{[\tilde{y}_i=1]}(1-\alpha) +\gamma\alpha\mathbf{1}_{[\tilde{y}_i=-1]}) - (\mathbf{1}_{[\tilde{y}_i=1]}(1-\alpha) - \alpha\mathbf{1}_{[\tilde{y}_i=-1]})\right] + \lambda b= 0
			\end{equation}}
		
		Let $a_i = (\mathbf{1}_{[\tilde{y}_i=1]}(1-\alpha) +\gamma\alpha\mathbf{1}_{[\tilde{y}_i=-1]})$ and $c_i = (\mathbf{1}_{[\tilde{y}_i=1]}(1-\alpha) - \alpha\mathbf{1}_{[\tilde{y}_i=-1]})$. Writing equation \eqref{eq: der-w-sq} and \eqref{eq: der-b-sq} in matrix form, we get the following system of linear equations:
		{ \footnotesize
			\begin{eqnarray} 
			{\underbrace{\left[ {\begin{array}{ccccc}
						\sum\limits_{i=1}^m x_{i1}^2a_i + \lambda & \sum\limits_{i=1}^m x_{i1}x_{i2}a_i & \ldots & \sum\limits_{i=1}^m x_{i1}x_{in}a_i & \sum\limits_{i=1}^m x_{i1}a_i \\
						\sum\limits_{i=1}^m x_{i2}x_{i1}a_i & \sum\limits_{i=1}^m x_{i2}^2a_i + \lambda & \ldots & \sum\limits_{i=1}^m x_{i2}x_{in}a_i & \sum\limits_{i=1}^m x_{i2}a_i\\
						\vdots & \vdots & \ddots & \vdots & \vdots\\
						\sum\limits_{i=1}^m x_{in}x_{i1}a_i & \sum\limits_{i=1}^m x_{in}x_{i2}a_i & \ldots & \sum\limits_{i=1}^m x_{in}^2a_i + \lambda & \sum\limits_{i=1}^m x_{in}a_i\\
						\sum\limits_{i=1}^m x_{i1}a_i & \sum\limits_{i=1}^m x_{i2}a_i & \ldots & \sum\limits_{i=1}^m x_{in}a_i & \sum\limits_{i=1}^m a_i + \lambda \\
						\end{array} } \right]}_{A + \lambda \mathbf{I}} } 
			\underbrace{\begin{bmatrix}
				w_1 \\
				w_2 \\
				\vdots \\
				w_n \\
				b
				\end{bmatrix}}_{\bar{\mathbf{w}}} 
			= \underbrace{\begin{bmatrix}
				\sum\limits_{i=1}^{m}x_{i1}c_i \\
				\sum\limits_{i=1}^{m}x_{i2}c_i  \\
				\vdots \\
				\sum\limits_{i=1}^{m}x_{in}c_i  \\
				\sum\limits_{i=1}^{m}c_i 
				\end{bmatrix}}_{\mathbf{c}} 
			\end{eqnarray}}
		where $A$ is a $(n+1 \times n+1)$ dimensional known symmetric matrix, $\bar{\mathbf{w}}$ is $n+1$ dimensional vector of variables and $\mathbf{c}$ is $n+1$ dimensional known vector. The above linear system of equations can be written as $ (A+ \lambda \mathbf{I})\bar{\mathbf{w}} = \mathbf{c}$ and the cost-sensitive and $(\alpha,\gamma,\rho)$-robust classifier $\hat{f}_r$ is as follows:
		\begin{equation}\label{eq: f-emp-uneven-sq}
		\hat{f}_r = \bar{\mathbf{w}}^* = (A + \lambda \mathbf{I})^{-1}\mathbf{c}
		\end{equation}
	\end{proof}
	
	\subsection{Proof of Lemma \ref{lem: max_dev_Rad}} \label{app: max_dev_Rad}
	\begin{proof}
		Consider the difference between corrupted true risk and corrupted empirical risk, i.e.,
		$R_{\tilde{D},l_{\alpha,usq}}(f) - \hat{R}_{\tilde{D},l_{\alpha,usq}}(f)$ 
		Now, using the Rademacher bound (\cite{Mohri}) on the maximal deviation between risks and empirical risks over $f \in \mathcal{H}_{lin}$, we get with probability at least $1-\delta$
		\begin{equation}
		\max\limits_{f\in \mathcal{H}_{lin}}|\hat{R}_{\tilde{D}, l_{\alpha,usq}}(f) - R_{\tilde{D},l_{\alpha,usq}}(f)| \leq 2\mathfrak{R}(l_{\alpha,usq} \circ \mathcal{H}_{lin}) + \sqrt{\frac{log(1/\delta)}{2m}},
		\end{equation}
		where $\mathfrak{R}(l_{\alpha,usq} \circ \mathcal{H}_{lin}) := E_{\tilde{D},\sigma}\left[\sup\limits_{f \in \mathcal{H}_{lin}}\frac{1}{m}\sum\limits_{i=1}^{m}\sigma_i l_{\alpha,usq}(f(\mathbf{x}_i,\tilde{y}_i))\right].$
		Now, using the Lipschitz composition property of the Rademacher averages, we have
		$ \mathfrak{R}(l_{\alpha,usq}\circ \mathcal{H}_{lin})\leq L \mathfrak{R}(\mathcal{H}_{lin}).$ Hence, the inequality in lemma statement follows.
	\end{proof}	
	
	\subsection{Proof of Lemma \ref{lem: R_noisy_clean_l_sq}} \label{app: R_noisy_clean_l_sq}
	\begin{proof}
		Consider the corrupted $l_{\alpha,usq}$ risk of a classifier $f \in \mathcal{H}$ as given in equation \eqref{eq: weighted uneven squared risk noisy}.
		{ \footnotesize
			\begin{eqnarray*}
				R_{\tilde{D},l_{\alpha,usq}}(f) &=& (1-\rho)E_{D}[(1-\alpha)\mathbf{1}_{\{y=1\}}(f(\mathbf{x})-y)^2 + \alpha\mathbf{1}_{\{y=-1\}}\frac{1}{\gamma}(\gamma f(\mathbf{x})-y )^2]  \\ 
				& & + ~ \rho E_{D}[(1-\alpha)\mathbf{1}_{\{y=1\}}(f(\mathbf{x})+y)^2 + \alpha\mathbf{1}_{\{y=-1\}}\frac{1}{\gamma}(\gamma f(\mathbf{x})+y )^2]  \\
				&=& E_{D}[(1-\alpha)\mathbf{1}_{\{y=1\}}(f(\mathbf{x})-y)^2 + \alpha\mathbf{1}_{\{y=-1\}}\frac{1}{\gamma}(\gamma f(\mathbf{x})-y )^2] \\
				& +& \rho E_D[(1-\alpha)\mathbf{1}_{\{y=1\}}(f(\mathbf{x})+y)^2 + \alpha\mathbf{1}_{\{y=-1\}}\frac{1}{\gamma}(\gamma f(\mathbf{x})+y )^2  \\
				&- & (1-\alpha)\mathbf{1}_{\{y=1\}}(f(\mathbf{x})-y)^2 - \alpha\mathbf{1}_{\{y=-1\}}\frac{1}{\gamma}(\gamma f(\mathbf{x})-y )^2] \\
				&=& R_{D,l_{\alpha,usq}}(f) + 4\rho E_{D}[yf(\mathbf{x})[(1-\alpha)\mathbf{1}_{\{y=1\}} + \alpha\mathbf{1}_{\{y=-1\}}] \\
		\end{eqnarray*}}
	\end{proof}
	
	\subsection{Proof of Theorem \ref{thm: perf_bound_f_from_usq}} \label{app: perf_bound_f_from_usq}
	\begin{proof}
		Let $fl^*_{l_{\alpha,usq}} \in \mathcal{H}_{lin}$ be the minimizer of $R_{D,l_{\alpha,usq}}$. Then, the excess $D$ risk of $\hat{f}_r$ is 
		\begin{eqnarray*}
			& & R_{D,l_{\alpha,usq}}(\hat{f}_r) - R_{D,l_{\alpha,usq}}(fl^*_{r,l_{\alpha,usq}}) \\
			&=&  R_{D,l_{\alpha,usq}}(\hat{f}_r) - \hat{R}^r_{\tilde{D},l_{\alpha,usq}}(\hat{f}_r) + \hat{R}^r_{\tilde{D},l_{\alpha,usq}}(\hat{f}_r) - \hat{R}^r_{\tilde{D},l_{\alpha,usq}}(fl^*_{\alpha,usq}) \\
			& & +  \hat{R}^r_{\tilde{D},l_{\alpha,usq}}(fl^*_{\alpha,usq}) - R_{D,l_{\alpha,usq}}(fl^*_{r,l_{\alpha,usq}}) \\
			& \leq&  0 + R_{D,l_{\alpha,usq}}(\hat{f}_r) - \hat{R}^r_{\tilde{D},l_{\alpha,usq}}(\hat{f}_r) + \hat{R}^r_{\tilde{D},l_{\alpha,usq}}(fl^*_{\alpha,usq}) - R_{D,l_{\alpha,usq}}(fl^*_{r,l_{\alpha,usq}})
		\end{eqnarray*}
		The last inequality holds because $\hat{f}_r$ is the minimizer of $\hat{R}^r_{\tilde{D},l_{\alpha,usq}}(f)$.  
		Using Lemma \ref{lem: R_noisy_clean_l_sq}, we get
		{ \scriptsize
			\begin{eqnarray*}
				& & R_{D,l_{\alpha,usq}}(\hat{f}_r) - R_{D,l_{\alpha,usq}}(fl^*_{r,l_{\alpha,usq}}) \\
				&\leq& R_{\tilde{D},l_{\alpha,usq}}(\hat{f}_r) - 4\rho E_{D}[\hat{f}_r(\mathbf{x})[(1-\alpha)\mathbf{1}_{[y=1]} - \alpha\mathbf{1}_{[y=-1]}]] - \hat{R}^r_{\tilde{D},l_{\alpha,usq}}(\hat{f}_r) \\
				&-& R_{\tilde{D},l_{\alpha,usq}}(fl^*_{l_{\alpha,usq}}) + 4\rho E_{D}[fl^*_{l_{\alpha,usq}}(\mathbf{x})[(1-\alpha)\mathbf{1}_{[y=1]} - \alpha\mathbf{1}_{[y=-1]}]] + \hat{R}^r_{\tilde{D},l_{\alpha,usq}}(fl^*_{l_{\alpha,usq}})  \\
				&\leq & 2\max\limits_{f\in \mathcal{H}_{lin}}|\hat{R}^r_{\tilde{D},l_{\alpha,usq}}(f) - R_{\tilde{D},l_{\alpha,usq}}(f)| + 4\rho E_{D}[((1-\alpha)\mathbf{1}_{[y=1]} + \alpha\mathbf{1}_{[y=-1]})y(fl^*_{l_{\alpha,usq}} - \hat{f}_r)]  \\
				&= & 2\max\limits_{f\in \mathcal{H}_{lin}}|\hat{R}_{\tilde{D},l_{\alpha,usq}}(f) - R_{\tilde{D},l_{\alpha,usq}}(f) + \lambda \Vert f\Vert_2^2| + 4\rho E_{D}[((1-\alpha)\mathbf{1}_{[y=1]} + \alpha\mathbf{1}_{[y=-1]})y(fl^*_{l_{\alpha,usq}} - \hat{f}_r)]  \\
				&\leq & 2\max\limits_{f\in \mathcal{H}_{lin}}|\hat{R}_{\tilde{D},l_{\alpha,usq}}(f) - R_{\tilde{D},l_{\alpha,usq}}(f)| + 2\lambda \max\limits_{f\in \mathcal{H}_{lin}} \Vert f\Vert_2^2  \\
				& & + 4\rho E_{D}[((1-\alpha)\mathbf{1}_{[y=1]} + \alpha\mathbf{1}_{[y=-1]})y(fl^*_{l_{\alpha,usq}} - \hat{f}_r)]  \\
				&\leq& 4L\mathfrak{R}(\mathcal{H}_{lin}) + 2\sqrt{\frac{log(1/\delta)}{2m}} + 2\lambda W^2 + 4\rho E_{D}\left[((1-\alpha)\mathbf{1}_{[y=1]} + \alpha\mathbf{1}_{[y=-1]})y \left(\frac{\tilde{f}l^*_{l_{\alpha,usq}}}{1-2\rho} - \hat{f}_r \right)\right]
		\end{eqnarray*}}
		Last second inequality follows by using Triangle's inequality for absolute deviations.	The last inequality follows from Lemma \ref{lem: max_dev_Rad} and equation \eqref{eq: f_sq_cost_clean_noisy}. Next, we simplify the term including expectation
		{ \footnotesize
			\begin{eqnarray*}
				& &	E_{D}\left[((1-\alpha)\mathbf{1}_{[y=1]} + \alpha\mathbf{1}_{[y=-1]})y \left(\frac{\tilde{f}l^*_{l_{\alpha,usq}}}{1-2\rho} - \hat{f}_r\right)\right]  \\
				&=& E_{\mathbf{X}}\left[\left(\frac{\tilde{f}l^*_{l_{\alpha,usq}}(\mathbf{x})}{1-2\rho} - \hat{f}_r(\mathbf{x})\right)E_{Y|{\mathbf{X}}}\left[y((1-\alpha)\mathbf{1}_{[y=1]} + \alpha\mathbf{1}_{[y=-1]})\right]\right] \\
				&=&E_{{\mathbf{X}}}\left[\left(\frac{\tilde{f}l^*_{l_{\alpha,usq}}(\mathbf{x})}{1-2\rho} - \hat{f}_r(\mathbf{x})\right)\left[ (1-\alpha)\eta(\mathbf{x}) -\alpha(1-\eta(\mathbf{x}))\right]\right] \\
				&=&E_{\mathbf{X}}\left[ \left(\frac{\tilde{f}l^*_{l_{\alpha,usq}}(\mathbf{x})}{1-2\rho} - \hat{f}_r(\mathbf{x})\right)(\eta(\mathbf{x}) - \alpha)\right] 
		\end{eqnarray*}}
		Substituting it back in the upper bound we get,
		\begin{eqnarray*}
			R_{D,l_{\alpha,usq}}(\hat{f}_r) &\leq& \min\limits_{f \in \mathcal{H}_{lin}}R_{D,l_{\alpha,usq}}(f) + 4L\mathfrak{R}(\mathcal{H}_{lin}) + 2\sqrt{\frac{log(1/\delta)}{2m}}  + 2\lambda W^2  \\ 
			& & + \frac{4\rho}{(1-2\rho)}E_{{\mathbf{X}}}[(\tilde{f}l^*_{l_{\alpha,usq}}(\mathbf{x}) -\hat{f}_r(\mathbf{x}))(\eta(\mathbf{x}) - \alpha)]  \\ 
			& & + \frac{8\rho^2}{(1-2\rho)}E_{{\mathbf{X}}}[\hat{f}_r(\mathbf{x})(\eta(\mathbf{x}) - \alpha)] ).
		\end{eqnarray*}
		Subtracting $\min\limits_{f\in \mathcal{H}}R_{D,l_{\alpha,usq}}(f)$ from both sides of above inequality and invoking equation \eqref{eq: alpha_CC_cond} we get the second result in the theorem statement.
	\end{proof}
	
	\subsection{Proof of Lemma \ref{lem: SLN_pi_mono}} \label{app: SLN_pi_mono}
	
	\begin{proof}
		Consider the case when $\pi < 0.5$ and $\rho < 0.5$ as follows:
		\begin{eqnarray*}
			\pi &<& 0.5  \Rightarrow (1-2\rho)\pi + \rho < (1-2\rho)0.5 + \rho \quad \text{ since } \rho <0.5 \\
			\pi &<& 0.5 \Rightarrow  \tilde{\pi}  < 0.5
		\end{eqnarray*}
		The case of $\pi > 0.5 \Rightarrow \tilde{\pi} > 0.5$ can be proved in similar lines as above by reversing the inequality.
		
		Since, $\tilde{\eta}(\mathbf{x}) = (1-2\rho)\eta(\mathbf{x})+\rho$ has the same coefficients as those relating $\pi$ and $\tilde{\pi}$, monotonicity w.r.t. $0.5$ holds for $\eta(\mathbf{x})$ too.
	\end{proof}
	
	\section{Details about various counter-examples}
	\subsection{Details of the counter-example \ref{exam: unifDis_noise}} \label{exam: unifDis_noise_details}
	
	Let $Y$ has a Bernoulli distribution with parameter $p=0.2$. Let $\mathbf{X} \subset \mathbb{R}$  be such that such that  $\mathbf{X}|Y=1 \sim Uniform(0,p)$ and $\mathbf{X}|Y=-1 \sim Uniform(1-p,1)$. Then, the in-class probability $\eta(\mathbf{x})$ is given as follows:
	{ \footnotesize
		\begin{eqnarray*}
			\eta(\mathbf{x}) = P(Y=1|\mathbf{X}=\mathbf{x}) &=& \frac{P(Y=1)P(\mathbf{X}=\mathbf{x}|Y=1)}{P(\mathbf{X}=\mathbf{x})} \\
			&=& \frac{P(Y=1)P(\mathbf{X}=\mathbf{x}|Y=1)}{P(Y=1)P(\mathbf{X}=\mathbf{x}|Y=1) + P(Y=-1)P(\mathbf{X}=\mathbf{x}|Y=-1)} \\
			&=& \frac{p\frac{1}{p}}{p\frac{1}{p} + (1-p)\frac{1}{p}} = p =0.2
	\end{eqnarray*}}
	Suppose uniform noise rate be $\rho = 0.3$. Then, $\tilde{\eta}(\mathbf{x}) = (1-2\rho)\eta(\mathbf{x})+\rho = 0.38$. Now, if we are given that positive class is more important with misclassification costs with $\alpha = 0.25$, then the corresponding $l_{0-1,\alpha}$ based optimal classifier on $D$ and $\tilde{D}$ are as follows:
	$$ f^*_{\alpha}(\mathbf{x}) = -1 \text{ and } \tilde{f}^*_{\alpha}(\mathbf{x}) = 1  ~~~ \forall \mathbf{x} \in \mathbf{X}.$$
	Consider the $\alpha$ weighted $0$-$1$ risk of the above classifiers  as below :
	{ \scriptsize
		\begin{eqnarray*}
			R_{D,\alpha}(f^*_{\alpha}) &=& E_{D}[l_{0-1,\alpha}(f^*_{\alpha}(\mathbf{x}),y)] \\
			&=& \sum\limits_{y}\int\limits_{\mathbf{x}}(\mathbf{1}_{\{f^*_{\alpha}(\mathbf{x}) > 0, y=-1\}}\alpha  + \mathbf{1}_{\{f^*_{\alpha}(\mathbf{x}) \leq 0, y=1\}}(1-\alpha)) P(\mathbf{X}=\mathbf{x}|Y=y)P(Y=y)d\mathbf{x} \\
			&=& \int\limits_{1-p}^{1} (\mathbf{1}_{\{f^*_{\alpha}(\mathbf{x}) > 0\}}\alpha) P(\mathbf{X}=\mathbf{x}|Y=-1)(1-p)d\mathbf{x}  +\int\limits_{0}^{p} (\mathbf{1}_{\{f^*_{\alpha}(\mathbf{x}) \leq 0\}}(1-\alpha)) P(\mathbf{X}=\mathbf{x}|Y=1)p d\mathbf{x}  \\
			&=& (1-\alpha)p  ~~~~~~~~~~~~~\text{ since } f^*_{\alpha}(\mathbf{x}) \leq 0 ~~ \forall \mathbf{x} \\
		\end{eqnarray*}
		\begin{eqnarray*}
			R_{D,\alpha}( \tilde{f}^*_{\alpha}) &=& E_{D}[l_{0-1,\alpha}(\tilde{f}^*_{\alpha}(\mathbf{x}),y)] \\
			&=& \sum\limits_{y}\int\limits_{\mathbf{x}}(\mathbf{1}_{\{\tilde{f}^*_{\alpha}(\mathbf{x}) > 0, y=-1\}}\alpha  + \mathbf{1}_{\{\tilde{f}^*_{\alpha}(\mathbf{x}) \leq 0, y=1\}}(1-\alpha)) P(\mathbf{X}=\mathbf{x}|Y=y)P(Y=y)d\mathbf{x} \\
			&=& \int\limits_{1-p}^{1} (\mathbf{1}_{\{\tilde{f}^*_{\alpha}(\mathbf{x}) > 0\}}\alpha) P(\mathbf{X}=\mathbf{x}|Y=-1)(1-p)d\mathbf{x} + \int\limits_{0}^{p} (\mathbf{1}_{\{\tilde{f}^*_{\alpha}(\mathbf{x}) \leq 0\}}(1-\alpha)) P(\mathbf{X}=\mathbf{x}|Y=1)pd\mathbf{x}  \\
			&=& \alpha (1-p)  ~~~~~~~~~~~~~\text{ since } \tilde{f}^*_{\alpha}(\mathbf{x}) > 0 ~~ \forall \mathbf{x}
	\end{eqnarray*}}
	Therefore, $R_{D,\alpha}( f^*_{\alpha}) \neq R_{D,\alpha}( \tilde{f}^*_{\alpha})$ implying that the $\alpha$ weighted $0$-$1$ loss function $l_{0-1,\alpha}$ is not uniform noise robust. Note that in this example $D$ is linearly separable because of $p < 0.5$. The counter-example suggests that even in the easiest case when the clean distribution is linearly separable, $l_{0-1,\alpha}$  is not uniform noise tolerant. If $p>0.5$, $D$ will be linearly inseparable and by changing $\alpha$ value another counter-example can be generated.
	
	\subsection{Details of counter-example \ref{exam: unifDis_noise_lin}} \label{exam: unifDis_noise_lin_details}
	Consider the training set $\{(3,-1),(8,-1),(12,1)\}$. Let the probability distribution on the feature space be uniformly concentrated on the training dataset. Let the linear classifier be of the form $fl = sign(\mathbf{x}+b)$. Let $\alpha = 0.3$ and the uniform noise be $\rho = 0.42$. Then,
	{ \footnotesize
		\begin{eqnarray*}
			fl^*_{\alpha} &=& \arg\min\limits_{fl} E_{D}[l_{0-1,\alpha}(y,fl)] \\
			&=& \arg\min\limits_{b} \frac{1}{3}[\alpha \mathbf{1}_{\{3+b > 0,y=-1\}} + \alpha \mathbf{1}_{\{8+b > 0,y=-1\}} + (1-\alpha) \mathbf{1}_{\{12+b \leq 0,y=1\}}] 
	\end{eqnarray*}}
	Therefore, $fl^*_{\alpha} = b^* \in (-8,-12) ~~\text{ with }~~ R_{\alpha,D}(fl^*_{\alpha}) = 0.$ Also,
	\begin{eqnarray*}
		\tilde{fl}^*_{\alpha} &=& \arg\min_{\tilde{fl}}E_{\tilde{D}}[l_{0-1,\alpha}(\tilde{y},\tilde{fl})] \\
		&=& \arg\min_{\tilde{fl}} (1-\rho)E_{D}[l_{0-1,\alpha}(y,\tilde{fl})] + \rho E_{D}[l_{0-1,\alpha}(-y,\tilde{fl})] \\
		&=& \arg\min_{\tilde{b}} \frac{1-\rho}{3}[\alpha \mathbf{1}_{\{3+\tilde{b} > 0,y=-1\}} + \alpha \mathbf{1}_{\{8+\tilde{b} > 0,y=-1\}} + (1-\alpha) \mathbf{1}_{\{12+\tilde{b}\leq 0,y=1\}}] \\
		& & + \frac{\rho}{3}[(1-\alpha) \mathbf{1}_{\{3+\tilde{b}\leq 0,y=1\}} + (1-\alpha) \mathbf{1}_{\{8+\tilde{b}\leq 0,y=1\}} + \alpha \mathbf{1}_{\{12+\tilde{b} > 0,y=-1\}}] \\
		&=& \arg\min_{\tilde{b}} \frac{\alpha}{3} \left[(1-\rho)( \mathbf{1}_{\{3+\tilde{b} >  0,y=-1\}} + \mathbf{1}_{\{8+\tilde{b} > 0,y=-1\}} ) + \rho \mathbf{1}_{\{12+\tilde{b}>  0,y=-1\}} \right]  \\
		&& + \frac{1-\alpha}{3} \left[\rho( \mathbf{1}_{\{3+\tilde{b}\leq 0,y=1\}} + \mathbf{1}_{\{8+\tilde{b}\leq 0,y=1\}} ) + (1-\rho) \mathbf{1}_{\{12+\tilde{b}\leq 0,y=1\}} \right] \\
		&=& \arg\min_{\tilde{b}} \begin{cases}
			0.474 ~~~ \text{  when }~~ \tilde{b} \in (-3,\infty) \\
			0.994 ~~~ \text{  when }~~ \tilde{b} \in (-\infty,-12)
		\end{cases}
	\end{eqnarray*}
	Therefore, $\tilde{fl}^*_{\alpha} = \tilde{b}^* \in (-3,\infty)$ with  $R_{\alpha,D}(\tilde{fl}^*_{\alpha}) = 0.2.$ Hence, $R_{\alpha,D}(fl^*_{\alpha}) \neq  R_{\alpha,D}(\tilde{fl}^*_{\alpha})$ implying that weighted $0$-$1$ loss is not uniform noise robust even with a classifier in $\mathcal{H}_{lin}.$
	\subsection{Details of counter-example \ref{exam: unifDis_noise_squ}} \label{exam: unifDis_noise_squ_details}
	Consider the minimizer of $R_{D,l_{\alpha,usq}}(f)$, the clean $l_{\alpha,usq}$-risk of $f$ given below:
	\begin{eqnarray} \nonumber
	f^*_{l_{\alpha,usq}} &=&
	\arg\min\limits_{f \in \mathcal{H}} R_{D,l_{\alpha,usq}}(f)  \\ \nonumber
	&=& \arg\min\limits_{f \in \mathcal{H}} E_{D}[(1-\alpha)\mathbf{1}_{\{y=1\}}(1-f(\mathbf{x}))^2 + \alpha\mathbf{1}_{\{y=-1\}}\frac{1}{\gamma}(1+\gamma f(\mathbf{x}))^2] \\ \nonumber
	& = & \arg\min\limits_{f \in \mathcal{H}} E_{\mathbf{X}}[E_{Y|\mathbf{X}}[(1-\alpha)\eta(\mathbf{x})(f(\mathbf{x})-y)^2 + \frac{\alpha}{\gamma}(1-\eta(\mathbf{x}))(1+ \gamma f(\mathbf{x}))]]  \\ \label{eq: f-sq-uneven-bayes_clean}
	&=& \frac{\eta(\mathbf{x}) - \alpha}{\eta(\mathbf{x})(1-\alpha) + \gamma\alpha(1-\eta(\mathbf{x}))}
	\end{eqnarray}
	Given $\mathbf{x}$, the optimal decision would be $sign(f^*_{l_{\alpha,usq}}) = sign(\eta(\mathbf{x}) - \alpha)$ since the denominator in the RHS of equation \eqref{eq: f-sq-uneven-bayes_clean} is always positive.
	In  similar lines as above, one can show that the minimizer of $R_{\tilde{D},l_{\alpha,usq}}(f)$ is as follows:
	\begin{equation} \label{eq: f-sq-uneven-bayes_noisy}
	\tilde{f}^*_{l_{\alpha,usq}} = \frac{\tilde{\eta}(\mathbf{x}) - \alpha}{\tilde{\eta}(\mathbf{x})(1-\alpha) + \gamma\alpha(1-\tilde{\eta}(\mathbf{x}))}
	\end{equation} 
	We show using following counter-example that clean weighted risks of $f^*_{l_{\alpha,usq}}$ and $\tilde{f}^*_{l_{\alpha,usq}}$ need not be equal.
	
	Consider the settings as in Example \ref{exam: unifDis_noise}. Let $\alpha = 0.25$ and $\gamma = 0.4$. Then, from equation \eqref{eq: f-sq-uneven-bayes_clean} and \eqref{eq: f-sq-uneven-bayes_noisy}, we get
	$$ f^*_{l_{\alpha,usq}} = -0.21 < 0 ~~~~~~~~~~~ \tilde{f}^*_{l_{\alpha,usq}} = 0.37 >0 ~~ \forall x.$$
	Now, we need to check for the clean $l_{0-1,\alpha}$ risk of above two classifiers. 
	This can be computed in the same way as in Example \ref{exam: unifDis_noise}. Therefore, we get
	$$R_{D,\alpha}(f^*_{l_{\alpha,usq}}) = (1-\alpha)p = 0.15, ~~~~R_{D,\alpha}(\tilde{f}^*_{l_{\alpha,usq}}) = \alpha (1-p) = 0.2$$
	Hence, $R_{D,\alpha}(f^*_{l_{\alpha,usq}}) \neq R_{D,\alpha}(\tilde{f}^*_{l_{\alpha,usq}})$ 
	implying that the classifiers $ f \in \mathcal{H}$ learnt from $l_{\alpha,usq}$ based ERM may not be both cost-sensitive and uniform noise robust.
	\subsection{Another counter-example to show that weighted $0$-$1$ loss, $l_{0-1,\alpha}$ is not cost sensitive uniform noise robust} \label{exam: normDis_noise}
	Following example shows that cost sensitive risk minimization under $l_{0-1,\alpha}$ is not uniform noise tolerant. 
	\begin{example} 
		Let $Y$ has a Bernoulli distribution with parameter $p=0.8$. Let $\mathbf{X} \subset \mathbb{R}$  be such that  $\mathbf{X}|Y=1 \sim N(2,1)$ and $\mathbf{X}|Y=-1 \sim N(-2.5,1)$. Then, using Lemma \ref{lem: eta-normal}, the in-class probability $\eta(\mathbf{x})$ is given as follows:
		\begin{eqnarray} \label{eq: eta_eg2}
		\eta(\mathbf{x})=  \frac{1}{1+ 0.25\exp(-1.125-4.5\mathbf{x})}.
		\end{eqnarray}
		Suppose we are given that the negative class is more important with misclassification cost $\alpha = 0.65$. Also, let the uniform noise rate to corrupt the data be $\rho = 0.3$. Now, using equation \eqref{eq: bayes_0-1_clean_alpha} and \eqref{eq: bayes_0-1_corrup_alpha}, the thresholds corresponding to $\eta(\mathbf{x})$ are $0.65$ and $0.875$. Then, the optimal classifiers are as follows: 
		{ \footnotesize
			\begin{minipage}[!htbp]{0.49\linewidth}
				\begin{eqnarray} \nonumber
				f^*_{\alpha}(\mathbf{x}) &=& \begin{cases}
				+1  ~~~~\text{if} ~~\eta(\mathbf{x}) > 0.65 \\ 
				-1  ~~~~\text{if} ~~\eta(\mathbf{x}) \leq 0.65  
				\end{cases} \\ \label{eq: eg2-f-clean}
				&=& \begin{cases}
				+1  ~~~~\text{if} ~~\mathbf{x} > 0.195629 \\ 
				-1  ~~~~\text{if} ~~\mathbf{x} \leq 0.195629 
				\end{cases}
				\end{eqnarray}
			\end{minipage}
			\begin{minipage}[!htbp]{0.49\linewidth}
				\begin{eqnarray} \nonumber
				\tilde{f}^*_{\alpha}(\mathbf{x}) &=& \begin{cases}
				+1  ~~~~\text{if} ~~\eta(\mathbf{x}) > 0.875 \\ 
				-1  ~~~~\text{if} ~~\eta(\mathbf{x}) \leq 0.875  
				\end{cases} \\ \label{eq: eg2-f-noisy}
				&=& \begin{cases}
				+1  ~~~~\text{if} ~~\mathbf{x} > 0.490489 \\ 
				-1  ~~~~\text{if} ~~\mathbf{x} \leq 0.490489 
				\end{cases}
				\end{eqnarray}
		\end{minipage}}
		
		The last equality is obtained by substituting the value of $\eta(\mathbf{x})$ in \eqref{eq: eta_eg2} and solving for $\mathbf{x}$. Now we compute the weighted $0$-$1$ risks of these classifiers show that they are different. 
		{\footnotesize
			\begin{eqnarray*}
				& & R_{D,\alpha}(f^*_{\alpha}) = E_{D}[l_{0-1,\alpha}(f^*_{\alpha}(\mathbf{x}),y)] \\
				&=& \sum\limits_{y}\int\limits_{\mathbf{x}}(\mathbf{1}_{\{f^*_{\alpha}(\mathbf{x}) > 0, y=-1\}}\alpha  + \mathbf{1}_{\{f^*_{\alpha}(\mathbf{x}) \leq 0, y=1\}}(1-\alpha)) P(\mathbf{X}=\mathbf{x}|Y=y)P(Y=y)d\mathbf{x} \\
				&=& \int\limits_{-\infty}^{\infty} (\mathbf{1}_{\{f^*_{\alpha}(\mathbf{x}) > 0\}}\alpha) P(\mathbf{X}=\mathbf{x}|Y=-1)P(Y=-1)d\mathbf{x}   \\
				& & + \int\limits_{-\infty}^{\infty} (\mathbf{1}_{\{f^*_{\alpha}(\mathbf{x}) \leq 0\}}(1-\alpha)) P(\mathbf{X}=\mathbf{x}|Y=1)P(Y=1)d\mathbf{x}  \\
				&=& (1-p)\alpha\int\limits_{0.1956}^{\infty} \frac{1}{\sqrt{2\pi}}\exp(-\frac{1}{2}(\mathbf{x}+2.5)^2) d\mathbf{x} + p (1-\alpha)\int\limits_{-\infty}^{0.1956} \frac{1}{\sqrt{2\pi}}\exp(-\frac{1}{2}(\mathbf{x}-2)^2) d\mathbf{x} \\
				&=& 0.13(1- \Phi(2.69)) + 0.28\Phi(-1.81)  ~~ \text{where } \Phi(\mathbf{x}) \text{ is the c.d.f of standard normal dist at } \mathbf{x}.  \\
				&=& 0.0103 
			\end{eqnarray*}
			\begin{eqnarray*}
				& & R_{D,\alpha}(\tilde{f}^*_{\alpha}) = E_{D}[l_{0-1,\alpha}(\tilde{f}^*_{\alpha}(\mathbf{x}),y)] \\
				&=& \sum\limits_{y}\int\limits_{\mathbf{x}}(\mathbf{1}_{\{\tilde{f}^*_{\alpha}(\mathbf{x}) > 0, y=-1\}}\alpha  + \mathbf{1}_{\{\tilde{f}^*_{\alpha}(\mathbf{x}) \leq 0, y=1\}}(1-\alpha)) P(\mathbf{X}=\mathbf{x}|Y=y)P(Y=y)d\mathbf{x} \\
				&=& \int\limits_{-\infty}^{\infty} (\mathbf{1}_{\{\tilde{f}^*_{\alpha}(\mathbf{x}) > 0\}}\alpha) P(\mathbf{X}=\mathbf{x}|Y=-1)P(Y=-1)d\mathbf{x}   \\
				& & + \int\limits_{-\infty}^{\infty} (\mathbf{1}_{\{\tilde{f}^*_{\alpha}(\mathbf{x}) \leq 0\}}(1-\alpha)) P(\mathbf{X}=\mathbf{x}|Y=1)P(Y=1)d\mathbf{x}  \\
				&=& (1-p)\alpha\int\limits_{0.4908}^{\infty} \frac{1}{\sqrt{2\pi}}\exp(-\frac{1}{2}(\mathbf{x}+2.5)^2) d\mathbf{x} + p  (1-\alpha)\int\limits_{-\infty}^{0.4908} \frac{1}{\sqrt{2\pi}}\exp(-\frac{1}{2}(\mathbf{x}-2)^2) d\mathbf{x} \\
				&=& 0.13(1- \Phi(2.99)) + 0.28\Phi(-1.51)  ~~ \text{where } \Phi(\mathbf{x}) \text{ is the c.d.f of standard normal dist at } \mathbf{x}.  \\
				&=& 0.0185
		\end{eqnarray*}}	
		Therefore, $R_{D,\alpha}(f^*_{\alpha}) \neq R_{D,\alpha}(\tilde{f}^*_{\alpha})$ implying that the risk minimization under weighted $0$-$1$ loss $l_{0-1,\alpha}$ is not uniform robust.
	\end{example}
	\section{Explicit form for $\eta$ when conditional distributions are normally distributed} \label{sec: eta_true_normal}
	Here, we present a result which provides the explicit from of in-class probability $\eta$ when the conditional distributions $\mathbf{X}|Y$ are normally distributed.
	\begin{lemma} \label{lem: eta-normal}
		Let Y has Bernoulli distribution with parameter $p$. Let $\mathbf{X} \subset \mathbb{R}^n$  be such that $\mathbf{X}|Y = 1 \sim N(\boldsymbol{\mu}_+,\Sigma_+)$ and $\mathbf{X}|Y = -1 \sim N(\boldsymbol{\mu}_-,\Sigma_-)$. Then, the in-class probability $\eta(\mathbf{x}) = P(Y=1|\mathbf{X}=\mathbf{x})$ is given as follows:
		{ \scriptsize
			\begin{equation}\label{eq: eta-normal}
			\eta(\mathbf{x}) = \frac{1}{1 + \frac{1-p}{p}\sqrt[]{\frac{|\Sigma_+|}{|\Sigma_-|}}\exp(-\frac{1}{2}( \mathbf{x}^T(\Sigma^{-}_- - \Sigma^{-}_+)\mathbf{x} - 2\mathbf{x}^T(\Sigma^{-}_-\boldsymbol{\mu}_- - \Sigma^{-}_+\boldsymbol{\mu}_+) + \boldsymbol{\mu}_-^T\Sigma_-^{-}\boldsymbol{\mu}_- - \boldsymbol{\mu}_+^T\Sigma_+^{-}\boldsymbol{\mu}_+))}
			\end{equation}}
		where $\Sigma_+^{-}$ and $\Sigma_-^{-}$ are the inverse of $\Sigma_+$ and $\Sigma_-$ respectively.
	\end{lemma}
	\begin{proof}
		Consider the in-class probability $\eta(\mathbf{x})$ given below:
		{\scriptsize
			\begin{eqnarray*}
				& & P(Y=1|\mathbf{X}=\mathbf{x}) = \frac{P(Y=1,\mathbf{X}=\mathbf{x})}{P(\mathbf{X}=\mathbf{x})} \\
				&=& \frac{P(\mathbf{X}=\mathbf{x}|Y=1)P(Y=1)}{\sum\limits_{y \in \{-1,1\}}P(\mathbf{X}=\mathbf{x}|Y=y)P(Y=y)} \\
				&=& \frac{\frac{1}{\sqrt{2\pi |\Sigma_+|}}\exp(-\frac{1}{2}(\mathbf{x}-\boldsymbol{\mu}_+)^T\Sigma_+^-(\mathbf{x}-\boldsymbol{\mu}_+)) \times p}{\frac{1}{\sqrt{2\pi |\Sigma_+|}}\exp(-\frac{1}{2}(\mathbf{x}-\boldsymbol{\mu}_+)^T\Sigma_+^-(\mathbf{x}-\boldsymbol{\mu}_+)) \times p + \frac{1}{\sqrt{2\pi |\Sigma_-|}}\exp(-\frac{1}{2}(\mathbf{x}-\boldsymbol{\mu}_-)^T\Sigma_-^-(\mathbf{x}-\boldsymbol{\mu}_-)) \times (1-p)} \\
				&=& \frac{1}{1+ \frac{1-p}{p}\sqrt[]{\frac{|\Sigma_+|}{|\Sigma_-|}}\exp(-\frac{1}{2}[(\mathbf{x}-\boldsymbol{\mu}_-)^T\Sigma_-^-(\mathbf{x}-\boldsymbol{\mu}_-) - (\mathbf{x}-\boldsymbol{\mu}_+)^T\Sigma_+^-(\mathbf{x}-\boldsymbol{\mu}_+)] )}
		\end{eqnarray*}}
		The last second equality uses the fact that $P(Y=1) = p$ and $\mathbf{X}|Y$ is normally distributed. Now, expanding the product terms in exponent of the last equality and collecting the common terms, we get equation \eqref{eq: eta-normal}.
	\end{proof}
	This result was used in synthetic data experiments to get the get the Bayes classifiers using $\eta$.
	
	\section{Experiments demonstrating the role of $\alpha$ and $\gamma$ in cost sensitive learning on clean data} \label{sec: role_alpha_gamma}
	In this section, our goal is to understand the role of $\alpha$ and $\gamma$ in $l_{\alpha,un}(f(\mathbf{x},y))$. Specifically, we work with $\alpha$-weighted uneven margin squared loss defined below:
	\begin{equation} \label{eq: un-sq_loss_supp}
	l_{\alpha,usq}(f(\mathbf{x}),y)= (1-\alpha)\mathbf{1}_{\{y=1\}}(1-f(\mathbf{x}))^2 + \alpha\mathbf{1}_{\{y=-1\}}\frac{1}{\gamma}(1+\gamma f(\mathbf{x}))^2, ~~ \gamma >0
	\end{equation}
	We consider $3$ performance measures, viz., Accuracy (Acc), Arithmetic mean (AM) and F measure.
	
	Every dataset is partitioned into training and test set with $80$-$20$ split. Further, for tuning the parameters like $\gamma$, we split the training data into $65$-$15$ and use $15\%$ as validation set for tuning the parameters
	To account for the randomness while partitioning of the data, this process is repeated 10 times. The reported values are performance measures averaged over 10 trials along with their standard deviation. To choose the value of the regularization parameter $\lambda$, we perform a CV over the set $\Lambda = \{.01, .1, 1, 10\}$. Next, we separately provide examples for the cases where there is a need for differential costing of misclassifications.
	
	\subsection{Data has only class imbalance and no user given cost}
	This is the case when there is only imbalance in the dataset and there is no explicit requirement of costing by the domain. In literature, researchers use only $\alpha$ to account for the imbalance, i.e., use surrogate loss function $l_{\alpha}(f(\mathbf{x}),y)$ given below:
	\begin{equation} \label{eq: l_alpha_supp}
	l_{\alpha}(f(\mathbf{x}),y) = (1-\alpha)l_1(f(\mathbf{x})) + \alpha l_{-1}(f(\mathbf{x}))
	\end{equation}
	where $\alpha > 0$ is tuned. However, as stated in \cite{scott2012calibrated}, uneven margins $\gamma \neq 1$ have been found to yield improved empirical performance in classification problems involving imbalanced data. With this motivation, we demonstrate the performance of cost sensitive linear classifier from $l_{\alpha,usq}$ based regularized ERM on two synthetic datasets across different level of imbalances given in the set $Im = \{0.2,0.3,0.4,0.5,0.6,0.7,0.8\}$ with 3 approaches given below:
	\begin{enumerate}
		\item Ap1 : $\alpha$ is tuned and $\gamma =1$ 
		\item Ap2 : $\gamma$ is tuned and $\alpha =0.5$ 
		\item Ap3 : $\alpha$ and $\gamma$ both are tuned 
	\end{enumerate}
	Results on the following two synthetic datasets w.r.t. Accuracy, AM, and F measure are presented in Table \ref{tab: cost_clean_acc_case1}, \ref{tab: cost_clean_AM_case1} and \ref{tab: cost_clean_F_case1} respectively.
	
	\textbf{Syn-dataset1}
	We first generate 1000 binary class labels $Y$ from Bernoulli distribution with parameter $\pi \in Im$ where $\pi = P(Y=1).$ Then a 2-dimensional feature vector $\mathbf{X}$ for each label is drawn from two different Gaussian distributions: $\mathbf{X}|Y =1 \sim N([1,0],\Sigma)$ $\&$ $\mathbf{X}|Y =-1 \sim N([-1,0],\Sigma)$ where $\Sigma = [[1, -0.25];[-0.25,1]]$.
	
	\textbf{Syn-dataset2}
	We first generate 1000 binary class labels $Y$ from Bernoulli distribution with parameter $\pi \in Im$ where $\pi = P(Y=1).$ Then, a 2-dimensional feature vector $\mathbf{X}$ for each label is drawn from two different Gaussian distributions: $\mathbf{X}|Y =1 \sim N([0.9,-0.5],\Sigma)$ $\&$ $\mathbf{X}|Y =-1 \sim N([-0.9,-0.8],\Sigma)$ where $\Sigma = [[1, 0.5];[0.5,1]]$.
	\FloatBarrier
	\begin{table}[H]
		\centering
		{ \scriptsize
			\begin{tabular}{|c|c|c|c|c|c|c|}
				\hline
				& \multicolumn{3}{c|}{\textbf{Syn-dataset1}} & \multicolumn{3}{c|}{\textbf{Syn-dataset2}} \\ \hline
				$\pi$ & \textbf{\begin{tabular}[c]{@{}c@{}}$\alpha$ tuned; \\ $\gamma=1$ \end{tabular}} & \textbf{\begin{tabular}[c]{@{}c@{}} $\gamma$ tuned; \\ $\alpha=0.5$ \end{tabular}} & \textbf{\begin{tabular}[c]{@{}c@{}}$\alpha$ $\&$ $\gamma$\\ tuned\end{tabular}} & \textbf{\begin{tabular}[c]{@{}c@{}}$\alpha$ tuned; \\ $\gamma=1$ \end{tabular}} & \textbf{\begin{tabular}[c]{@{}c@{}} $\gamma$ tuned; \\ $\alpha=0.5$ \end{tabular}} & \textbf{\begin{tabular}[c]{@{}c@{}}$\alpha$ $\&$ $\gamma$\\ tuned\end{tabular}} \\ \hline
				\textbf{0.2} & 0.877  $\pm$  0.017 & 0.873  $\pm$  0.017 & 0.877  $\pm$  0.018 & 0.877  $\pm$  0.018 & 0.872  $\pm$  0.016 & 0.871  $\pm$  0.014 \\ \hline
				\textbf{0.3} & 0.853  $\pm$  0.024 & 0.852  $\pm$  0.025 & 0.853  $\pm$  0.025 & 0.852  $\pm$  0.022 &  0.856  $\pm$  0.025 & 0.854  $\pm$  0.024 \\ \hline
				\textbf{0.4} & 0.850  $\pm$  0.016 & 0.852  $\pm$  0.016 & 0.852  $\pm$  0.016 & 0.841  $\pm$  0.018 & 0.843  $\pm$  0.014& 0.841  $\pm$  0.012 \\ \hline
				\textbf{0.5} & 0.851  $\pm$  0.016 & 0.849  $\pm$  0.0160 & 0.852  $\pm$  0.016 & 0.825  $\pm$  0.018 & 0.834  $\pm$  0.018 & 0.8325  $\pm$  0.0172 \\ \hline
				\textbf{0.6} & 0.859  $\pm$  0.018 & 0.857  $\pm$  0.019 & 0.857  $\pm$  0.019 & 0.852  $\pm$  0.022 & 0.85  $\pm$  0.021 & 0.852  $\pm$  0.022 \\ \hline
				\textbf{0.7} & 0.885  $\pm$  0.020 & 0.884  $\pm$  0.016 & 0.882  $\pm$  0.017 & 0.869  $\pm$  0.015 & 0.872  $\pm$  0.023 & 0.874  $\pm$  0.018 \\ \hline
				\textbf{0.8} & 0.895  $\pm$  0.009 & 0.900  $\pm$  0.010 & 0.9015  $\pm$  0.0105 & 0.892  $\pm$  0.0120 & 0.890  $\pm$  0.014 & 0.895  $\pm$  0.013 \\ \hline
		\end{tabular}}
		\caption{Averaged accuracy over 10 trials of the linear classifiers obtained from $l_{\alpha,usq}$ based regularized ERM on two synthetic datasets. The 3 columns for each dataset correspond to the three approaches mentioned when there is no user given cost but imbalance in the data.}
		\label{tab: cost_clean_acc_case1}
	\end{table}
	\FloatBarrier
	\begin{table}[H]
		\centering
		{ \scriptsize
			\begin{tabular}{|c|c|c|c|c|c|c|}
				\hline
				& \multicolumn{3}{c|}{\textbf{Syn-dataset1}} & \multicolumn{3}{c|}{\textbf{Syn-dataset2}} \\ \hline
				$\pi$ & \textbf{\begin{tabular}[c]{@{}c@{}}$\alpha$ tuned; \\ $\gamma=1$ \end{tabular}} & \textbf{\begin{tabular}[c]{@{}c@{}} $\gamma$ tuned; \\ $\alpha=0.5$ \end{tabular}} & \textbf{\begin{tabular}[c]{@{}c@{}}$\alpha$ $\&$ $\gamma$\\ tuned\end{tabular}} & \textbf{\begin{tabular}[c]{@{}c@{}}$\alpha$ tuned; \\ $\gamma=1$ \end{tabular}} & \textbf{\begin{tabular}[c]{@{}c@{}} $\gamma$ tuned; \\ $\alpha=0.5$ \end{tabular}} & \textbf{\begin{tabular}[c]{@{}c@{}}$\alpha$ $\&$ $\gamma$\\ tuned\end{tabular}} \\ \hline
				\textbf{0.2} & 0.812  $\pm$  0.033 & 0.8195  $\pm$  0.03 & 0.782  $\pm$  0.034 & 0.8136  $\pm$  0.028 & 0.78  $\pm$  0.041 & 0.816  $\pm$  0.024 \\ \hline
				\textbf{0.3} & 0.831  $\pm$  0.0270 & 0.839  $\pm$  0.0281 & 0.828  $\pm$  0.032 & 0.830  $\pm$  0.036 & 0.829  $\pm$  0.031 & 0.840  $\pm$  0.034 \\ \hline
				\textbf{0.4} & 0.842  $\pm$  0.021 & 0.847  $\pm$  0.018 & 0.849  $\pm$  0.016 & 0.821  $\pm$  0.017 & 0.827  $\pm$  0.020 & 0.827  $\pm$  0.019 \\ \hline
				\textbf{0.5} & 0.850  $\pm$  0.0170 & 0.848  $\pm$  0.016 & 0.848  $\pm$  0.014 & 0.831  $\pm$  0.019 & 0.837  $\pm$  0.0186 & 0.837  $\pm$  0.018 \\ \hline
				\textbf{0.6} & 0.841  $\pm$  0.023 & 0.850  $\pm$  0.024 & 0.849  $\pm$  0.019 & 0.846  $\pm$  0.023 & 0.840  $\pm$  0.026 & 0.842  $\pm$  0.023 \\ \hline
				\textbf{0.7} & 0.859  $\pm$  0.025 & 0.862  $\pm$  0.019 & 0.867  $\pm$  0.018 & 0.838  $\pm$  0.018 & 0.843  $\pm$  0.031 & 0.844  $\pm$  0.027 \\ \hline
				\textbf{0.8} & 0.845  $\pm$  0.031 & 0.846  $\pm$  0.032 & 0.850  $\pm$  0.029 & 0.844  $\pm$  0.024 & 0.841  $\pm$  0.030 & 0.843  $\pm$  0.024 \\ \hline
				
		\end{tabular}}
		\caption{Averaged AM over 10 trials of the linear classifiers obtained from $l_{\alpha,usq}$ based regularized ERM on two synthetic datasets. The 3 columns for each dataset correspond to the three approaches mentioned when there is no user given cost but imbalance in the data.}
		\label{tab: cost_clean_AM_case1}
	\end{table}
	\FloatBarrier
	\begin{table}[H]
		\centering
		{ \scriptsize
			\begin{tabular}{|c|c|c|c|c|c|c|}
				\hline
				& \multicolumn{3}{c|}{\textbf{Syn-dataset1}} & \multicolumn{3}{c|}{\textbf{Syn-dataset2}} \\ \hline
				$\pi$ & \textbf{\begin{tabular}[c]{@{}c@{}}$\alpha$ tuned; \\ $\gamma=1$ \end{tabular}} & \textbf{\begin{tabular}[c]{@{}c@{}} $\gamma$ tuned; \\ $\alpha=0.5$ \end{tabular}} & \textbf{\begin{tabular}[c]{@{}c@{}}$\alpha$ $\&$ $\gamma$\\ tuned\end{tabular}} & \textbf{\begin{tabular}[c]{@{}c@{}}$\alpha$ tuned; \\ $\gamma=1$ \end{tabular}} & \textbf{\begin{tabular}[c]{@{}c@{}} $\gamma$ tuned; \\ $\alpha=0.5$ \end{tabular}} & \textbf{\begin{tabular}[c]{@{}c@{}}$\alpha$ $\&$ $\gamma$\\ tuned\end{tabular}} \\ \hline
				\textbf{0.2} & 0.671  $\pm$  0.054 & 0.670  $\pm$  0.054 & 0.677  $\pm$  0.057 & 0.661  $\pm$  0.034 & 0.659  $\pm$  0.057 & 0.670  $\pm$  0.052 \\ \hline
				\textbf{0.3} & 0.744  $\pm$  0.039 & 0.753  $\pm$  0.040 & 0.751  $\pm$  0.043 & 0.748  $\pm$  0.035 & 0.746  $\pm$  0.044 & 0.745  $\pm$  0.041 \\ \hline
				\textbf{0.4} & 0.807  $\pm$  0.023 & 0.812  $\pm$  0.021 & 0.815  $\pm$  0.020 & 0.782  $\pm$  0.020 & 0.787  $\pm$  0.023 & 0.788  $\pm$  0.021 \\ \hline
				\textbf{0.5} & 0.847  $\pm$  0.011 & 0.846  $\pm$  0.013 & 0.849$\pm$  0.010 & 0.833  $\pm$  0.018 & 0.837  $\pm$  0.017 & 0.837  $\pm$  0.017 \\ \hline
				\textbf{0.6} & 0.885  $\pm$  0.015 & 0.886  $\pm$  0.013 & 0.885  $\pm$  0.015 & 0.886 $\pm$  0.015 & 0.882  $\pm$  0.016 & 0.882  $\pm$  0.015 \\ \hline
				\textbf{0.7} & 0.919  $\pm$  0.015 & 0.919  $\pm$  0.013 & 0.919  $\pm$  0.013 & 0.909  $\pm$  0.013 & 0.909  $\pm$  0.010 & 0.909  $\pm$  0.01 \\ \hline
				\textbf{0.8} & 0.936  $\pm$  0.005 & 0.939  $\pm$  0.006 & 0.939  $\pm$  0.006 & 0.935  $\pm$  0.007 & 0.934  $\pm$  0.007 & 0.936  $\pm$  0.008 \\ \hline	
		\end{tabular}}
		\caption{Averaged F measure over 10 trials of the linear classifiers obtained from $l_{\alpha,usq}$ based regularized ERM on two synthetic datasets. The 3 columns for each dataset correspond to the three approaches mentioned when there is no user given cost but imbalance in the data.}
		\label{tab: cost_clean_F_case1}
	\end{table}
	
	\textbf{Observations}
	In Table \ref{tab: cost_clean_acc_case1}, we observe that when there is only imbalance in the data and no user given cost, all three approaches, i.e., Ap1, Ap2 and Ap3 have almost equal accuracy on both the synthetic datasets considered. When AM measure is used for evaluation, one can observe in Table \ref{tab: cost_clean_AM_case1} that Ap2 has marginally better performance for some values of $\pi$. Again with F measure, deciding on one approach to handle only class imbalance is not easy. We believe that one can consider either tuning $\alpha$ with $\gamma=1$ or tuning $\gamma$ with $\alpha=0.5$. Also, as can be seen in Table \ref{tab: cost_clean_acc_case1}, Accuracy values exhibits convexity w.r.t the imbalance $\pi$. This provides us with an evidence that Accuracy is not the correct measure when there is imbalance in the data. It will lead to high values of Accuracy in the presence of imbalance.
	
	\subsection{Data has no class imbalance but there is a user given cost}
	This is the case where there a user given cost $\alpha$ and $\gamma$ can be set to $1$. However, using Bupa dataset \cite{UCI_dataset} we show that, in this case also tuning $\gamma$ can lead to improvement in the performance of the classifier.
	\begin{table}[htbp!]
		\centering
		\begin{tabular}{|c|c|c|}
			\hline
			& \textbf{$\gamma$ tuned $\&$ $\alpha = 0.25$} & \textbf{$\gamma=1$ $\&$ $\alpha = 0.25$} \\ \hline
			\textbf{Acc} & 0.971  $\pm$  0.019 & 0.915  $\pm$  0.035 \\ \hline
			\textbf{AM} & 0.971  $\pm$  0.019 & 0.914  $\pm$  0.035 \\ \hline
			\textbf{F} & 0.966  $\pm$  0.017 & 0.923 $\pm$  0.029 \\ \hline
		\end{tabular}
		\caption{Averaged accuracy, AM and F measures of the linear classifiers obtained from $l_{\alpha,usq}$ based regularized ERM on Bupa dataset. The dataset has 6 features and is balanced. This is the case when there is a user given cost but the data has class imbalance. The above results show that in this case too tuning $\gamma$ improves the performance w.r.t all 3 measures. }
		\label{tab: cost_clean_all_case2_BUPA}
	\end{table}
	\subsection{Data has class imbalance and there is a user given cost too}
	This is the scenario when there is a domain requirement of differential costing and the dataset available is imbalanced too. In this case, we fix the value of $\alpha$ and let the data decide the value of $\gamma$. Here, we present the performance of linear classifiers learnt from $l_{\alpha, usq}$ based regularized ERM on Syn-dataset 2. The results are presented across different level of imbalances from the set $Im$ with the corresponding cost set as $C = \{0.25,0.35,0.45,0.4,0.55,0.65,0.75\}$, i.e., $\alpha \in C$.
	
	\begin{table}[htbp!]
		\centering
		\begin{tabular}{|c|c|c|c|c|}
			\hline
			\textbf{p} & $\alpha$& \textbf{Acc} & \textbf{AM} & \textbf{F} \\ \hline
			\textbf{0.2} & 0.25 &0.870  $\pm$  0.014 & 0.816  $\pm$  0.030 & 0.678  $\pm$  0.040 \\ \hline
			\textbf{0.3} & 0.35& 0.853  $\pm$  0.023 & 0.841  $\pm$  0.033 & 0.746  $\pm$  0.039 \\ \hline
			\textbf{0.4} & 0.45&0.836  $\pm$  0.019 & 0.825  $\pm$  0.018 & 0.786  $\pm$  0.020 \\ \hline
			\textbf{0.5} & 0.4&0.829  $\pm$  0.015 & 0.834  $\pm$  0.017 & 0.833  $\pm$  0.019 \\ \hline
			\textbf{0.6} & 0.55&0.853  $\pm$  0.020 & 0.842 $\pm$  0.024& 0.879  $\pm$  0.014 \\ \hline
			\textbf{0.7} & 0.65&0.870  $\pm$  0.019 & 0.846  $\pm$  0.025 & 0.909  $\pm$  0.013 \\ \hline
			\textbf{0.8} & 0.75&0.882  $\pm$  0.018 & 0.842  $\pm$  0.027 & 0.926  $\pm$  0.012 \\ \hline
		\end{tabular}
		\caption{Averaged accuracy, AM and F measures of the classifiers obtained from $l_{\alpha,usq}$ based regularized ERM on Syn-dataset2. This is the case when there is a user given cost $\alpha$ and there is imbalance also in the dataset as given by $p$.}
		\label{tab: cost_clean_all_case3_syn2}
	\end{table}
	
	\textbf{Observations} In this case, the user given cost is $\alpha$ and the value of $\gamma$ is tuned to account for the class imbalance in data. We observe that the values of $\gamma$ picked up when evaluation measure was Accuracy are almost complementary to the value of $\gamma$ picked up when evaluation measure is AM. Specifically, with Accuracy measure, for $p < 0.5$ the $\gamma$ value is less than $1$ and for $p>0.5$, it is more than 1. However, with AM measure, for $p <0.5$ the $\gamma$ value is greater than $1$ and for $p>0.5$, it is less than 1. Based on the functional form of $l_{\alpha,un}$, the correct thing would to choose the value of $\gamma$ as given by AM measure.  The F measure based evaluation leads to $\gamma$ values lying somewhere between the values obtained from Accuracy and AM measure. For $p < 0.5$, there is a mix of both $\gamma$ less than 1 and more than 1 but for $p>0.5$, all values of $\gamma$ are more than 1.
	
	\section{Additional experimental results when there is a need for differential costing and data has uniform label noise}
	
	\subsection{F and Weighted Cost measure of $l_{\alpha,usq}$-regularized ERM based classifiers and Re-sampling based algorithm on UCI datasets} \label{sec: real_F_WC_supp}
	In this section, we present the performance of \ref{alg: eta_cost_SLN} and $l_{\alpha,usq}$ based regularized ERM on UCI datasets w.r.t. F measure and Weighted Cost measure.
	From Table \ref{tab: F_measure_cost}, we conclude that for Breast Cancer and German dataset, \ref{alg: eta_cost_SLN} performs better among the two proposed schemes and for the other two datasets $l_{\alpha,usq}$ based RegERM is good. With the WC criteria in Table \ref{tab: WC_measure_cost}, \ref{alg: eta_cost_SLN} does well on all datasets. 
	\FloatBarrier
	\begin{table}[!ht]
		\centering
		{ \footnotesize
			\bgroup
			\def\arraystretch{0.87}%
			{\setlength{\tabcolsep}{0.22em}	
				\begin{tabular}{|c|c|c|c|c|c|c|c|c|}
					\hline
					\textbf{Dataset} & \textbf{Cost} & \textbf{$\rho$} & \multicolumn{4}{c|}{\textbf{\ref{alg: eta_cost_SLN}}: $\tilde{\eta}$ estimate from } & \textbf{\begin{tabular}[c]{@{}c@{}}$l_{\alpha,usq}$\\ RegERM\end{tabular}} \\ \hline
					$(n,m_+,m_-)$& $\alpha$ & \textbf{} & \textbf{\begin{tabular}[c]{@{}c@{}} \textit{Lk-fun} \\ with $l_{sq}$\end{tabular}} & \textbf{\begin{tabular}[c]{@{}c@{}}\textit{Lk-fun} \\ with $l_{log}$\end{tabular}} & \textbf{LSPC} & \textbf{KLIEP} & \\ \hline
					\textbf{\begin{tabular}[c]{@{}c@{}}Breast\\ Cancer \\ (9,77,186)\end{tabular}} & 0.2  & \begin{tabular}[c]{@{}c@{}}0.0 \\ 0.2\\ 0.4\end{tabular} & \begin{tabular}[c]{@{}c@{}}0.46$\pm$0.1 \\ \textbf{0.49$\pm$0.02}\\ \textbf{0.44$\pm$0.01}\end{tabular} & \begin{tabular}[c]{@{}c@{}}\textbf{0.48$\pm$0.09}\\  \textbf{0.50$\pm$0.04}\\ \textbf{0.43$\pm$0.01}\end{tabular} & \begin{tabular}[c]{@{}c@{}} 0.47$\pm$0.09 \\ 0.46$\pm$0.07\\ \textbf{0.44$\pm$0.03}\end{tabular} & \begin{tabular}[c]{@{}c@{}}0.46$\pm$0.11 \\ 0.45$\pm$0.06\\ \textbf{0.43$\pm$0.02} \end{tabular} & \begin{tabular}[c]{@{}c@{}}0.47$\pm$0.11\\ 0.35$\pm$0.10 \\ 0.32$\pm$0.14\end{tabular} \\ \hline
					\textbf{\begin{tabular}[c]{@{}c@{}}Pima\\ Diabetes \\ (8,268,500)\end{tabular}} & 0.25  & \begin{tabular}[c]{@{}c@{}}0.0 \\ 0.2\\ 0.4\end{tabular} & \begin{tabular}[c]{@{}c@{}}0.67$\pm$0.02 \\ 0.60$\pm$0.03 \\ 0.51$\pm$0.01\end{tabular} & \begin{tabular}[c]{@{}c@{}}0.66 $\pm$ 0.02 \\ 0.62$\pm$0.04\\ 0.52$\pm$0.61\end{tabular} & \begin{tabular}[c]{@{}c@{}} 0.64$\pm$0.03 \\ 0.58$\pm$0.03\\ 0.51$\pm$0.04\end{tabular} & \begin{tabular}[c]{@{}c@{}}0.62$\pm$0.03 \\ 0.58$\pm$0.02\\ 0.52$\pm$0.37\end{tabular} & \begin{tabular}[c]{@{}c@{}}\textbf{0.68$\pm$0.03} \\ \textbf{0.64$\pm$0.04}\\ \textbf{0.58$\pm$0.07}\end{tabular} \\ \hline
					\textbf{\begin{tabular}[c]{@{}c@{}}German \\ (20,300,700)\end{tabular}} & 0.16  & \begin{tabular}[c]{@{}c@{}}0.0 \\ 0.2\\ 0.4\end{tabular} & \begin{tabular}[c]{@{}c@{}}0.58$\pm$0.03 \\ \textbf{0.52$\pm$0.01}\\ \textbf{0.46$\pm$0.1}\end{tabular} & \begin{tabular}[c]{@{}c@{}}0.57$\pm$0.03 \\ \textbf{0.53$\pm$0.02}\\ \textbf{0.46$\pm$0.01}\end{tabular}  & \begin{tabular}[c]{@{}c@{}} \textbf{0.59$\pm$0.04} \\  \textbf{0.52$\pm$0.03}\\ 0.45$\pm$0.01\end{tabular} & \begin{tabular}[c]{@{}c@{}}0.54$\pm$0.03\\0.48$\pm$0.02\\ \textbf{0.46$\pm$0.01} \end{tabular} & \begin{tabular}[c]{@{}c@{}}0.57$\pm$0.04 \\ 0.49$\pm$0.72\\ 0.45$\pm$0.04\end{tabular} \\ \hline
					\textbf{\begin{tabular}[c]{@{}c@{}}Thyroid \\ (5,65,150)\end{tabular}} & 0.3  & \begin{tabular}[c]{@{}c@{}}0.0 \\ 0.2\\ 0.4\end{tabular} & \begin{tabular}[c]{@{}c@{}}0.71$\pm$0.09 \\ 0.63$\pm$0.12\\ 0.44$\pm$0.04 \end{tabular} & \begin{tabular}[c]{@{}c@{}}0.79$\pm$ 0.08 \\ 0.63$\pm$0.10\\ 0.46$\pm$0.02 \end{tabular}  & \begin{tabular}[c]{@{}c@{}}\textbf{ 0.90$\pm$0.05}\\  \textbf{0.76$\pm$0.06}\\ 0.46$\pm$0.0\end{tabular} & \begin{tabular}[c]{@{}c@{}}0.73$\pm$0.08 \\ 0.56$\pm$0.26\\ 0.43$\pm$0.13\end{tabular} & \begin{tabular}[c]{@{}c@{}}0.72$\pm$0.09 \\ 0.65$\pm$0.09\\ \textbf{0.63$\pm$0.14}\end{tabular} \\ \hline
		\end{tabular}}}
		\egroup
		\caption[F measure from \ref{alg: eta_cost_SLN} and $l_{\alpha,usq}$ on real datasets]{Averaged F measure $(\pm~s.d.)$  of the cost sensitive predictions made by \ref{alg: eta_cost_SLN} and $l_{\alpha,usq}$ based regularized ERM on UCI datasets corrupted by uniform label noise.}
		\label{tab: F_measure_cost}
	\end{table}
	\FloatBarrier
	\begin{table}[!ht]
		\centering
		{ \footnotesize
			\bgroup
			\def\arraystretch{0.87}%
			{\setlength{\tabcolsep}{0.22em}	
				\begin{tabular}{|c|c|c|c|c|c|c|c|c|c|}
					\hline
					\textbf{Dataset} & \textbf{Cost} & \textbf{$\rho$} & \multicolumn{4}{c|}{\textbf{\ref{alg: eta_cost_SLN}}: $\tilde{\eta}$ estimate from } & \textbf{\begin{tabular}[c]{@{}c@{}}$l_{\alpha,usq}$\\ RegERM\end{tabular}} \\ \hline
					$(n,m_+,m_-)$& $\alpha$ & \textbf{} & \textbf{\begin{tabular}[c]{@{}c@{}} \textit{Lk-fun} \\ with $l_{sq}$\end{tabular}} & \textbf{\begin{tabular}[c]{@{}c@{}}\textit{Lk-fun} \\ with $l_{log}$\end{tabular}} & \textbf{LSPC} & \textbf{KLIEP} & \\ \hline
					\textbf{\begin{tabular}[c]{@{}c@{}}Breast\\ Cancer \\ (9,77,186)\end{tabular}} & 0.2 & \begin{tabular}[c]{@{}c@{}}0.0 \\ 0.2\\ 0.4\end{tabular} & \begin{tabular}[c]{@{}c@{}} 8.16$\pm$1.6 \\ \textbf{8.40$\pm$1.16}\\ 9.93$\pm$0.09\end{tabular} & \begin{tabular}[c]{@{}c@{}}8.19$\pm$1.5 \\ 8.95$\pm$0.91\\ 9.98$\pm$0.0\end{tabular} & \begin{tabular}[c]{@{}c@{}}\textbf{7.18$\pm$1.63} \\9.88$\pm$1.41\\ 9.96$\pm$ 0.07\end{tabular} & \begin{tabular}[c]{@{}c@{}}8.02$\pm$1.75 \\ \textbf{7.89$\pm$0.27}\\ \textbf{8.66$\pm$1.13} \end{tabular} & \begin{tabular}[c]{@{}c@{}}9.10$\pm$1.70\\ 12.30$\pm$1.75 \\ 15.09 $\pm$2.62\end{tabular} \\ \hline
					\textbf{\begin{tabular}[c]{@{}c@{}}Pima\\ Diabetes \\ (8,268,500)\end{tabular}} & 0.25 & \begin{tabular}[c]{@{}c@{}}0.0 \\ 0.2\\ 0.4\end{tabular} & \begin{tabular}[c]{@{}c@{}}16.36$\pm$2.16 \\ \textbf{22.34$\pm$2.84}\\ 28.19$\pm$0.51\end{tabular} & \begin{tabular}[c]{@{}c@{}}16.43$\pm$2.14 \\ \textbf{22.01$\pm$2.76} \\ 28.07$\pm$0.61\end{tabular} & \begin{tabular}[c]{@{}c@{}}  \textbf{15.92$\pm$2.67} \\ 28.21$\pm$0.11\\ 28.26$\pm$3.55\end{tabular} & \begin{tabular}[c]{@{}c@{}}17.54$\pm$2.26 \\ \textbf{21.83$\pm$1.62}\\ \textbf{24.09$\pm$0.89} \end{tabular} & \begin{tabular}[c]{@{}c@{}}19.14$\pm$2.20 \\  26.22$\pm$3.36 \\ 33.02$\pm$5.09\end{tabular} \\ \hline
					\textbf{\begin{tabular}[c]{@{}c@{}}German \\ (20,300,700)\end{tabular}} & 0.16 & \begin{tabular}[c]{@{}c@{}}0.0 \\ 0.2\\ 0.4\end{tabular} & \begin{tabular}[c]{@{}c@{}}22.57$\pm$2.55 \\ 27.57$\pm$1.11\\ 30.03$\pm$0.45\end{tabular} & \begin{tabular}[c]{@{}c@{}}22.95$\pm$2.16 \\ 26.98$\pm$1.07\\ 30.20$\pm$0.14\end{tabular} & \begin{tabular}[c]{@{}c@{}}  \textbf{19.84$\pm$2.65}  \\ 29.51$\pm$0.97\\ 30.16$\pm$0.08 \end{tabular} & \begin{tabular}[c]{@{}c@{}}24.30$\pm$0.54\\ \textbf{24.56$\pm$0.50} \\ \textbf{24.65$\pm$0.37} \end{tabular} & \begin{tabular}[c]{@{}c@{}}27.10$\pm$2.82 \\ 43.60$\pm$5.56\\ 59.54$\pm$5.10\end{tabular} \\ \hline
					\textbf{\begin{tabular}[c]{@{}c@{}}Thyroid \\ (5,65,150)\end{tabular}} & 0.3  & \begin{tabular}[c]{@{}c@{}}0.0 \\ 0.2\\ 0.4\end{tabular} & \begin{tabular}[c]{@{}c@{}}3.81$\pm$1.13 \\ 5.63$\pm$2.36\\ 9.57$\pm$0.78 \end{tabular} & \begin{tabular}[c]{@{}c@{}}2.76$\pm$1.19 \\  6.29$\pm$2.79\\ 9.34$\pm$0.58 \end{tabular} & \begin{tabular}[c]{@{}c@{}}\textbf{1.46$\pm$0.72} \\ \textbf{3.07$\pm$0.69}\\ 9.18$\pm$0.0\end{tabular} & \begin{tabular}[c]{@{}c@{}}3.47$\pm$1.09 \\ 7.54$\pm$2.17\\  9.90 $\pm$\ 2.22 \end{tabular} & \begin{tabular}[c]{@{}c@{}}3.89$\pm$1.02 \\ 4.68$\pm$0.99\\ \textbf{5.38$\pm$1.87}\end{tabular} \\ \hline
		\end{tabular}}}
		\egroup
		\caption[ Weighted cost from \ref{alg: eta_cost_SLN} and $l_{\alpha,usq}$ on real datasets]{Averaged Weighted Cost $(\pm~s.d.)$  of the cost sensitive predictions made by \ref{alg: eta_cost_SLN} and $l_{\alpha,usq}$ based regularized ERM on UCI datasets corrupted by uniform label noise.}
		\label{tab: WC_measure_cost}
	\end{table}		
	
	\subsection{Result of \ref{alg: eta_cost_SLN} on Bupa dataset when $\tilde{\eta}$ is estimated using \textit{Lk-fun} with squared loss} \label{sec: Bupa_squared_loss_experiments}
	
	In this section, we obtained the classifiers from \ref{alg: eta_cost_SLN} when $\tilde{\eta}$ is estimated using \textit{Lk-fun} with $l_{sq}$ on Bupa datasets. We compared the cases when the classifiers are learnt by tuning $\gamma$ and when $\gamma =1$. We observed that tuning $\gamma$ is always better than fixing $\gamma$ at 1 because all the classifiers have improved value of all performance measures with the former.
	\begin{table}[H]
		\centering
		\bgroup
		\def\arraystretch{0.87}%
		{\setlength{\tabcolsep}{0.22em}		
			{ \scriptsize
				\begin{tabular}{|c|c|c|c|c|c|c|c|c|}
					\hline
					& \multicolumn{8}{c|}{ \ref{alg: eta_cost_SLN}: $\tilde{\eta}$ estimate from \textit{Lk-fun} with $l_{sq}$ } \\
					\hline
					& \multicolumn{4}{c|}{\textbf{$\gamma$ tuned}} & \multicolumn{4}{c|}{\textbf{$\gamma = 1$}} \\ \hline
					\textbf{$\rho$} & \textbf{Acc} & \textbf{AM} & \textbf{F} & \textbf{WC} & \textbf{Acc} & \textbf{AM} & \textbf{F} & \textbf{WC} \\ \hline
					0.0 & 0.96$\pm$0.02 & 0.97$\pm$0.02 & 0.96$\pm$0.02 & 1.49$\pm$1.02 & 0.90$\pm$0.03 & 0.90$\pm$0.03 & 0.91$\pm$0.02 & 1.67$\pm$0.84 \\ \hline
					0.1 & 0.96$\pm$0.02 & 0.96$\pm$0.03 & 0.97$\pm$0.02 & 3.70$\pm$2.05 & 0.71$\pm$0.88 & 0.70$\pm$0.09 & 0.77$\pm$0.05 & 6.12$\pm$1.61 \\ \hline
					0.2 & 0.9$\pm$0.03 & 0.89$\pm$0.03 & 0.90$\pm$0.03 & 5.65$\pm$0.05 & 0.56$\pm$0.05 & 0.55$\pm$0.06 & 0.69$\pm$0.02 & 7.95$\pm$0.79 \\ \hline
					0.3 & 0.85$\pm$0.04 & 0.85$\pm$0.04 & 0.81$\pm$0.1 & 5.4$\pm$1.74 & 0.51$\pm$0.01 & 0.50$\pm$0.01 & 0.67$\pm$0.01 & 8.5$\pm$0.23 \\ \hline
					0.4 & 0.59$\pm$0.09 & 0.58$\pm$0.09 & 0.67$\pm$0.00 & 5.66$\pm$0.00 & 0.51$\pm$0.0 & 0.50 $\pm$0.01 & 0.67$\pm$0.01 & 8.32$\pm$0.37 \\ \hline
		\end{tabular}}}
		\egroup
		\caption{Dataset: BUPA Liver Disorder. The above table depicts the performance of \ref{alg: eta_cost_SLN} when $\tilde{\eta}$ estimate is from \textit{Lk-fun} with $l_{sq}$ based regularized ERM ($\lambda$ tuned via CV). Here, $\alpha =0.25$ and $\gamma$ is tuned.}
		\label{tab: Re-sq-ERM-Bupa}
	\end{table}

	\subsection{Performance of $l_{\alpha,usq}$-regularized ERM based classifier and Re-sampling based algorithm on Synthetic dataset} \label{sec: syn_experiments}
	
	$\mathbf{Syn3D_{imb}}$ (adapted from \cite{efron1997improvements}): Generate $4000$ binary class labels $Y \sim Bern(0.35)$. Then, a $3$-dimensional feature vector $\mathbf{X}$ for each label is drawn from two different Gaussian distributions: $\mathbf{X}|Y=1 \sim N([1.5;0;-1],\Sigma)$ $\&$ $\mathbf{X}|Y=-1 \sim N([-1.5; 1; 1], \Sigma)$ where $\Sigma = [[2,-0.3,-0.3];[-0.3,2,-0.3];[-0.3,-0.3,2]]$. \\
	For $\alpha = 0.3$ and $\gamma$ tuned suitably for the measure in consideration, we compare the quality of cost sensitive predictions from \ref{alg: eta_cost_SLN} and $l_{\alpha,usq}$ based regularized ERM implemented on corrupted data with cost sensitive Bayes classifier $f^*_{0-1,\alpha,\gamma}(\mathbf{x}) = sign(\eta(\mathbf{x})-\frac{\alpha}{\gamma + \alpha(1-\gamma)})$ learnt on clean data.
	
	It can be observed in Table \ref{tab: syn_cost_AC_AM} that Accuracy and AM values for \ref{alg: eta_cost_SLN} and $l_{\alpha,usq}$ based classifiers are comparable to that of $f^*_{0-1,\alpha,\gamma}$ till the noise rate $\rho$ does not cross $0.4$ after which it starts deteriorating. However, as shown in Table \ref{tab: syn_cost_F_WC}, for F measure and Weighted Cost, this threshold is reached at close to $\rho = 0.3$ implying that these measures are more sensitive to noise.  \FloatBarrier
	\begin{table}[!htbp]
		\centering
		{\footnotesize
			\begin{tabular}{|c|c|c|c|c|}
				\hline
				& \multicolumn{4}{c|}{\textbf{$f^*_{0-1,\alpha,\gamma}$ performance}} \\ \hline
				& \textbf{AC} & \textbf{AM} & \textbf{F} & \textbf{WC} \\ \hline
				\textbf{$\rho = 0$} & 0.90$\pm$0.01 & 0.88$\pm$0.01 & 0.85$\pm$0.01 & 42.91$\pm$4.01 \\ \hline
		\end{tabular}}
		\caption{Performance of cost weighted Bayes classifier on $Syn3D_{imb}$ w.r.t. various performance measures. The misclassification cost $\alpha$ is taken as 0.3 and $\gamma$ is tuned.}
		\label{tab: syn_cost_bayes}
	\end{table}
	\vspace{-1.5cm}
	\FloatBarrier
	\begin{table}[!htbp]
		\centering
		{ \scriptsize
			\begin{tabular}{|c|c|c|c|c|c|c|}
				\hline
				\textbf{} & \multicolumn{3}{c|}{\textbf{Accuracy}} & \multicolumn{3}{c|}{\textbf{AM}} \\ \hline
				\textbf{} & \multicolumn{2}{c|}{\ref{alg: eta_cost_SLN}: $\tilde{\eta}$ from} & & \multicolumn{2}{c|}{\ref{alg: eta_cost_SLN}: $\tilde{\eta}$ from} &  \\ \hline
				\textbf{$\rho$} & \textbf{\begin{tabular}[c]{@{}c@{}} Link-func \\ with $l_{log}$ \end{tabular} } & \textbf{ \begin{tabular}[c]{@{}c@{}} LSPC \end{tabular} } & \textbf{\begin{tabular}[c]{@{}c@{}}$l_{\alpha,usq}$\\ RegERM\end{tabular}} & \textbf{\begin{tabular}[c]{@{}c@{}}  Link-func \\ with $l_{log}$\end{tabular}} & \textbf{\begin{tabular}[c]{@{}c@{}}  LSPC \end{tabular}} & \textbf{\begin{tabular}[c]{@{}c@{}}$l_{\alpha,usq}$\\ RegERM\end{tabular}} \\ \hline
				\textbf{0.1}  & 0.90$\pm$0.01 &0.89$\pm$0.01 & 0.90$\pm$0.01 & 0.89$\pm$0.01 &0.89$\pm$0.01 &0.88$\pm$0.01 \\ \hline
				\textbf{0.2}  & 0.89$\pm$0.01 &0.89$\pm$0.00 & 0.90$\pm$0.01 & 0.88$\pm$0.01 &0.88$\pm$0.01 &0.89$\pm$0.01 \\ \hline
				\textbf{0.3} & 0.89$\pm$0.01 &0.89$\pm$0.01 & 0.89$\pm$0.01 & 0.88$\pm$0.01 &0.88$\pm$0.01 &0.88$\pm$0.01 \\ \hline
				\textbf{0.4} &  0.88$\pm$0.01 &0.71$\pm$0.04 &0.89$\pm$0.01 & 0.87$\pm$0.01 & 0.69$\pm$0.03&0.88$\pm$0.01 \\ \hline
				\textbf{0.45} & 0.83$\pm$0.03 & 0.60$\pm$0.06&0.80$\pm$0.06 & 0.81$\pm$0.07 & 0.57$\pm$0.05& 0.81$\pm$0.07 \\ \hline
				\textbf{0.49} &0.46$\pm$0.16 &0.38$\pm$0.02 & 0.48$\pm$0.16 & 0.53$\pm$0.14 & 0.51$\pm$0.01& 0.54$\pm$0.15 \\ \hline
		\end{tabular}}
		\caption{Comparison of \ref{alg: eta_cost_SLN} and $l_{\alpha,usq}$ based regularized ERM's linear classifier on Accuracy and AM measure. The value of $\alpha = 0.3$ and $\gamma$ is tuned.}
		\label{tab: syn_cost_AC_AM}
	\end{table}
	\vspace{-1cm}
	\FloatBarrier
	\begin{table}[!htbp]
		\centering
		{ \scriptsize
			\begin{tabular}{|c|c|c|c|c|c|c|}
				\hline
				\textbf{} & \multicolumn{3}{c|}{\textbf{F measure}} & \multicolumn{3}{c|}{\textbf{Weighted Cost}} \\ \hline
				\textbf{} & \multicolumn{2}{c|}{\ref{alg: eta_cost_SLN}: $\tilde{\eta}$ from} & & \multicolumn{2}{c|}{\ref{alg: eta_cost_SLN}: $\tilde{\eta}$ from} &  \\ \hline
				\textbf{$\rho$} & \textbf{\begin{tabular}[c]{@{}c@{}} Link-func \\ with $l_{log}$ \end{tabular} } & \textbf{ \begin{tabular}[c]{@{}c@{}} LSPC \end{tabular} } & \textbf{\begin{tabular}[c]{@{}c@{}}$l_{\alpha,usq}$\\ RegERM\end{tabular}} & \textbf{\begin{tabular}[c]{@{}c@{}}  Link-func \\ with $l_{log}$\end{tabular}} & \textbf{\begin{tabular}[c]{@{}c@{}}  LSPC \end{tabular}} & \textbf{\begin{tabular}[c]{@{}c@{}}$l_{\alpha,usq}$\\ RegERM\end{tabular}} \\ \hline
				\textbf{0.1} &  0.85$\pm$0.02 &0.85$\pm$0.01 & 0.85$\pm$0.02 & 45.62$\pm$3.37 &40.16$\pm$3.61 &44.06$\pm$3.58 \\ \hline
				\textbf{0.2} & 0.84$\pm$0.01 &0.83$\pm$0.01 &0.84$\pm$0.01 & 60.40$\pm$4.14 &58.01$\pm$7.14 &58.09$\pm$4.75 \\ \hline
				\textbf{0.3} & 0.69$\pm$0.03 &0.61$\pm$0.01 &0.68$\pm$0.01 & 106.9$\pm$7.35 &109.77$\pm$11..24 &107.47$\pm$6.84 \\ \hline
				\textbf{0.4} & 0.53$\pm$0.01 &0.52$\pm$0.01 &0.53$\pm$0.01 & 187.41$\pm$6.45 &158.76$\pm$1.76 &190.10$\pm$6.20 \\ \hline
				\textbf{0.45} & 0.51$\pm$0.00 &0.51$\pm$0.00 &0.52$\pm$0.00 & 200.37$\pm$0.24 &159.55$\pm$0.0 &200.49$\pm$0.00 \\ \hline
				\textbf{0.49} & 0.51$\pm$0.00 &0.52$\pm$0.02 &0.52$\pm$0.00 & 200.49$\pm$2.84 & 159.55$\pm$0.0 & 200.49$\pm$0.00 \\ \hline
		\end{tabular}}
		\caption{Comparison of \ref{alg: eta_cost_SLN} and $l_{\alpha,usq}$ based regularized ERM's linear classifier on F and Weighted Cost measure. The value of $\alpha = 0.3$ and $\gamma$ is tuned.}
		\label{tab: syn_cost_F_WC}
	\end{table}
	
	\section{Various in-class probability estimation methods} \label{sec: eta_comaparative}
	In this section, we present a comparative study between 3 existing in-class probability $\eta$ estimation methods and a minimum squared deviation method. The later method is based on our idea that $\eta$ is a conditional expectation and hence a minimizer of a suitable expected squared loss. We first present the details of 3 existing $\eta$ estimation methods viz., Link function based method, LSPC, KLIEP and $k$-Nearest Neighbour ($k$-NN). Then, we present our method and various estimators based on KLIEP.
	
	\subsection{Link function based approach \cite{reid2010composite}} \label{sec: link_classifier_eta}
	This method is the most commonly used approach for getting the in-class probabilities from a given dataset. In this method, first a classifier is learnt from the data using a special class of loss functions and then a suitable transformation is used to get the $\eta$ estimates. This special class of loss functions is the class of strongly proper loss functions defined below:
	
	\begin{defn}[Strongly proper loss \cite{agarwal2013surrogate}]
		Let $c:\{\pm1\} \times [0,1] \mapsto \mathbb{R}_{+}$ be a binary Class probability estimation (CPE) loss and let $\lambda >0 $.Then, $c$ is $\lambda$-strongly proper if $\forall \eta,\hat{\eta} \in [0,1]$ 
		$$ C_c(\eta,\hat{\eta}) -  H_{c}(\eta)\geq \frac{\lambda}{2}(\eta-\hat{\eta})^2$$
		where $H_c(\eta) = \inf\limits_{{\eta^{\prime}} \in [0,1]} C_{c}(\eta,{\eta^{\prime}})$.
	\end{defn}
	
	Therefore, a loss $l: \{\pm 1\}\times \mathbb{R} \mapsto \mathbb{R}_+$ is said to be strongly proper composite if it can be written as 
	$$ l(y,\hat{y}) = c(y,\psi^{-1}(\hat{y}))$$
	for some strongly proper loss $c$ and and strictly increasing link function $\psi : [0,1]\mapsto \mathbb{R}$.
	
	Given a  strongly proper loss function $l$, one can obtain an optimal classifier $f^*(\mathbf{x})$ by minimizing $l$ based empirical risk minimization (ERM). Then, the in-class probability can be computed as $\eta_{l}(\mathbf{x}) = \psi^{-1}_l(f(x))$. Some examples of strongly proper composite loss functions along with their corresponding inverse link function $\psi^{-1}(\cdot)$ \cite{agarwal2013surrogate,zhang2004statistical} are as follows:
	\begin{enumerate}
		\item Logistic loss: $l_{log}(f(\mathbf{x}),y) = \log(1+e^{-yf(\mathbf{x})})/\log(2)$, $\psi^{-1}_{l_{log}}(f(\mathbf{x})) = \frac{1}{1+ e^{-f(\mathbf{x})}}$ 
		\item Squared loss: $l_{sq}(f(\mathbf{x}),y) = (1-yf(\mathbf{x}))^2$, $\psi^{-1}_{l_{sq}}(f(\mathbf{x})) = \frac{1+f(\mathbf{x})}{2}$ 
		\item Mod-squared loss: $l_{msq}(f(\mathbf{x}),y) = (\max\{0,(1-yf(\mathbf{x}))\})$, $\psi^{-1}_{l_{msq}}(f(\mathbf{x})) = \frac{1+T(f(\mathbf{x}))}{2}$ where $T(p)= \min\{\max\{p,-1\},1\}$ 
		\item Exponential loss: $l_{exp}(f(\mathbf{x}),y) = e^{-yf(\mathbf{x})}$, $\psi^{-1}_{l_{log}}(f(\mathbf{x})) = \frac{1}{1+ e^{-2f(\mathbf{x})}}$ 
	\end{enumerate}
	
	It is shown in \cite{agarwal2013surrogate} that if $l$ is a strongly proper loss then the $L_2$ distance between $\psi^{-1}_l(f(\mathbf{X}))$ and $\eta_l(\mathbf{X})$ can be upper bounded in terms of $l$-regret of $f$. Some special cases of this result are also presented in \cite{zhang2004statistical}, \cite{steinwart2007compare} (for logistic loss). In this method, for some loss functions like squared loss and mod-squared loss, one has to suitably truncate the inverse link function $\psi^{-1}_l$ so that it is a valid probability. 
	
	The quality of ${\eta}$ estimate will depend upon the quality of the classifier too. So, whether one uses linear or nonlinear (kernel)  classifiers is expected to make a difference. Also, in case of non-linear classifier, regularization parameter tuning is required.
	
	As hinge loss is not a strongly proper loss, one cannot use the above approach to get the $\eta(\mathbf{x})$. It was shown in \cite{platt2000} that one can get the in-class probability by using a sigmoid transformation and solving an optimization problem in 2 variables. However, in \cite{zhang2004statistical}, it is shown that SVM formulation doesn't lead to reliable probability estimates. 
	
	For the sake of convenience, a compact form of this scheme is available in \ref{alg: eta_via_classifier}. 
	\FloatBarrier
	\begin{algorithm}[H]
		\renewcommand{\thealgorithm}{Algorithm $\boldsymbol{(f,\hat{\eta})}$}
		\floatname{algorithm}{}
		\caption{$\hat{\eta}$ via classifier scheme} 
		\label{alg: eta_via_classifier}
		\begin{algorithmic}[1]
			\Statex \hspace{-0.44cm}\textbf{Input:} Training data $\tilde{S}_{tr} = \{(x_1,{y}_i)\}_{i=1}^{m_{tr}}$, test data $S_{te} = \{(x_i,y_i)\}_{i=1}^{m_{te}}$, strongly proper composite loss $l(f(x),{y})$, its inverse link $\psi^{-1}_{l}(.)$.		\Statex \hspace{-0.44cm}\textbf{Output:} $\hat{\eta}$ estimate of $\eta$ on test data $S_{te}$
			\State Compute $\hat{f}^*_{{D},l} = \arg\min\limits_{f\in \mathcal{H}_{lin}}\sum\limits_{i=1}^{m\_tr}l(f(x_i),{y}_i)$.
			\For{$i=1,2,\ldots,m_{te}$}
			\State Compute the estimate $\hat{{\eta}}(x_i)= \psi^{-1}_{l}(\hat{f}^*_{{D},l}(x_i))$.
			\EndFor
			
		\end{algorithmic}
	\end{algorithm}
	
	\subsection{Least Squared Probabilistic Classifier (LSPC) \cite{sugiyama2010superfast}} \label{sec: eta_LSPC}
	Least Squared Probabilistic Classifier (LSPC) \cite{sugiyama2010superfast} is a classifier which makes predictions based on the in-class probability $\eta$. This is a direct $\eta$ estimation method. Suppose the in-class probability is modeled as follows :
	$$ q(y|\mathbf{x},\boldsymbol{\alpha}) := \sum\limits_{l=1}^b\alpha_l\phi_l(\mathbf{x},y) = \boldsymbol{\alpha}^T\phi(\mathbf{x},y)$$
	where $\boldsymbol{\alpha}$, a $b$-dimensional vector is to be learned from the sample and $\phi(\mathbf{x},y)$ is non-negative $b$-dimensional vector of basis functions.
	The parameter $\boldsymbol{\alpha}$ is obtained so that the following squared error $J$ is minimized:
	\begin{eqnarray}
	J = \frac{1}{2}\sum\limits_{y=1}^{c}\int(q(y|\mathbf{x},\boldsymbol{\alpha}) - p(y|\mathbf{x}))^2p(\mathbf{x})d\mathbf{x}
	\end{eqnarray}
	Here, $c$ is the number of classes. Solving a regularized empirical version of the above optimization problem leads to the following closed form expression for the optimal value of $\boldsymbol{\alpha}$,
	$$ \hat{\boldsymbol{\alpha}} = (\hat{H} + \lambda I_b)^{-1}\hat{h}$$ 
	where  $\hat{H}=\frac{1}{m}\sum\limits_{y=1}^{c}\sum\limits_{i=1}^{m}\phi(\mathbf{x}_i,y_i)\phi(\mathbf{x}_i,y_i)^T	 $ and  $ \hat{h}= \frac{1}{m}\sum\limits_{i=1}^{m}\phi(\mathbf{x}_i,y_i)$ and $\lambda$ is th regularization parameter. The final solution is obtained by normalizing the solution to get,
	\begin{eqnarray} \label{eq: LSPC_final_sol}
	\hat{p}(y|\mathbf{x})= \frac{\max{(0,\hat{\boldsymbol{\alpha}}^T \phi(\mathbf{x},y))}}{\sum\limits_{y'=1}^{c}\max{(0,\hat{\boldsymbol{\alpha}}^T \phi(\mathbf{x},y'))}}
	\end{eqnarray}
	
	In particular, LSPC uses the Gaussian kernels centered at training points. Use of kernel functions, allows to learn the parameters in a class-wise manner.
	The kernel parameters are chosen based on the dataset and the regularization parameter is chosen by CV. It is claimed that LSPC can be viewed as an application of density ratio estimation method called unconstrained Least Squared Importance Fitting method (uLSIF \cite{kanamori2009least}) and hence the theoretical guarantees like consistency and stability will follow.  
	
	However, the consistency results of \cite{kanamori2009least} are under the following assumption. If $n_{nu}$ and $n_{de}$ are the data points from the numerator and denominator  density respectively, then $n_{nu} = \omega(n_{de}^2)$. This means that there exists $c$ such that $\vert cn_{de}^2 \vert < \vert n_{nu}\vert$. Now, the density ratio in case of in-class probability is $\eta(\mathbf{x}) = \frac{P(\mathbf{x},y)}{P(\mathbf{x})}$. Therefore, the number of sample points in numerator and denominator would be equal i.e., $n_{nu} = n_{de}$ and the assumption of uLSIF is not satisfied. 
	
	Centiles from the data are known to be good candidates for Gaussian kernel width. Pseudo code given in \ref{alg: centile_method} is borrowed from the code made available by \cite{sugiyama2010superfast}. It generates the set of kernel width $\sigma$ to be chosen from for tuning Gaussian kernel's parameter $\sigma$. 
	
	\begin{algorithm}[H]
		\renewcommand{\thealgorithm}{Pseudocode Centile method}
		\floatname{algorithm}{}
		\caption{Method for generating a set of $\sigma$ values for Gaussian kernels.} 
		\label{alg: centile_method}
		\begin{algorithmic}[1]
			\Statex \hspace{-0.44cm}\textbf{Input:} Sample points $\mathbf{x}_i, i=1,\ldots, m$, Centile Vector $C$ of size $k$ e.g. C = $\{10,20,\ldots,90\}$ implies k=9.		\Statex \hspace{-0.44cm}\textbf{Output:} $\Sigma_{set}$, a vector of size $k$ containing possible $\sigma$ values.
			\State Set number of pairs for which the distance is computed as $np = \min(2000,m)$ and $n_{cen}=k.$
			\State randOrder1 = $\pi(1:m)$ and randOrder2 = $\pi(1:m)$ where $\pi(\cdot)$ is a permutation of $(1, \cdots, m)$.
			\For{$i=1,2,\ldots,np$}
			\State dist[i] = $\Vert \mathbf{x}_{randOrder1[i]} - \mathbf{x}_{randOrder2[i]} \Vert_2^2$
			\EndFor
			\State Sort the vector dist.
			\For{$i=1,\ldots,n_{cen}$}
			\State out[i] = $dist\left[ np \times \frac{C[i]}{100}\right]$
			\EndFor
			\State Return $\Sigma_{set} = \sqrt[]{out}.$			
		\end{algorithmic}
	\end{algorithm}
	
	\subsection{Kullback-Leibler Importance Estimation Procedure (KLIEP) \cite{sugiyama2008direct}}
	The in-class probability for a class $y$ can be written as
	\begin{equation} \label{KLIEP_depen}
	P(Y=y|\mathbf{x}) = \frac{P(\mathbf{x},y)}{P(\mathbf{x})} = \frac{P(\mathbf{x}|Y=y)}{P(\mathbf{x})}P(Y=y)
	\end{equation}
	The authors in \cite{sugiyama2008direct} deal with the problem of estimating importance $w(\mathbf{x}):= \frac{p_{nu}(\mathbf{x})}{p_{de}(\mathbf{x})}$ using $\{\mathbf{x}_i^{de}\}_{i=1}^{n_{de}}$ and $\{\mathbf{x}_j^{nu}\}_{j=1}^{n_{nu}}$ from the probability density $p_{de}$ and $p_{nu}$ respectively. The estimated ratio is given a linear parametric form of $\hat{w}(\mathbf{x}) = \sum\limits_{l=1}^{b}\alpha_{l}\phi_l(\mathbf{x})$. Here, $\phi_{l}$ are known non negative basis functions and $\alpha_l$ are the parameters to be estimated. The numerator density can be estimated as $\hat{p}_{nu}(\mathbf{x}) = \hat{w}(\mathbf{x}) p_{de}(\mathbf{x})$. The parameters are determined by  minimizing the KL divergence between the true and estimated numerator density as follows:
	\begin{eqnarray}  \nonumber 
	KL[p_{nu}(\mathbf{x})\Vert\hat{p}_{nu}(\mathbf{x})] &=& \int p_{nu}(\mathbf{x})\log\frac{p_{nu}(\mathbf{x})}{\hat{w}(\mathbf{x})p_{de}(\mathbf{x})}d\mathbf{x}  \label{eq: KL_KLIEP}  \\
	&=& \int p_{nu}(\mathbf{x})\log\frac{p_{nu}(\mathbf{x})}{p_{de}(\mathbf{x})}d\mathbf{x} - \int p_{nu}(\mathbf{x})\log \hat{w}(\mathbf{x})d\mathbf{x} 
	\end{eqnarray}
	The first term in equation \eqref{eq: KL_KLIEP} is independent of the $\{\alpha_l\}_{l=1}^{b}$, and incorporating the fact that $\hat{p}_{nu}(\mathbf{x})$ is probability density, the following optimization problem is obtained:
	
	\begin{eqnarray} \label{eq: optProb_KLIEP}
	& & \max\limits_{\{\alpha_l\}_{l=1}^{b}}\left[ \sum\limits_{j=1}^{n_{nu}} \log \left( \sum\limits_{l=1}^{b}\alpha_l\phi_l(\mathbf{x}^{nu}_j) \right) \right] \\
	& &\text{ subject to } \sum\limits_{i=1}^{n_{de}}\sum\limits_{l=1}^{b}\alpha_l\phi_l(\mathbf{x}^{de}_i) = n_{de} ~~ \text{ and } ~~ \alpha_1,\ldots,\alpha_b \geq 0
	\end{eqnarray}
	This is a convex optimization problem and the authors have provided an algorithm based on gradient ascent to solve this problem. The solution is claimed to be sparse and hence leads to a reduction in computational complexity.
	
	\subsection{$k$-Nearest Neighbour method for $\eta$ estimation \cite{chen2018explaining}}
	In classical binary setup, $k$-NN works based on the labels of the neighbours. Firstly, an estimate of in-class probability is computed and then it is thresholded to make the label prediction. $$\hat{\eta}_{k-NN}(\mathbf{x}) = \frac{1}{k}\sum\limits_{i=1}^{k}Y_{(i)}(\mathbf{x})$$
	Here, $Y_{(i)}$'s are ordered w.r.t. some distance measure $\rho.$ Choice of $k$ and $\rho$ are the two levers available while learning. Also, there is a trade-off between bias and variance.
	
	\subsection{Minimization of Squared Deviation method} \label{sec: eta_min_sq}
	This is a direct $\eta$ estimation method, i.e., it doesn't involve learning a classifier. It is known that the expected squared loss is minimized by a suitable conditional expectation. In particular, given two r.v.s $P,Q$ such that $Q$ is integrable, $Z^*=E[Q|P]$ is the minimizer of $E[(Q-Z)^2]$ a.s. among all random variables $Z$, i.e.,
	$$ \forall h(P), ~~~~ E[(Q-Z^*)^2] \leq  E[(Q-h(P))^2]$$
	where $h$ belongs to the class of all measurable functions.
	Given $(\mathbf{X},Y), ~~\mathbf{X} \in \mathbb{R}^n, ~Y\in \{-1,1\}$ and interpretation of $\eta(\mathbf{x}) = P(Y=1|\mathbf{X})$ as the conditional expectation, we are interested in knowing what is the corresponding $P,Q$ in the binary classification framework.
	\begin{eqnarray*}
		E[Q|P] &=& P(Y=1|\mathbf{X})~~a.s. \\
		&=& E[\mathbf{1}_{\{Y=1\}}|\mathbf{X}] ~~ a.s. 
	\end{eqnarray*}
	Therefore, we have $P = \mathbf{X}$ and $Q= \mathbf{1}_{\{Y=1\}}$ and the minimization problem is 
	\begin{equation} \label{eq: opt_prob}
	\min\limits_{Z}E_{\mathbf{X}\times Y}[(\mathbf{1}_{\{Y=1\}}-Z)^2]
	\end{equation}
	From above, $Z^* = \eta = P[Y = 1|\mathbf{X}]~a.s. $ and hence $Z^*\in [0,1] ~a.s. $ as a function of $\mathbf{X}$. This has some resemblance to the squared loss for CPE which is $l_{1} = (1-\hat{\eta})^2, ~~l_{-1}=\hat{\eta}^2.$ The advantage here is that the minimizer is the conditional expectation which in our case is the conditional probability. This need not be true for the log-loss ( $l_{1}(\hat{\eta}) = log(\hat{\eta}), ~~l_{-1}(\hat{\eta}) =log(1-\hat{\eta})$ ),  etc.
	
	Consider the class of positive valued functions with an element denoted by $\phi_l(\cdot,\cdot)$. Let $\phi_l$ be called a basis function. Later, we will be specializing this class to be the class of Gaussian kernels. Now, consider the following functional form of $Z$ in equation \eqref{eq: opt_prob}, 
	\begin{equation} \label{Z_form}
	Z = \sum\limits_{l=1}^{b}\alpha_l\phi_l(\mathbf{X},Y) = \boldsymbol{\alpha}^T\phi(\mathbf{X},{Y})
	\end{equation}
	where $\boldsymbol{\alpha} \in \mathbb{R}^b$ is the variable to be optimized over and $\phi(\mathbf{X},Y) \geq 0$ is the vector of  basis functions $\phi_{l}, ~~l= 1,\ldots,b$, i.e., $Z$ is a weighted combination of the positive basis functions. An $L_2$ regularized version of equation \eqref{eq: opt_prob} can be rewritten as follows:
	\begin{equation} \label{eq: opt_prob_reg}
	\min\limits_{\boldsymbol{\alpha} \in \mathbb{R}^b}E_{\mathbf{X}\times Y}[(\mathbf{1}_{\{Y=1\}}-\boldsymbol{\alpha}^T\phi(\mathbf{X},Y))^2] + \lambda\boldsymbol{\alpha}^T\boldsymbol{\alpha}
	\end{equation}
	Here, $\lambda >0$ is the regularization parameter which is usually determined using CV.
	If one considers the multi-class classification problem with $c$ classes, then the in-class probability for each class is defined as follows:
	\begin{equation}
	\eta_{y}(\mathbf{x}) = P(Y=y|\mathbf{X}=\mathbf{x}), ~~~ y\in \{1,2,\ldots,c\}
	\end{equation}
	The idea that expected squared loss is minimized by a suitable conditional expectation can be extended to multi class problems too by formulating following optimization problems, one for each class $y \in \{1,2,\ldots,c\}$, to get the $\eta_{y}(\mathbf{x})$'s.
	\begin{equation}
	\min\limits_{\boldsymbol{\alpha}^{(y)}}E_{\mathbf{X}\times Y}[(\mathbf{1}_{\{Y=y\}}-\boldsymbol{\alpha}^{(y)T}\phi(\mathbf{X},Y))^2] + \lambda\boldsymbol{\alpha}^{(y)T}\boldsymbol{\alpha}^{(y)}, ~~~ \lambda >0
	\end{equation}
	Solving for $\boldsymbol{\alpha}^{(y)}$ will lead to the following system of equations,
	\begin{equation}
	\boldsymbol{\alpha}^{(y)*}_{reg} = [E[\phi(\mathbf{X},Y)\phi^T(\mathbf{X},Y)] + \lambda I_{b}]^{-1}E[\phi(\mathbf{X},Y)\mathbf{1}_{\{Y=y\}}]
	\end{equation}
	Given the multi-class data $(\mathbf{x}_i,y_i)$, $i \in \{1,2,\ldots,m\}, \mathbf{x}_i \in \mathbf{R}^n, ~ y \in \{1,\ldots,c\}$, we have the sample versions of the solution as follows:
	\begin{eqnarray}
	\hat{\boldsymbol{\alpha}}^{(y)}_{reg} &=& \left[\underbrace{\frac{1}{m}\sum\limits_{i=1}^m\phi(\mathbf{x}_i,y_i)\phi^T(\mathbf{x}_i,y_i)}_{\hat{A}} + \lambda \mathbf{1}_b\right]^{-1}\left[\underbrace{\frac{1}{m}\sum\limits_{i\in m_y}\phi(\mathbf{x}_i,y_i)\mathbf{1}_{ \{y_i=y\} }}_{\hat{a}^{(y)}}\right] \\
	&=& [\hat{A}+\lambda\mathbf{1}_b]^{-1}\hat{\mathbf{a}}^{(y)}
	\end{eqnarray}
	where $m_y$ is the index of examples with label $y$. One important question here is to decide about the nature of basis functions $\phi$. Kernel function $K(\mathbf{x},\mathbf{x}^{\prime})$ comes as natural choice for the basis functions due to the information they carry. Using these kernels, the elements of $\hat{A}$ and $\hat{\mathbf{a}}^{(y)}$ can be written as follows:
	\begin{eqnarray*}
		\hat{A}_{l,l'} &=& \frac{1}{m}\sum\limits_{i=1}^{m}K(\mathbf{x}_i,\mathbf{x}_l)K(\mathbf{x}_{i},\mathbf{x}_{l'}) \\
		\hat{\mathbf{a}}^{(y)} &=& \frac{1}{m}\sum\limits_{i \in m_y}K(\mathbf{x}_i,\mathbf{x}_l)
	\end{eqnarray*}
	Therefore, here we would get $c$ sets of $\hat{\boldsymbol{\alpha}}^{(y)}$, one from each system of equation. To get the final estimates of the in-class probabilities, one has to do the max and the normalization as follows:
	\begin{equation} \label{eq: final_eta_multiclass}
	\hat{\eta}_{y}(x) = \frac{\max\{0, \sum\limits_{l=1}^{m}\hat{\boldsymbol{\alpha}}^{(y)}K(\mathbf{x},\mathbf{x}_l)\}}{\sum\limits_{y'=1}^{c}\max\{0, \sum\limits_{l=1}^{m}\hat{\boldsymbol{\alpha}}^{(y')}K(\mathbf{x},\mathbf{x}_l)\}} ~~~ \forall y \in \{1,2,\ldots,c\}
	\end{equation}
	In-class probability for a binary classification problem can be written as a special case of equation \eqref{eq: final_eta_multiclass}.
	
	
	
	\begin{remark}
		It is not necessary that the regularized expected squared loss is minimized by conditional expectation. Mathematically,
		\begin{equation}
		\min\limits_{\boldsymbol{\alpha} \in \mathbf{R}^b}E[(\mathbf{1}_{\{Y=1\}}-Z)^2] + \lambda \boldsymbol{\alpha}^T\boldsymbol{\alpha}
		\end{equation}
		where $Z = \boldsymbol{\alpha}^T \phi(\mathbf{X},Y)$. Here, it is not necessary that $Z^* = \eta= E[\mathbf{1}_{\{Y=1\}}|\mathbf{X}]$. So, there would be an approximation error if one tries to use minimizer of regularized expected squared loss to estimate in-class probability. Let us denote the estimates of in-class probability from regularized and unregularized version by \eqref{eq: opt_prob}  as $\hat{\eta}_{reg}$ and $\hat{\eta}$ respectively. We expect $\hat{\eta}$ to have model explaining properties and $\hat{\eta}_{reg}$ to have good prediction properties due to regularization and avoidance of over-fitting. 
	\end{remark}
	
	\begin{remark}
		It can be observed that the solution in equation \eqref{eq: final_eta_multiclass} and \eqref{eq: LSPC_final_sol} are same even though the initial objective is quite different. Therfore, in the experiments, we denote the results by the name of LSPC.
	\end{remark}	
	
	\subsection{Various estimators based on KLIEP} \label{sec: eta_KLIEP}
	We use KLIEP to estimate $\frac{P(\mathbf{x}|Y=y)}{P(\mathbf{x})}$, ratio of two continuous densities. One requirement in this method is that the ratio is such that its product with estimated class marginal, $\hat{\pi} =  \hat{P}(Y=y)$ is a valid probability i.e., $\hat{P}(Y=y|\mathbf{x}) < 1.$ This condition need not be true in general as we observed in our experiments. 
	
	For choosing the number of basis functions $b$, the authors in \cite{sugiyama2008direct} suggest to use $b = \min(100, n_{nu})$. In our case, $n_{nu}$ would be number of points from class $y$. As seen in equation \eqref{KLIEP_depen}, one can either use KLIEP with positive class points in the numerator and get $\hat{{\eta}}$ or use KLIEP twice, one with positive class in the numerator and other with negative class in the numerator and get  $\hat{{\eta}}$.
	
	Based on these observations, we have come up with following estimates of $\eta$ using KLIEP. 
	\begin{enumerate}
		\item $\mathbf{\hat{\eta}_{pos}:}$ This estimator is obtained by implementing KLIEP with positive class points in the numerator and obtaining the ratio $\hat{w}_{pos}$. That is, the ratio of density to be estimated is $\frac{P(\mathbf{x}|Y=1)}{P(\mathbf{X})}.$
		\begin{equation*}
		\hat{\eta}_{pos}(\mathbf{x}) = \hat{w}_{pos}(\mathbf{x})\hat{\pi}
		\end{equation*}
		This estimator of $\eta$ integrates to 1 but there is no guarantee that it will always be less than 1.
		\item $\mathbf{\hat{\eta}_{neg}:}$ This estimator is obtained by implementing KLIEP with negative class points in the numerator and obtaining the ratio $\hat{w}_{neg}$. That is, the ratio of density to be estimated is $\frac{P(\mathbf{x}|Y=-1)}{P(\mathbf{X})}.$
		\begin{equation*}
		\hat{\eta}_{neg}(\mathbf{x}) = 1 - \hat{w}_{neg}(\mathbf{x})(1-\hat{\pi})
		\end{equation*}
		This estimator of $\eta$ integrates to 1 but it is prone to being less than 0.
		\item $\mathbf{\hat{\eta}_{norm}:}$ This estimator is obtained by implementing two KLIEPs, one with positive class in the numerator and other with negative class in the numerator. Then, the estimates are combined as follows:
		\begin{equation*}
		\hat{\eta}_{norm} = \frac{\hat{\eta}_{pos}}{\hat{\eta}_{pos}+\hat{\eta}_{neg}}
		\end{equation*}
		This $\eta$ estimate is a valid density as it integrates to 1 and always lies in the interval $[0,1].$ 
		\item $\mathbf{\hat{\eta}_{pos,s}:}$ This is a scaled version of $\hat{\eta}_{pos}$. Let $\hat{\eta}_{max,p}$ be the maximum value of $\hat{\eta}_{pos}$ in the training data. Then,
		\begin{equation*}
		\hat{\eta}_{pos,s} = \frac{\hat{\eta}_{pos}}{\hat{\eta}_{max,p}}
		\end{equation*}
		This estimator will integrate to 1 and will be in the interval $[0,1]$ for training points but for test points it can cross 1.
		\item $\mathbf{\hat{\eta}_{neg,s}:}$ This is a scaled version of $\hat{\eta}_{neg}$. Let $\hat{\eta}_{max,n}$ be the maximum value of $\hat{\eta}_{neg}$ in the training data. Then,
		\begin{equation*}
		\hat{\eta}_{neg,s} = \frac{\hat{\eta}_{neg}}{\hat{\eta}_{max,n}}
		\end{equation*}
		This estimator will integrate to 1 and will be in the interval $[0,1]$ for training points but for test points it can take values less than 0.
		\item $\mathbf{\hat{\eta}_{norm,s}:}$ This is obtained by normalizing the estimators $\hat{\eta}_{pos,s}$ and $\hat{\eta}_{neg,s}$ as follows:
		\begin{equation*}
		\hat{\eta}_{norm,s} = \frac{\hat{\eta}_{pos,s}}{\hat{\eta}_{pos,s}+\hat{\eta}_{neg,s}}
		\end{equation*}
		This is a valid probability density estimate as it integrates to 1 and lies in $[0,1].$
		\item $\mathbf{\hat{\eta}_{Cnorm,s}:}$ This is also a normalized estimator but the normalizing is different from $\hat{\eta}_{norm,s}$. It is given as below:
		\begin{equation*}
		\hat{\eta}_{Cnorm,s} = \frac{\hat{\eta}_{neg,s}}{\hat{\eta}_{neg,s}+\hat{\eta}_{pos,s}}
		\end{equation*}
		This is a valid probability density estimate as it integrates to 1 and lies in $[0,1].$
		\item $\mathbf{\hat{\eta}_{avg,s}:}$ This is the averaged version of scaled estimators $\hat{\eta}_{pos,s}$ and $\hat{\eta}_{neg,s}$ and given as follows:
		\begin{equation*}
		\hat{\eta}_{avg,n} = \frac{\hat{\eta}_{pos,s} + \hat{\eta}_{neg,s}}{2} 
		\end{equation*}
		This estimator integrates to 1 but for test data points it is prone to lying outside the interval $[0,1]$.
	\end{enumerate}
	The estimators which does not lie inside the interval $[0,1]$ can be truncated to lie inside the interval. However, there performance might be affected. A compact form of the above estimators is available in Table \ref{tab: KLIEP_diff_est}. 
	\begin{table}[H]
		\centering
		\begin{tabular}{|c|c|}
			\hline
			\textbf{$\hat{\eta}$} & \textbf{$\hat{\eta}(\mathbf{x})\in [0,1]$ for $\mathbf{x}$ in test set} \\ \hline
			$\hat{\eta}_{pos}(\mathbf{x}) = \hat{w}_{pos}(\mathbf{x})\hat{\pi}$ & May not  \\ \hline
			$\hat{\eta}_{neg}(\mathbf{x}) = 1 - \hat{w}_{neg}(\mathbf{x})(1-\hat{\pi})$& May not \\ \hline
			$\hat{\eta}_{norm} = \frac{\hat{\eta}_{pos}}{\hat{\eta}_{pos}+\hat{\eta}_{neg}}$& Yes \\ \hline
			$\hat{\eta}_{pos,s} = \frac{\hat{\eta}_{pos}}{\hat{\eta}_{max,p}}$ & May not \\ \hline
			$\hat{\eta}_{neg,s} = \frac{\hat{\eta}_{neg}}{\hat{\eta}_{max,n}}$& May not \\ \hline
			$\hat{\eta}_{norm,s} = \frac{\hat{\eta}_{pos,s}}{\hat{\eta}_{pos,s}+\hat{\eta}_{neg,s}}$& Yes \\  \hline
			$\hat{\eta}_{Cnorm,s} = \frac{\hat{\eta}_{neg,s}}{\hat{\eta}_{neg,s}+\hat{\eta}_{pos,s}}$& Yes \\  \hline
			$\hat{\eta}_{avg,n} = \frac{\hat{\eta}_{pos,s} + \hat{\eta}_{neg,s}}{2} $& May not \\ \hline
		\end{tabular}
		\caption{Details about various KLIEP based $\eta$ estimates and whether they are valid probability or not. We only check for $\hat{\eta} \in [0,1]$ as all the estimates integrate to 1.}
		\label{tab: KLIEP_diff_est}
	\end{table}
	The comparison between these estimators w.r.t. various measures is done in Section \ref{sec: Exp_eta}.  
	In our implementation, we have used Gaussian kernels as the basis functions with centres as the sample points from the numerator density (positive/negative class). 
	The kernel width $\sigma$ is tuned by KLIEP's inbuilt CV procedure where the criterion is maximum objective value of the KL based optimization problem. The possible candidates for kernel width $\sigma$ to be used in KLIEP's CV procedure were generated through Centile method described in \ref{alg: centile_method}.

	\subsection{An experimental comparative study of various in-class probability estimation methods} \label{sec: Exp_eta}
	
	In this section, we compare the performance of the $\eta$ estimation methods described in Section \ref{sec: link_classifier_eta}, \ref{sec: eta_LSPC} and \ref{sec: eta_KLIEP}. We implemented \ref{alg: eta_via_classifier} with squared loss, logistic loss and modified squared loss and LSPC, KLIEP and $k$-NN on 3 synthetic datasets so that we can comment on the quality of $\eta$ estimates. In case of $k$-NN, the value of number of neighbours $k$ to be used is selected by cross validating w.r.t. the measure to be evaluated on. The dataset generation scheme is adapted from \cite{efron1997improvements}. We consider 9 performance measures out of which first 5 (MSE, RMSE, MAD, MD, KL) are purely for $\eta$ estimation, next 2 (Acc, BS) are for the scenario when these estimates are used in label prediction and the last 2 measure the algorithms estimation capability at the boundary i.e., maximum and minimum. For an estimate $\hat{\eta}$, when the train data is $\{(\mathbf{x}_i,y_i)\}_{i=1}^{m_{tr}}$ and test data set is $\{(\mathbf{x}_i,y_i)\}_{i=1}^{m_{te}}$, the details about all these measures are given below:
	\begin{enumerate}
		\item Mean squared error (MSE): $MSE = \frac{1}{m_{te}}\sum\limits_{i=1}^{m_{te}}(\hat{\eta}(\mathbf{x}_{i}) - \eta(\mathbf{x}_i))^2$
		\item Root mean squared error (RMSE): $RMSE = \sqrt[]{MSE}$
		\item Mean absolute deviation (MAD): $MAD = \frac{1}{m_{te}}\sum\limits_{i=1}^{m_{te}}\vert \hat{\eta}(\mathbf{x}_{i}) - \eta(\mathbf{x}_i)\vert$
		\item Mean deviation (MD): $MD = \frac{1}{m_{te}}\sum\limits_{i=1}^{m_{te}} (\hat{\eta}(\mathbf{x}_{i}) - \eta(\mathbf{x}_i))$
		\item Averaged Kullback-Leibler divergence (KL): Since $\eta$ is a density, we can compute the KL between its estimate and true value as follows:
		\begin{eqnarray*}
			KL &=& \frac{1}{m}\sum\limits_{i=1}^{m_{te}}KL[\eta(\mathbf{x}_{i})\Vert\hat{\eta}(\mathbf{x}_{i})] \\
			&=& \frac{1}{m}\sum\limits_{i=1}^{m_{te}} \eta(\mathbf{x}_i)\log\frac{\eta(\mathbf{x}_{i})}{\hat{\eta}(\mathbf{x}_{i})} + (1-\eta(\mathbf{x}_i))\log\frac{1-\eta(\mathbf{x}_{i})}{1-\hat{\eta}(\mathbf{x}_{i})}
		\end{eqnarray*}
		\item Accuracy (Acc): $Acc = \frac{1}{m_{te}}\sum\limits_{i=1}^{m_{te}}\mathbf{1}_{[(2\hat{\eta}(\mathbf{x}_{i}) -1)y_{i} \leq 0]}$
		\item Brier Score (BS): Brier score is a proper score function which measures the accuracy of probabilistic predictions and is defined below:
		\begin{equation*}
		BS = \frac{1}{m_{te}}\sum\limits_{i=1}^{m_{te}}(\hat{\eta}(\mathbf{x}_i) -y_{i})^2
		\end{equation*}
		\item DiffMax: In some applications, one requires how good the estimates are at the boundary. In such cases, we have to compare the maximum value of true $\eta$ to the maximum value of $\hat{\eta}$ on the training set. This measure plays an important role when one is interested in estimating noise rates as some noise estimating schemes use corrupted ${\eta}$ values \cite{liu2016classification}. Therefore,
		$$ DiffMax : = \max_{i=1,\ldots,m_{tr}}\hat{\eta}(\mathbf{x}_{i}) - \max_{i=1,\ldots,m_{tr}}\eta(\mathbf{x}_i)$$
		\item DiffMin: Similar to DiffMax, this measure is also a deciding factor for quality of noise rates. 
		$$ DiffMin : = \min_{i=1,\ldots,m_{tr}}\hat{\eta}(\mathbf{x}_{i}) - \min_{i=1,\ldots,m_{tr}}\eta(\mathbf{x}_i)$$
	\end{enumerate}
	The synthetic dataset is partitioned into training and test set with $80$-$20$ split. The first 7 performance measures are computed on the test set and the last two on the training set. To make the comparisons more robust, this process is repeated for 10 trials. Finally, the reported values are performance measures averaged over the 10 trials along with the standards deviation.
	\subsubsection{Synthetic dataset 1 ($Syn_{2D}$)} 
	This is a 2 dimensional dataset. We first generate 1000 binary class labels $Y$ from Bernoulli distribution with parameter $p=0.5$. Then a 2-dimensional feature vector $\mathbf{X}$ for each label is drawn from two different Gaussian distributions: $\mathbf{X}|Y =1 \sim N([1.2,0],\Sigma)$ $\&$ $\mathbf{X}|Y =-1 \sim N([-1.2,0],\Sigma)$ where $\Sigma = [[1, 0.4];[0.4,1]]$. Table \ref{tab: eta_syn2D_cls_LSPC}, \ref{tab: eta_syn2D_KLIEP_123} and \ref{tab: eta_syn2D_KLIEP_45678} present the values of various measures for this dataset.
	
	\begin{table}[htbp!]
		\centering
		\bgroup
		\def\arraystretch{0.85}%
		{\footnotesize
			\begin{tabular}{|c|c|c|c|c|c|}
				\hline
				&\multicolumn{5}{c|}{\textbf{ ${\eta}$ estimated from }}      \\ \hline        
				\textbf{\begin{tabular}[c]{@{}c@{}}Performance \\ measures \end{tabular}} & \textbf{\begin{tabular}[c]{@{}c@{}}\ref{alg: eta_via_classifier} \\ with $l_{sq}$\end{tabular}} & \textbf{\begin{tabular}[c]{@{}c@{}}\ref{alg: eta_via_classifier}\\ with $l_{log}$\end{tabular}} & \textbf{\begin{tabular}[c]{@{}c@{}}\ref{alg: eta_via_classifier}\\ with $l_{msq}$\end{tabular}} & \textbf{LSPC} & \textbf{k-NN} \\ \hline
				\textbf{MSE} & 0.02$\pm$0.002 & 0.0008$\pm$0.0007 & 0.001$\pm$0.0003 & 0.003$\pm$0.002 & 0.013  $\pm$  0.003\\ \hline
				\textbf{RMSE} & 0.144$\pm$0.007 & 0.025$\pm$0.012 & 0.034$\pm$0.004 & 0.054$\pm$0.02 & 0.114  $\pm$  0.014\\ \hline
				\textbf{MAD} & 0.123$\pm$0.005 & 0.018$\pm$0.010 & 0.025$\pm$0.002 & 0.037$\pm$0.015 & 0.066  $\pm$  0.008\\ \hline
				\textbf{MD} & -0.0032$\pm$0.011 & -0.0076$\pm$0.001 & -0.007$\pm$0.002 & -0.007$\pm$0.004 & -0.007 $\pm$  0.005\\ \hline
				\textbf{KL} & NA & 0.004$\pm$0.004 & NA & NA & NA\\ \hline
				\textbf{Acc} & 0.911$\pm$0.016 & 0.91$\pm$0.017 & 0.91$\pm$0.018 & 0.91$\pm$0.01 & 0.906  $\pm$  0.016\\ \hline
				\textbf{BS} & 0.0804$\pm$0.004 & 0.0651$\pm$0.010 & 0.0657$\pm$0.007 & 0.06809$\pm$0.0068 & 0.077 $\pm$  0.010\\ \hline
				\textbf{DiffMax} & 0.531$\pm$0.04 & 0.0 $\pm$0.0 & 0.0 $\pm$0.0 & 0.0 $\pm$0.0 &  0.0 $\pm$0.0\\ \hline
				\textbf{DiffMin} & -0.497$\pm$0.042 & 0.0 $\pm$0.0 & 0.0 $\pm$0.0 & 0.0 $\pm$0.0 &  0.0 $\pm$0.0\\ \hline
			\end{tabular}
		}\egroup
		\caption{Performance of estimates obtained by \ref{alg: eta_via_classifier} with squared loss, logistic loss and modified squared loss and LSPC and $k$-NN on $Syn_{2D}$ dataset. The cells with NA means that some estimate was 0 and the KL couldn't be defined. Cells with $0.0$ entry means that the value is less than  $1 \times 10^{-5}$.}
		\label{tab: eta_syn2D_cls_LSPC}
	\end{table}

	\begin{table}[htbp!]
		\centering
		\bgroup
		\def\arraystretch{0.85}%
		{\setlength{\tabcolsep}{0.25em}	
			{\footnotesize
				\begin{tabular}{|c|c|c|c|}
					\hline
					&\multicolumn{3}{c|}{\textbf{ KLIEP based estimator used }}      \\ \hline   
					\textbf{\begin{tabular}[c]{@{}c@{}}Performance \\ measures \end{tabular}} & $\mathbf{\hat{\eta}_{pos}}$ & $\mathbf{\hat{\eta}_{neg}}$ & $\mathbf{\hat{\eta}_{norm}}$ \\ \hline
					\textbf{MSE} & 0.012$\pm$ 0.003 & 0.014$\pm$0.007 & 0.001$\pm$ 0.00 \\ \hline
					\textbf{RMSE} & 0.108$\pm$ 0.015 & 0.118$\pm$0.027 & 0.034$\pm$0.009 \\ \hline
					\textbf{MAD} & 0.062$\pm$ 0.006 & 0.069$\pm$0.016 & 0.021$\pm$0.005  \\ \hline
					\textbf{MD} & -0.006$\pm$0.008 & --0.006$\pm$0.013 & -0.003$\pm$0.005 \\ \hline
					\textbf{KL} & NA & NA & 0.005$\pm$ 0.002  \\ \hline
					\textbf{Acc} & 0.89 $\pm$ 0.02 & 0.90 $\pm$ 0.01 & 0.91 $\pm$ 0.02 \\ \hline
					\textbf{BS} & 0.075 $\pm$ 0.011 & 0.074$\pm$ 0.006 & 0.065$\pm$0.006 \\ \hline
					\textbf{DiffMax} & 0.16$\pm$0.02 & 0.0 $\pm$0.0 & 0.0 $\pm$0.0 \\ \hline
					\textbf{DiffMin} & 0.0 $\pm$0.0 & -0.16 $\pm$0.06 & 0.0 $\pm$0.0 \\ \hline
		\end{tabular}}}
		\egroup
		\caption{Performance of $\hat{\eta}_{pos}, \hat{\eta}_{neg}, \hat{\eta}_{norm}$, KLIEP based estimates on $Syn_{2D}$ dataset. The cells with NA means that some estimate was $0$ and the KL couldn't be defined. Cells with $0.0$ entry means that the value is less than  $1 \times 10^{-5}$.}
		\label{tab: eta_syn2D_KLIEP_123}
	\end{table}
	
	\begin{table}[htbp!]
		\centering
		\bgroup
		\def\arraystretch{0.85}%
		{\setlength{\tabcolsep}{0.25em}	
			{\footnotesize
				\begin{tabular}{|c|c|c|c|c|c|}
					\hline
					&\multicolumn{5}{c|}{\textbf{ KLIEP based estimator used }}      \\ \hline   
					\textbf{\begin{tabular}[c]{@{}c@{}}Performance \\ measures \end{tabular}} & $\mathbf{\hat{\eta}_{pos,s}}$ & $\mathbf{\hat{\eta}_{neg,s}}$ & $\mathbf{\hat{\eta}_{norm,s}}$ & $\mathbf{\hat{\eta}_{Cnorm,s}}$ & $\mathbf{\hat{\eta}_{avg,s}}$ \\ \hline
					\textbf{MSE} & 0.009$\pm$0.004 & 0.019$\pm$0.01 & 0.001$\pm$0.00 & 0.178$\pm$0.004 & 0.009$\pm$0.002 \\ \hline
					\textbf{RMSE} &  0.096$\pm$0.019 & 0.133$\pm$0.035 & 0.036$\pm$0.009 & 0.133$\pm$0.013 & 0.097$\pm$0.01 \\ \hline
					\textbf{MAD} &  0.095$\pm$0.008 & 0.072$\pm$0.027 & 0.021$\pm$0.005 & 0.105$\pm$0.014 & 0.072$\pm$0.01 \\ \hline
					\textbf{MD} &  -0.08$\pm$0.008 & 0.06$\pm$0.027 & -0.003$\pm$0.007 & -0.019$\pm$0.018 & -0.012$\pm$0.014 \\ \hline
					\textbf{KL} &  0.073 $\pm$0.002 & NA & 0.005 $\pm$ 0.002 & NA & 0.047$\pm$0.008 \\ \hline
					\textbf{Acc} &  0.87 $\pm$ 0.02 & 0.89 $\pm$ 0.02 & 0.90 $\pm$ 0.01 & 0.90 $\pm$0.01 & 0.90 $\pm$ 0.02 \\ \hline
					\textbf{BS} &  0.086$\pm$0.010 & 0.084$\pm$0.009 & 0.066$\pm$0.006 & 0.083$\pm$0.006 & 0.0748$\pm$0.005 \\ \hline
					\textbf{DiffMax} &  0.0 $\pm$0.0 & 0.0 $\pm$0.0 & 0.0 $\pm$0.0 & 0.0 $\pm$0.0 & -0.003$\pm$0.00 \\ \hline
					\textbf{DiffMin} &  0.0 $\pm$0.0 & 0.0 $\pm$0.0 & 0.0 $\pm$0.0 & 0.0 $\pm$0.0 & -0.005$\pm$0.004 \\ \hline
		\end{tabular}}}
		\egroup
		\caption{Performance of $\hat{\eta}_{pos,s},\hat{\eta}_{neg,s},\hat{\eta}_{norm,s},\hat{\eta}_{Cnorm,s},\hat{\eta}_{avg,s}$, KLIEP based estimates on $Syn_{2D}$ dataset. The cells with NA means that some estimate was $0$ and the KL couldn't be defined. Cells with $0.0$ entry means that the value is less than  $1 \times 10^{-5}$.}
		\label{tab: eta_syn2D_KLIEP_45678}
	\end{table}
	
	\subsubsection{Synthetic dataset 2 ($Syn_{3D}$)}
	This is a 3 dimensional dataset. We first generate 1000 binary class labels $Y$ from Bernoulli distribution with parameter $p=0.5$. Then a 3-dimensional feature vector $\mathbf{X}$ for each label is drawn from two different Gaussian distributions: $\mathbf{X}|Y =1 \sim N([1.3,0,0],\Sigma)$ $\&$ $\mathbf{X}|Y =-1 \sim N([-1.3,0,0],\Sigma)$ where $\Sigma = [[1, 0.1,0.1];[0.1,1,0.1];[0.1,0.1,1]]$. Table \ref{tab: eta_syn3D_cls_LSPC}, \ref{tab: eta_syn3D_KLIEP_123} and \ref{tab: eta_syn3D_KLIEP_45678} present the values of various measures for this dataset.
	
	\begin{table}[htbp!]
		\centering
		\bgroup
		\def\arraystretch{0.85}%
		{\footnotesize
			\begin{tabular}{|c|c|c|c|c|c|}
				\hline
				&\multicolumn{5}{c|}{\textbf{ ${\eta}$ estimated from }}      \\ \hline        
				\textbf{\begin{tabular}[c]{@{}c@{}}Performance \\ measures \end{tabular}} & \textbf{\begin{tabular}[c]{@{}c@{}}\ref{alg: eta_via_classifier} \\ with $l_{sq}$\end{tabular}} & \textbf{\begin{tabular}[c]{@{}c@{}}\ref{alg: eta_via_classifier}\\ with $l_{log}$\end{tabular}} & \textbf{\begin{tabular}[c]{@{}c@{}}\ref{alg: eta_via_classifier}\\ with $l_{msq}$\end{tabular}} & \textbf{LSPC} & \textbf{k-NN} \\ \hline
				\textbf{MSE} & 0.02$\pm$0.001 & 0.0008$\pm$0.0007 & 0.001$\pm$0.0003 & 0.003$\pm$0.002 & 0.0108 $\pm$  0.0022\\ \hline
				\textbf{RMSE} & 0.145$\pm$0.007 & 0.026$\pm$0.007 & 0.035$\pm$0.003 & 0.064$\pm$0.003 & 0.103  $\pm$  0.0102 \\ \hline
				\textbf{MAD} & 0.126$\pm$0.003 & 0.018$\pm$0.007 & 0.025$\pm$0.002 & 0.045$\pm$0.015 & 0.0612  $\pm$ 0.005\\ \hline
				\textbf{MD} & -0.016$\pm$0.008 & -0.0078$\pm$0.002 & -0.006$\pm$0.0036 & -0.009$\pm$0.005 & -0.008  $\pm$  0.007\\ \hline
				\textbf{KL} & NA & 0.004$\pm$0.003 & NA & NA & NA\\ \hline
				\textbf{Acc} & 0.92$\pm$0.014 & 0.923$\pm$0.015 & 0.924$\pm$0.014 & 0.923$\pm$0.016 & 0.929  $\pm$  0.0197\\ \hline
				\textbf{BS} & 0.0820$\pm$0.006 & 0.0638$\pm$0.008 & 0.0663$\pm$0.008 & 0.0692$\pm$0.008 & 0.0667  $\pm$  0.0127\\ \hline
				\textbf{DiffMax} & 0.628$\pm$0.01 & 0.0 $\pm$0.0 & 0.0 $\pm$0.0 & 0.0 $\pm$0.0 & 0.0 $\pm$0.0\\ \hline
				\textbf{DiffMin} & -0.458$\pm$0.01 & 0.0 $\pm$0.0 & 0.0 $\pm$0.0 & 0.0 $\pm$0.0 & 0.0 $\pm$0.0\\ \hline
		\end{tabular}}
		\egroup
		\caption{Performance of estimates obtained by \ref{alg: eta_via_classifier} with squared loss, logistic loss and modified squared loss and LSPC and $k$-NN on $Syn_{3D}$ dataset. The cells with NA means that some estimate was 0 and the KL couldn't be defined. Cells with $0.0$ entry means that the value is less than  $1 \times 10^{-5}$.}
		\label{tab: eta_syn3D_cls_LSPC}
	\end{table}

	\begin{table}[htbp!]
		\centering
		\bgroup
		\def\arraystretch{0.85}%
		{\setlength{\tabcolsep}{0.25em}	
			{\footnotesize
				\begin{tabular}{|c|c|c|c|}
					\hline
					&\multicolumn{3}{c|}{\textbf{ KLIEP based estimator used }}      \\ \hline   
					\textbf{\begin{tabular}[c]{@{}c@{}}Performance \\ measures \end{tabular}} & $\mathbf{\hat{\eta}_{pos}}$ & $\mathbf{\hat{\eta}_{neg}}$ & $\mathbf{\hat{\eta}_{norm}}$ \\ \hline
					\textbf{MSE} & 0.022$\pm$ 0.003 & 0.025$\pm$0.005 & 0.002$\pm$ 0.00  \\ \hline
					\textbf{RMSE} & 0.146$\pm$ 0.010 & 0.158$\pm$0.015 & 0.050$\pm$0.005  \\ \hline
					\textbf{MAD} & 0.097$\pm$ 0.007 & 0.103$\pm$0.007 & 0.040$\pm$0.005  \\ \hline
					\textbf{MD} & -0.0004$\pm$0.009 & --0.009$\pm$0.019 & -0.006$\pm$0.004  \\ \hline
					\textbf{KL} & NA & NA & 0.014$\pm$ 0.003  \\ \hline
					\textbf{Acc} & 0.90 $\pm$ 0.02 & 0.89 $\pm$ 0.0 & 0.92 $\pm$ 0.016  \\ \hline
					\textbf{BS} & 0.0793 $\pm$ 0.0077 & 0.084$\pm$ 0.010 & 0.067$\pm$0.007  \\ \hline
					\textbf{DiffMax} & 0.34$\pm$0.02 & 0.0 $\pm$0.0 & 0.0 $\pm$0.0  \\ \hline
					\textbf{DiffMin} & 0.0002 $\pm$0.0 & -0.344 $\pm$0.02 & 0.006 $\pm$0.002  \\ \hline
		\end{tabular}}}
		\egroup
		\caption{Performance of $\hat{\eta}_{pos}, \hat{\eta}_{neg}, \hat{\eta}_{norm}$, KLIEP based estimates on $Syn_{3D}$ dataset. The cells with NA means that some estimate was $0$ and the KL couldn't be defined. Cells with $0.0$ entry means that the value is less than  $1 \times 10^{-5}$.}
		\label{tab: eta_syn3D_KLIEP_123}
	\end{table}
	
	\begin{table}[htbp!]
		\centering
		\bgroup
		\def\arraystretch{0.85}%
		{\setlength{\tabcolsep}{0.25em}	
			{\footnotesize
				\begin{tabular}{|c|c|c|c|c|c|}
					\hline
					&\multicolumn{5}{c|}{\textbf{ KLIEP based estimator used }}      \\ \hline   
					\textbf{\begin{tabular}[c]{@{}c@{}}Performance \\ measures \end{tabular}} & $\mathbf{\hat{\eta}_{pos,s}}$ & $\mathbf{\hat{\eta}_{neg,s}}$ & $\mathbf{\hat{\eta}_{norm,s}}$ & $\mathbf{\hat{\eta}_{Cnorm,s}}$ & $\mathbf{\hat{\eta}_{avg,s}}$ \\ \hline
					\textbf{MSE} &  0.014$\pm$0.004 & 0.044$\pm$0.001 & 0.03$\pm$0.005 & 0.041$\pm$0.003 & 0.024$\pm$0.002 \\ \hline
					\textbf{RMSE} &  0.118$\pm$0.014 & 0.209$\pm$0.021 & 0.051$\pm$0.005 & 0.203$\pm$0.01 & 0.154$\pm$0.009 \\ \hline
					\textbf{MAD} &  0.164$\pm$0.008 & 0.141$\pm$0.014 & 0.042$\pm$0.005 & 0.183$\pm$0.010 & 0.136$\pm$0.009 \\ \hline
					\textbf{MD} & -0.151$\pm$0.004 & 0.122$\pm$0.0151 & -0.008$\pm$0.015 & -0.018$\pm$0.009 & -0.014$\pm$0.008 \\ \hline
					\textbf{KL} &  0.150 $\pm$0.011 & NA & 0.015 $\pm$ 0.003 & NA & 0.102$\pm$0.009 \\ \hline
					\textbf{Acc} & 0.8535 $\pm$ 0.01 & 0.83 $\pm$ 0.02 & 0.92 $\pm$ 0.01 & 0.92 $\pm$0.01 & 0.92 $\pm$ 0.02 \\ \hline
					\textbf{BS} &  0.1147$\pm$0.006 & 0.1144$\pm$0.011 & 0.0677$\pm$0.0.007 & 0.108$\pm$0.005 & 0.0905$\pm$0.005 \\ \hline
					\textbf{DiffMax} & 0.0 $\pm$0.0 & 0.0 $\pm$0.0 & 0.0 $\pm$0.0 & 0.0 $\pm$0.0 & -0.006$\pm$0.02 \\ \hline
					\textbf{DiffMin} & 0.00 $\pm$0.0 & 0.0 $\pm$0.0 & 0.0006 $\pm$0.0 & 0.0 $\pm$0.0 & -0.008$\pm$0.002 \\ \hline
		\end{tabular}}}
		\egroup
		\caption{Performance of $\hat{\eta}_{pos,s},\hat{\eta}_{neg,s},\hat{\eta}_{norm,s},\hat{\eta}_{Cnorm,s},\hat{\eta}_{avg,s}$, KLIEP based estimates on $Syn_{3D}$ dataset. The cells with NA means that some estimate was $0$ and the KL couldn't be defined. Cells with $0.0$ entry means that the value is less than  $1 \times 10^{-5}$.}
		\label{tab: eta_syn3D_KLIEP_45678}
	\end{table}

	\subsubsection{Synthetic dataset 3 ($Syn_{10D}$)}
	This is a 10 dimensional dataset. We first generate 1000 binary class labels $Y$ from Bernoulli distribution with parameter $p=0.5$. Then, a 10-dimensional feature vector $\mathbf{X}$ for each label is drawn from two different Gaussian distributions: $\mathbf{X}|Y =1 \sim N(\boldsymbol{\mu}_+,\Sigma)$ $\&$ $\mathbf{X}|Y =-1 \sim N(\boldsymbol{\mu}_-,\Sigma)$ where $\Sigma$ is such that all the diagonal elements (variances) are $1$ and the non-diagonal elements (covariances)  are $10$. Here, the means are $\boldsymbol{\mu}_{+} = [2, 1.2, 2, 2.1, 2,2, 2, 0.2, 2, 3.6]$ and $\boldsymbol{\mu}_{+} = [2, -1.2, 2, -2.1, 2, 2, 2, -0.2, 2, -3.6].$ Table \ref{tab: eta_syn10D_cls_LSPC}, \ref{tab: eta_syn10D_KLIEP_123} and \ref{tab: eta_syn10D_KLIEP_45678} present the values of various measures for this dataset.
	
	\begin{table}[htbp!]
		\centering
		\bgroup
		\def\arraystretch{0.85}%
		{\footnotesize
			\begin{tabular}{|c|c|c|c|c|c|}
				\hline
				&\multicolumn{5}{c|}{\textbf{ ${\eta}$ estimated from }}      \\ \hline        
				\textbf{\begin{tabular}[c]{@{}c@{}}Performance \\ measures \end{tabular}} & \textbf{\begin{tabular}[c]{@{}c@{}}\ref{alg: eta_via_classifier} \\ with $l_{sq}$\end{tabular}} & \textbf{\begin{tabular}[c]{@{}c@{}}\ref{alg: eta_via_classifier}\\ with $l_{log}$\end{tabular}} & \textbf{\begin{tabular}[c]{@{}c@{}}\ref{alg: eta_via_classifier}\\ with $l_{msq}$\end{tabular}} & \textbf{LSPC} & \textbf{k-NN} \\ \hline
				\textbf{MSE} & 0.026$\pm$0.002 & 0.002$\pm$0.0008 & 0.003$\pm$0.0008 & 0.011$\pm$0.002 & 0.0232  $\pm$  0.002\\ \hline
				\textbf{RMSE} & 0.161$\pm$0.006 & 0.047$\pm$0.009 & 0.005$\pm$0.007 & 0.106$\pm$0.010 & 0.152  $\pm$  0.007\\ \hline
				\textbf{MAD} & 0.138$\pm$0.005 & 0.026$\pm$0.005 & 0.031$\pm$0.004 & 0.077$\pm$0.011 & 0.0911  $\pm$  0.00756\\ \hline
				\textbf{MD} & 0.004$\pm$0.006 & 0.002$\pm$0.006 & 0.003$\pm$0.006 & 0.003$\pm$0.003 & 0.01434 $\pm$  0.005\\ \hline
				\textbf{KL} & NA & 0.010$\pm$0.003 & NA & NA & NA\\ \hline
				\textbf{Acc} & 0.912$\pm$0.018 & 0.911$\pm$0.016 & 0.908$\pm$0.016 & 0.913$\pm$0.017 & 0.9195 $\pm$  0.0181\\ \hline
				\textbf{BS} & 0.0795$\pm$0.008 & 0.0645$\pm$0.012 & 0.0636$\pm$0.012 & 0.0733$\pm$0.0102 & 0.0834 $\pm$  0.011 \\ \hline
				\textbf{DiffMax} & 0.463$\pm$0.03 & 0.0 $\pm$0.0 & 0.0 $\pm$0.0 & 0.0 $\pm$0.0 & 0.0 $\pm$0.0\\ \hline
				\textbf{DiffMin} & -0.551$\pm$0.024 & 0.0 $\pm$0.0 & 0.0 $\pm$0.0 & 0.0 $\pm$0.0 & 0.0 $\pm$0.0\\ \hline
		\end{tabular}}
		\egroup
		\caption{Performance of estimates obtained by \ref{alg: eta_via_classifier} with squared loss, logistic loss and modified squared loss and LSPC and $k$-NN on $Syn_{10D}$ dataset. The cells with NA means that some estimate was 0 and the KL couldn't be defined. Cells with $0.0$ entry means that the value is less than  $1 \times 10^{-5}$.}
		\label{tab: eta_syn10D_cls_LSPC}
	\end{table}
	
	\begin{table}[htbp!]
		\centering
		\bgroup
		\def\arraystretch{0.85}%
		{\setlength{\tabcolsep}{0.25em}	
			{\footnotesize
				\begin{tabular}{|c|c|c|c|}
					\hline
					&\multicolumn{3}{c|}{\textbf{ KLIEP based estimator used }}      \\ \hline   
					\textbf{\begin{tabular}[c]{@{}c@{}}Performance \\ measures \end{tabular}} & $\mathbf{\hat{\eta}_{pos}}$ & $\mathbf{\hat{\eta}_{neg}}$ & $\mathbf{\hat{\eta}_{norm}}$ \\ \hline
					\textbf{MSE} & 0.06$\pm$ 0.003 & 0.062$\pm$0.006 & 0.035$\pm$ 0.005  \\ \hline
					\textbf{RMSE} & 0.250$\pm$ 0.006 & 0.249$\pm$0.012 & 0.189$\pm$0.014  \\ \hline
					\textbf{MAD} & 0.208$\pm$ 0.007 & 0.207$\pm$0.013 & 0.177$\pm$0.014  \\ \hline
					\textbf{MD} & 0.0031$\pm$0.01 & 0.006$\pm$0.012 & 0.008$\pm$0.005  \\ \hline
					\textbf{KL} & NA & NA & 0.014$\pm$ 0.017  \\ \hline
					\textbf{Acc} & 0.848 $\pm$ 0.03 & 0.846 $\pm$ 0.017  \\ \hline
					\textbf{BS} & 0.125 $\pm$ 0.010 & 0.123$\pm$ 0.009 & 0.0995$\pm$0.006  \\ \hline
					\textbf{DiffMax} & 0.32$\pm$0.098 & -0.02 $\pm$0.006 & -0.04 $\pm$0.001  \\ \hline
					\textbf{DiffMin} & 0.015 $\pm$0.005 & -0.422 $\pm$0.12 & 0.031 $\pm$0.009 \\ \hline
		\end{tabular}}}
		\egroup
		\caption{Performance of $\hat{\eta}_{pos}, \hat{\eta}_{neg}, \hat{\eta}_{norm}$, KLIEP based estimates on $Syn_{10D}$ dataset. The cells with NA means that some estimate was $0$ and the KL couldn't be defined. Cells with $0.0$ entry means that the value is less than  $1 \times 10^{-5}$.}
		\label{tab: eta_syn10D_KLIEP_123}
	\end{table}
	
	\begin{table}[htbp!]
		\centering
		\bgroup
		\def\arraystretch{0.85}%
		{\setlength{\tabcolsep}{0.25em}	
			{\footnotesize
				\begin{tabular}{|c|c|c|c|c|c|}
					\hline
					&\multicolumn{5}{c|}{\textbf{ KLIEP based estimator used }}      \\ \hline   
					\textbf{\begin{tabular}[c]{@{}c@{}}Performance \\ measures \end{tabular}} & $\mathbf{\hat{\eta}_{pos,s}}$ & $\mathbf{\hat{\eta}_{neg,s}}$ & $\mathbf{\hat{\eta}_{norm,s}}$ & $\mathbf{\hat{\eta}_{Cnorm,s}}$ & $\mathbf{\hat{\eta}_{avg,s}}$ \\ \hline
					\textbf{MSE} &  0.035$\pm$0.005 & 0.061$\pm$0.003 & 0.104$\pm$0.014 & 0.036$\pm$0.005 & 0.084$\pm$0.006 \\ \hline
					\textbf{RMSE} &  0.247$\pm$0.006 & 0.321$\pm$0.022 & 0.191$\pm$0.013 & 0.290$\pm$0.01 & 0.260$\pm$0.006 \\ \hline
					\textbf{MAD} &  0.245$\pm$0.009 & 0.263$\pm$0.009 & 0.178$\pm$0.014 & 0.275$\pm$0.011 & 0.247$\pm$0.007 \\ \hline
					\textbf{MD} &  -0.181$\pm$0.03 & 0.157$\pm$0.035 & 0.019$\pm$0.007 & 0.018$\pm$0.012 & -0.019$\pm$0.010 \\ \hline
					\textbf{KL} &  0.262 $\pm$0.022 & NA & 0.143 $\pm$ 0.016 & NA & 0.228$\pm$0.009 \\ \hline
					\textbf{Acc} &  0.750 $\pm$ 0.044 & 0.736 $\pm$ 0.044 & 0.913 $\pm$ 0.016 & 0.913 $\pm$0.016 & 0.913 $\pm$ 0.016 \\ \hline
					\textbf{BS} &  0.158$\pm$0.0159 & 0.165$\pm$0.0158 & 0.099$\pm$0.006 & 0.149$\pm$0.008 & 0.1318$\pm$0.006 \\ \hline
					\textbf{DiffMax} &  0.0 $\pm$0.0 & -0.01 $\pm$0.006 & -0.040 $\pm$0.012 & 0.0 $\pm$0.0 & -0.08$\pm$0.018 \\ \hline
					\textbf{DiffMin} & 0.011 $\pm$0.005 & 0.0 $\pm$0.0 & 0.033 $\pm$0.009 & 0.0 $\pm$0.0 & 0.081$\pm$0.02 \\ \hline
		\end{tabular}}}
		\egroup
		\caption{Performance of $\hat{\eta}_{pos,s},\hat{\eta}_{neg,s},\hat{\eta}_{norm,s},\hat{\eta}_{Cnorm,s},\hat{\eta}_{avg,s}$, KLIEP based estimates on $Syn_{10D}$ dataset. The cells with NA means that some estimate was $0$ and the KL couldn't be defined. Cells with $0.0$ entry means that the value is less than  $1 \times 10^{-5}$.}
		\label{tab: eta_syn10D_KLIEP_45678}
	\end{table}
	
	\textbf{Observations}
	We observe that as far as label prediction is concerned \ref{alg: eta_via_classifier} with squared loss, logistic loss and modified squared loss, LSPC, $k$-NN and KLIEP's $\hat{\eta}_{norm}$ estimators perform equally well as their accuracy is the highest in comparison to others. This trend has been observed across all 3 synthetic datasets.
	
	With respect to DiffMax and DiffMin measures, \ref{alg: eta_via_classifier} with logistic loss and modified squared loss, LSPC, $k$-NN and KLIEP's $\hat{\eta}_{norm}$ estimators are good.
	
	\section{Risk minimization under weighted 0-1 loss in the extended cost space} \label{sec: 4cost_neg}
	
	In view of the negative result obtained with the $\alpha$-weighted $0$-$1$ loss function, we explore a larger cost space. In a similar kind of representation as that of $C_1, C_{-1}$, one can extend the parameter space from $(C_1,C_{-1})$ to $(C_1,C_{-1},c_1,c_{-1})$ given as follows.
	\begin{itemize}
		\item[$\bullet$] $C_1$ (or $C_{-1}$): loss when a positive (or negative)  class point is misclassified
		\item[$\bullet$] $c_1$ (or $c_{-1}$): loss when a positive (or negative) class point is correctly classified
	\end{itemize}
	Let the extended cost sensitive $0$-$1$ loss be defined as
	\begin{small}
		\begin{equation}
		l_{c,4}(f(\mathbf{x}),y) = \mathbf{1}_{\{f(\mathbf{x}) > 0, y=1\}}c_{1} + \mathbf{1}_{\{f(\mathbf{x}) > 0, y=-1\}}C_{-1}  + \mathbf{1}_{\{f(\mathbf{x}) \leq 0, y=1\}}C_{1} +\mathbf{1}_{\{f(\mathbf{x}) \leq 0, y=-1\}}c_{-1}
		\end{equation}
	\end{small}
	Then, the corresponding risk on $D$ and $\tilde{D}$ be denoted by $R_{D,l_{c,4}}(f)$ and $R_{\tilde{D},l_{c,4}}(f)$ respectively. 
	
	One can reduce the extended $(C_1,c_1,C_{-1},c_1)$ based $4$ cost space again to $(C_1-c_1,C_{-1}-c_{-1})$ based $2$ cost space as the threshold of the optimal classifiers does not change by reducing the parameter space. This has also been shown in \cite{lingcost_encyclopedia}. However, we still compute the minimizers of the risks $R_{D,l_{c,4}}(f)$ and $R_{\tilde{D},l_{c,4}}(f)$ and show that they need not be equal.
	
	Using the same technique as in $(C_1,C_{-1})$ cost case, one can show that the minimizers of $R_{D,l_{c,4}}$ and $R_{\tilde{D},l_{c,4}}$ are as follows:
	\begin{equation} \label{eq: f_C4_0-1_clean}
	f^*_{4c}(\mathbf{x}) = sign\left( \eta(\mathbf{x}) - \frac{C_{-1}-c_{-1}}{C_1 -c_1 + C_{-1}-c_{-1}} \right)
	\end{equation}
	\begin{equation} \label{eq: f_C4_0-1_noisy}
	\tilde{f}^*_{4c}(\mathbf{x}) = sign \left(\eta(\mathbf{x}) - \frac{C_{-1}-c_{-1} - \rho(C_1-c_1 + C_{-1}-c_{-1})}{(1-2\rho)(C_1-c_1 + C_{-1}-c_{-1})}\right) 
	\end{equation}
	The threshold of $\eta$ in equation \eqref{eq: f_C4_0-1_clean} and \eqref{eq: f_C4_0-1_noisy} cannot be same unless $C_1-c_1 = s = C_{-1} -c_{-1}$ or $\rho = 0$. Therefore, even in the extended cost space, under uniform noise and differential cost, $f^*_{4c}\neq \tilde{f}^*_{4c}$, i.e., learning a classifier from $l_{c,4}$ which is both noise robust and cost sensitive is not always possible. This is independent of the extended  sufficient condition (E-SC) that we now introduce; as we will see, this condition doesn't help us for joint uniform noise robustness and cost sensitivity. 
	
	A natural extension of sufficient condition (Symmetry condition SC \cite{ghosh2015making}) in cost sensitive case when used with weighted $0$-$1$ loss can be written as follows. For an arbitrary classifier $f(\mathbf{x})$ and $K>0$,
	\begin{eqnarray} \label{eq: extended_symmetry_condition}
	\hspace{-0.5cm}\text{\textbf{E-SC}} \hspace{2.5cm}
	\text{when }f(\mathbf{x}) > 0 ~\text{ SC becomes }~ C_{-1} + c_{1} = K \\ \nonumber 
	\text{when }f(\mathbf{x}) \leq 0 ~\text{ SC becomes }~ C_1 + c_{-1} = K
	\end{eqnarray}
	Clearly, $C_{-1} + c_{1} = K = C_1 + c_{-1}$ which can be equivalently written as $C_{-1} -c_{-1} = s = C_1 -c_1$. This implies that the effective cost for each class is same, i.e. $s$ and hence no differential costs for the classes. Therefore, E-SC cannot be sufficient condition for SLN robustness if there is differential costing of $(C_1,c_1,C_{-1},c_{-1}).$

\end{document}